\def\paperversion{arxiv}
\pdfoutput=1

\providecommand{\paperversion}{arxiv}

\documentclass[letterpaper,10pt,journal,twoside]{IEEEtran}

\newif\ifarxivversion
\edef\OAselectedversion{\paperversion}
\def\OAtcstversion{tcst}
\def\OAarxivversion{arxiv}
\ifx\OAselectedversion\OAarxivversion
  \arxivversiontrue
\else
  \ifx\OAselectedversion\OAtcstversion
    \arxivversionfalse
  \else
    \PackageError{OA-CBF}{Unknown paper version: \paperversion}
      {Set \string\paperversion\space to tcst or arxiv.}
  \fi
\fi

\usepackage{graphics} % for pdf, bitmapped graphics files
\usepackage{subfig} % for subcaption
\usepackage{wrapfig}
\usepackage{placeins} % keep main-paper floats before the references

\usepackage{amsthm} % for definition, theorm and etc
\usepackage{amsmath,amssymb}
\usepackage{graphicx}
\usepackage{tabularx}
\usepackage{multirow, makecell}
\usepackage{diagbox}
\usepackage{slashbox}
\usepackage{rotating}
\usepackage{cite}

\usepackage{url}
\usepackage{bm}
\usepackage[table]{xcolor}
\usepackage{hyperref}
\usepackage{siunitx} % degree symbol
\usepackage{booktabs}
\usepackage{cleveref}
\usepackage{mathtools} % for := \coloneqq , and split equation with raisetag. ex means x-height
\usepackage{upgreek} % to use uptau

\usepackage{enumitem} % using enumerate with ref roman option

\usepackage[ruled,vlined]{algorithm2e}
\newcommand{\algovspace}{\vspace{1pt}}
\usepackage{pifont} % for x mark 
\newcommand{\xmark}{\ding{55}}%

\usepackage[nolist, nohyperlinks]{acronym}
\acrodef{MPC}[MPC]{Model Predictive Control}
\acrodef{QP}[QP]{Quadratic Program}
\acrodef{CBF}[CBF]{Control Barrier Function}

\newcommand{\vx}{{\boldsymbol x}}
\newcommand{\vu}{{\boldsymbol u}}
\newcommand{\vq}{{\boldsymbol q}}
\newcommand{\vc}{{\boldsymbol c}}
\newcommand{\vy}{{\boldsymbol y}}
\newcommand{\vz}{{\boldsymbol z}}

\newcommand{\vp}{{\boldsymbol p}} % position 
\newcommand{\vv}{{\boldsymbol v}} % velocity

\newcommand{\ve}{{\boldsymbol e}} % unit vector (one element)

\newcommand{\lieder}{L}
\newcommand{\StateSpace}{\mathcal{X}}
\newcommand{\ControlSpace}{\mathcal{U}}
\newcommand{\RealSpace}{\mathbb{R}}
\newcommand{\Prob}{\mathbb{P}}

\newcommand{\NaturalNumber}{\mathbb{N}}
\newcommand{\Rn}{\mathbb{R}^{n}}
\newcommand{\Rm}{\mathbb{R}^{m}}
\newcommand{\Rnm}{\mathbb{R}^{n \times m}}

\newcommand{\calC}{\mathcal{C}} % invariant set
\newcommand{\calK}{\mathcal{K}} 

\newcommand{\calG}{\mathcal{G}} % graph
\newcommand{\calE}{\mathcal{E}} % edge
\newcommand{\calV}{\mathcal{V}} % vertice
\newcommand{\calX}{\mathcal{X}} % state
\newcommand{\calT}{\mathcal{T}} % tree
\newcommand{\calS}{\mathcal{S}} % safe set
\newcommand{\calN}{\mathcal{N}} % noraml distribution
\newcommand{\calD}{\mathcal{D}} % dataset
\newcommand{\calL}{\mathcal{L}} % loss function
\newcommand{\calA}{\mathcal{A}} % locally validated parameters

\newcommand{\vF}{{\bf F}}
\newcommand{\vE}{{\bf E}}
\newcommand{\vX}{{\bf X}} % set of nn inputs
\newcommand{\vY}{{\bf Y}} % set of nn outputs
\newcommand{\VSigma}{{\bm{\Sigma}}}
\newcommand{\vmu}{{\bm{\mu}}}

\newcommand{\horizon}{H} % MPC horizon

\DeclareMathOperator*{\argmin}{arg\,min}

\newtheorem{definition}{Definition}
\newtheorem{theorem}{Theorem}
\newtheorem{lemma}{Lemma}
\newtheorem{proposition}{Proposition}
\theoremstyle{definition}
\newtheorem{remark}{Remark} % no italic in remark

\theoremstyle{definition}
\newtheorem{problem}{Problem}

\newtheorem{subproblem}{Problem}[problem]

\crefname{subproblem}{Problem}{Problems}
\Crefname{subproblem}{Problem}{Problems}
\theoremstyle{definition}

\theoremstyle{definition}

\newif\ifstatementboxes
\statementboxestrue

\ifstatementboxes
  \usepackage[skins,breakable]{tcolorbox}

  \makeatletter
  \newcommand{\paperboxfootnotes}{%
    \def\@mpfn{footnote}%
    \def\thempfn{\thefootnote}}
  \makeatother

  \definecolor{statementblue}{RGB}{239,247,252}
  \definecolor{statementline}{RGB}{210,220,235}
  \definecolor{assumptionblue}{RGB}{242,249,244}
  \definecolor{assumptionline}{RGB}{211,228,216}
  \definecolor{problemblue}{RGB}{243,246,253}
  \definecolor{problemline}{RGB}{215,222,238}
  \definecolor{definitionblue}{RGB}{241,248,247}
  \definecolor{definitionline}{RGB}{210,228,224}
  \definecolor{remarkblue}{RGB}{250,249,244}
  \definecolor{remarkline}{RGB}{230,225,208}
  \definecolor{proofblue}{RGB}{251,253,255}
  \definecolor{proofline}{RGB}{238,242,247}

  \tcbset{
    paperbox/.style={
      enhanced, breakable,
      boxrule=0pt, frame hidden,
      arc=0pt, outer arc=0pt, boxsep=0pt,
      left=7pt, right=6pt, top=6pt, bottom=6pt,
      before skip=6pt, after skip=6pt,
      before upper=\paperboxfootnotes,
      coltext=.
    },
    paperstatement/.style={
      paperbox,
      colback=statementblue, colframe=statementblue,
      borderline west={3pt}{0pt}{statementline}
    },
    paperassumption/.style={
      paperbox,
      colback=assumptionblue, colframe=assumptionblue,
      borderline west={3pt}{0pt}{assumptionline}
    },
    paperproblem/.style={
      paperbox,
      colback=problemblue, colframe=problemblue,
      borderline west={3pt}{0pt}{problemline}
    },
    paperdefinition/.style={
      paperbox,
      colback=definitionblue, colframe=definitionblue,
      borderline west={3pt}{0pt}{definitionline}
    },
    paperremark/.style={
      paperbox,
      colback=remarkblue, colframe=remarkblue,
      borderline west={3pt}{0pt}{remarkline}
    },
    paperproof/.style={
      paperbox,
      colback=proofblue, colframe=proofblue,
      borderline west={3pt}{0pt}{proofline}
    }
  }

  \tcolorboxenvironment{theorem}{paperstatement}
  \tcolorboxenvironment{lemma}{paperstatement}
  \tcolorboxenvironment{proposition}{paperstatement}
  \tcolorboxenvironment{problem}{paperproblem}
  \tcolorboxenvironment{subproblem}{paperproblem}
  \tcolorboxenvironment{definition}{paperdefinition}
  \tcolorboxenvironment{remark}{paperremark}
  \tcolorboxenvironment{proof}{paperproof}
  \tcolorboxenvironment{approach}{paperdefinition}
  \tcolorboxenvironment{example}{paperremark}
  \ifcsname assumption\endcsname
    \tcolorboxenvironment{assumption}{paperassumption}
  \fi
\fi
\newif\ifshowrevisions
\showrevisionstrue
\makeatletter
\def\oa@pagecounter{page}
\def\oa@savecounter#1{%
  \def\oa@countername{#1}%
  \ifx\oa@countername\oa@pagecounter\else
    \expandafter\edef\csname oa@saved@#1\endcsname{\number\value{#1}}%
  \fi}
\def\oa@restorecounter#1{%
  \def\oa@countername{#1}%
  \ifx\oa@countername\oa@pagecounter\else
    \setcounter{#1}{\csname oa@saved@#1\endcsname}%
  \fi}
\AtBeginDocument{\let\oa@original@xfloat\@xfloat}
\def\oa@revisioncolor#1{%
  \color{#1}%
  \def\@xfloat##1[##2]{%
    \oa@original@xfloat{##1}[##2]%
    \def\normalcolor{\color{#1}}\color{#1}}}
\DeclareRobustCommand{\beforefix}[1]{%
  \ifshowrevisions
    \begingroup
    \let\oa@saved@elt\@elt
    \let\@elt\oa@savecounter\cl@@ckpt
    \let\@elt\oa@saved@elt
    \@fileswfalse
    \let\oa@old@nocite\@nocite
    \def\@nocite##1{\begingroup\@fileswtrue\oa@old@nocite{##1}\endgroup}%
    \let\label\@gobble
    \let\ltx@label\@gobble
    \let\protected@write\@gobblethree
    \let\addcontentsline\@gobblethree
    \NoHyper
    \oa@revisioncolor{red}%
    #1%
    \let\@elt\oa@restorecounter\cl@@ckpt
    \endNoHyper
    \endgroup
    \ifhmode\unskip\space\fi % separate adjacent inline red/blue text
  \fi}
\DeclareRobustCommand{\fixed}[1]{%
  \ifshowrevisions
    \begingroup\oa@revisioncolor{blue}#1\endgroup
  \else
    #1%
  \fi}
\makeatother
\makeatletter
\let\save@mathaccent\mathaccent
\newcommand*\if@single[3]{%
  \setbox0\hbox{${\mathaccent"0362{#1}}^H$}%
  \setbox2\hbox{${\mathaccent"0362{\kern0pt#1}}^H$}%
  \ifdim\ht0=\ht2 #3\else #2\fi
  }
\newcommand*\rel@kern[1]{\kern#1\dimexpr\macc@kerna}
\newcommand*\widebar[1]{\@ifnextchar^{{\wide@bar{#1}{0}}}{\wide@bar{#1}{1}}}
\newcommand*\wide@bar[2]{\if@single{#1}{\wide@bar@{#1}{#2}{1}}{\wide@bar@{#1}{#2}{2}}}
\newcommand*\wide@bar@[3]{%
  \begingroup
  \def\mathaccent##1##2{%
    \let\mathaccent\save@mathaccent
    \if#32 \let\macc@nucleus\first@char \fi
    \setbox\z@\hbox{$\macc@style{\macc@nucleus}_{}$}%
    \setbox\tw@\hbox{$\macc@style{\macc@nucleus}{}_{}$}%
    \dimen@\wd\tw@
    \advance\dimen@-\wd\z@
    \divide\dimen@ 3
    \@tempdima\wd\tw@
    \advance\@tempdima-\scriptspace
    \divide\@tempdima 10
    \advance\dimen@-\@tempdima
    \ifdim\dimen@>\z@ \dimen@0pt\fi
    \rel@kern{0.6}\kern-\dimen@
    \if#31
      \overline{\rel@kern{-0.6}\kern\dimen@\macc@nucleus\rel@kern{0.4}\kern\dimen@}%
      \advance\dimen@0.4\dimexpr\macc@kerna
      \let\final@kern#2%
      \ifdim\dimen@<\z@ \let\final@kern1\fi
      \if\final@kern1 \kern-\dimen@\fi
    \else
      \overline{\rel@kern{-0.6}\kern\dimen@#1}%
    \fi
  }%
  \macc@depth\@ne
  \let\math@bgroup\@empty \let\math@egroup\macc@set@skewchar
  \mathsurround\z@ \frozen@everymath{\mathgroup\macc@group\relax}%
  \macc@set@skewchar\relax
  \let\mathaccentV\macc@nested@a
  \if#31
    \macc@nested@a\relax111{#1}%
  \else
    \def\gobble@till@marker##1\endmarker{}%
    \futurelet\first@char\gobble@till@marker#1\endmarker
    \ifcat\noexpand\first@char A\else
      \def\first@char{}%
    \fi
    \macc@nested@a\relax111{\first@char}%
  \fi
  \endgroup
}
\makeatother
\newenvironment{authorbiography}[2][]{%
  \par\noindent
  \begin{minipage}[t]{\columnwidth}
    \begin{minipage}[t]{1in}
      \vspace{0pt}
      #1
    \end{minipage}
    \hspace{0.8em}%
    \begin{minipage}[t]{\dimexpr\columnwidth-1.15in\relax}
      \vspace{0pt}
      \small
      \textbf{#2}\ 
}{%
    \end{minipage}
  \end{minipage}
  \par\vspace{0.8em}
}

\DeclareCaptionFont{mysize}{\fontsize{8}{9.6}\selectfont}
\hypersetup{hypertexnames=false}

\DeclareRobustCommand{\appendixref}[2]{%
  \ifarxivversion\ref{#1}\else#2\fi}
\DeclareRobustCommand{\appendixeqref}[2]{%
  \ifarxivversion\eqref{#1}\else\textup{(#2)}\fi}
\newcommand{\discreteappendix}{Appendix~\appendixref{app:discrete_time}{B}\ifarxivversion\else\ of the extended version\fi}

\ifdefined\OAacceptedview\showrevisionsfalse\fi

\title{Learning to Adapt Control Barrier Functions Under Epistemic and Aleatoric Uncertainty}

\author{Taekyung Kim$^{1}$, Robin Inho Kee$^{1,2}$, and Dimitra Panagou$^{1,3}$%
\thanks{This research was supported by the Center for Autonomous Air Mobility and Sensing (CAAMS), an NSF IUCRC, under Award Number 2137195, an NSF CAREER under Award Number 1942907.}
\thanks{$^{1}$Taekyung Kim and Robin Inho Kee are with the Department of Robotics, University of Michigan, Ann Arbor, MI 48109 USA (e-mails: taekyung@umich.edu; inhokee@umich.edu).}
\thanks{$^{2}$Robin Inho Kee is also a Draper Scholar with the Charles Stark Draper Laboratory, Cambridge, MA, USA. Robin Inho Kee acknowledges support from the Draper Scholars Program.}
\thanks{$^{1,3}$Dimitra Panagou is with the Department of Robotics and the Department of Aerospace Engineering, University of Michigan, Ann Arbor, MI 48109 USA (e-mail: dpanagou@umich.edu).}
}

\begin{document}
\maketitle

\ifdefined\OAblueonlyview
  \begin{abstract}
Control barrier functions (CBFs) provide a tractable mechanism for enforcing safety constraints, but the system response depends strongly on their class-$\calK$ function parameters. Under input constraints, these parameters affect both the feasibility of the underlying optimization problem and the tradeoff between conservative and aggressive responses. This paper proposes Online Adaptive CBF (OA-CBF), a framework for adapting CBF parameters at runtime to improve task performance under input constraints. We introduce locally validated CBF parameters as a finite-horizon validation target and train a probabilistic ensemble neural network to evaluate queried parameters rather than directly predict them. A graph-attention encoder provides a fixed-dimensional representation of variable-size environments, while the probabilistic ensemble characterizes epistemic uncertainty through disagreement among its members and aleatoric uncertainty through their individual predictive variances. To account for these uncertainties, we use conformal prediction to reject potentially out-of-distribution predictions, and distributionally robust conditional value-at-risk to assess the predicted safety risk across ensemble members. Among the retained candidates, OA-CBF selects the parameter with the best predicted progress metric and applies it through either an MPC-CBF or CBF-QP safety filter. Simulation studies on dynamic unicycle, planar and three-dimensional quadrotor, kinematic bicycle, and VTOL quadplane benchmarks show that OA-CBF reduces the conservatism of fixed-parameter CBF controllers while maintaining low collision and infeasibility rates. \href{https://github.com/tkkim-robot/online_adaptive_cbf}{\textcolor{red}{[Code]}}  \href{https://www.taekyung.me/oa-cbf}{\textcolor{red}{[Project Page]}}\footnote{Project page: \href{https://www.taekyung.me/oa-cbf}{https://www.taekyung.me/oa-cbf}}
\end{abstract} %%

\section{INTRODUCTION}
\subsection{Background}

\begin{figure*}[t]
\centering
\includegraphics[width=0.99\textwidth]{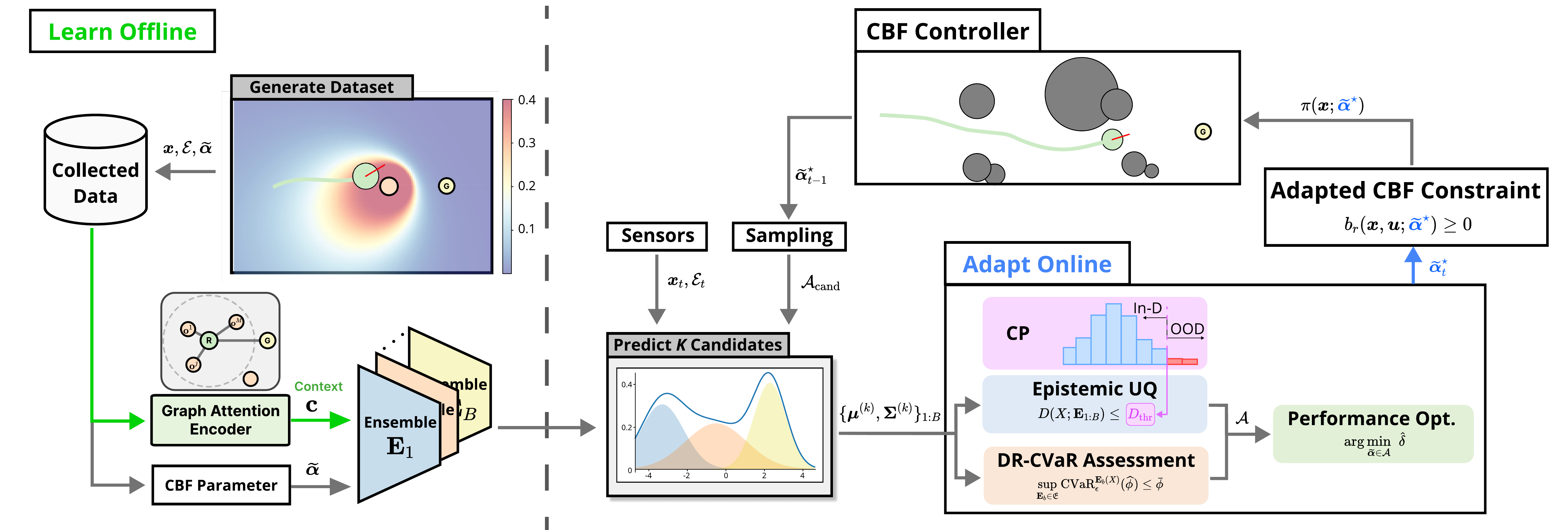}
\caption{
Overview of OA-CBF. Offline rollouts are used to train a graph-attention encoder and probabilistic ensemble neural networks that predict safety and 
performance metrics for queried CBF parameters. During deployment, candidate 
parameters are sampled online, evaluated in batch, filtered by an epistemic uncertainty gate calibrated using conformal prediction and an aleatoric uncertainty gate evaluated by distributionally robust CVaR, and selected according to predicted task performance before updating the CBF-based controller.
}
\label{fig:overview}
\label{fig:main_diagram}
\vspace{-5pt}
\end{figure*}

Robotic and autonomous systems operating in cluttered or dynamic environments must satisfy hard safety constraints while still making efficient progress toward their goals. Control Barrier Functions (CBFs) have become a standard tool for safety-critical control because they synthesize safety filters that minimally modify a nominal controller while enforcing forward invariance of a safe set~\cite{ames_control_2019}. For robotic systems with high-relative-degree constraints or hard actuator limits, this idea has been extended through higher-order CBFs~\cite{xiao_control_2019}, input-constrained CBFs~\cite{agrawal_safe_2021}, and sampled-data MPC-CBF formulations~\cite{zeng_safetycritical_2021, zeng_enhancing_2021}.

A persistent practical difficulty is that the behavior of a CBF-based controller depends strongly on the chosen class-$\mathcal K$ functions, or equivalently on the decay parameters used in the CBF constraint~\cite{xiao_feasibilityguided_2020, parwana_ratetunable_2025, kim_your_2026}. These parameters do not merely tune performance; under input constraints, they also influence feasibility and the size of the corresponding safe set that the controller can render forward invariant~\cite{agrawal_safe_2021, kim_how_2025}. Conservative parameters typically activate the safety constraint early and preserve feasibility, but they can lead to unnecessarily slow or circuitous motion. Aggressive parameters may improve nominal progress, but they can make the safety filter infeasible near obstacles or under limited control authority.

This tradeoff is especially restrictive when a single fixed parameter is used across all states and environments. A parameter that is safe for one operating condition may be overly conservative in another, while a parameter that performs well far from obstacles may render the safety filter infeasible when the system approaches the boundary of the safe set with insufficient actuation margin. Therefore, effective CBF deployment requires a mechanism that can adapt the CBF parameter online while rejecting unsafe or unreliable updates before they are applied to the controller.

Learning-based adaptation offers a way to reduce this conservatism, but direct prediction of CBF parameters introduces a new challenge: the learned model may be overconfident on states, constraint functions, or candidate parameters that are poorly represented in the training data. Such overconfident extrapolation can be more harmful than conservative behavior, because an invalid parameter update directly changes the feasible-control set of the safety filter. This motivates a different viewpoint: rather than treating the learned model as a direct parameter predictor, we use it as an uncertainty-aware evaluator of candidate CBF parameters.

In this paper, we develop Online Adaptive CBF~(OA-CBF), a framework that verifies and adapts CBF parameters using predictions over a finite horizon. The proposed method evaluates the predicted safety and performance of candidate parameters, applies epistemic and aleatoric uncertainty tests, and then selects among the retained candidates according to predicted task performance. The epistemic test controls the rejection probability of in-distribution trajectories, while the aleatoric test bounds risk under the predictive model. This provides a probabilistic verification step before each parameter update; the overall offline-training and online-adaptation workflow is summarized in Fig.~\ref{fig:overview}.
\subsection{Our Contributions}

This paper makes the following contributions.
\begin{itemize}
    \item We formulate online CBF-parameter adaptation under input constraints through the notion of locally validated CBF parameters. Instead of requiring a globally valid CBF parameter for all future states, the proposed formulation verifies candidate parameters over a finite prediction horizon around the current operating condition. We show that repeated local validation over overlapping update intervals preserves safety of the resulting adaptive CBF controller.

    \item We develop an uncertainty-aware verification procedure for CBF-parameter adaptation. A probabilistic ensemble neural network predicts safety-risk and performance metrics for queried candidate parameters, while epistemic uncertainty is quantified through ensemble disagreement and calibrated using conformal prediction. For candidates that pass this epistemic gate, a distributionally robust CVaR condition is used to screen aleatoric risk before the parameter is accepted.

    \item We demonstrate that the OA-CBF framework can be coupled with both MPC-CBF and CBF-QP safety filters. The proposed method is evaluated on dynamic unicycle, Quad2D, Quad3D, kinematic bicycle, and VTOL quadplane systems, covering both distance-based and non-distance-based CBF candidates. The results show that OA-CBF reduces the conservatism of fixed-parameter controllers and reduces safety-failure rates relative to aggressive fixed-parameter, optimal-decay, and learning-based CBF-QP baselines.
\end{itemize}

\subsection{Related Work}

CBF-based safety methods relevant to this paper can be grouped into
constructive synthesis, learned certificate construction, and online adaptation of safety filters. A key distinction is not only \emph{which safety certificate or safety-filter component is constructed or adapted}, but also whether \emph{deployment-time reconfiguration under input constraints} is supported together with an explicit validation mechanism. \autoref{tab:related_work} provides a summary comparison of these methods.

\begin{table*}[t]
\centering
\caption{Comparison of representative directions for constructing or adapting barrier-based safety filters.}
\label{tab:related_work}
\renewcommand{\arraystretch}{1.3}
\resizebox{\textwidth}{!}{%
\begin{tabular}{lcccccc}
\toprule
\textbf{Method} & \textbf{Design Target} & \textbf{Control Limit Treatment} & \makecell[c]{\textbf{Runtime} \\ \textbf{Reconfiguration}} & \textbf{Runtime Uncertainty Treatment} & \textbf{Core Mechanism} \\
\midrule
SOS / Algebraic CBF synthesis~\cite{xu_correctness_2018, clark_verification_2021}
& CBF itself
& Explicit
& \xmark
& N/A
& SOS program / Positivstellensatz \\

CBVF / HJ reachability~\cite{choi_robust_2021}
& Value function used as CBF
& Native
& \xmark
& Bounded-disturbance robustness
& HJI variational inequality \\

RefineCBF / HJ-Patch~\cite{tonkens_refining_2022, tonkens_refining_2026}
& Candidate CBF / Almost-safe value function
& Native
& \checkmark
& N/A
& Warm-started / Local HJ refinement \\

Robust neural CLBF~\cite{dawson_safe_2021}
& NN approximating CLBF
& Problem-dependent
& \xmark
& Modeled uncertainty; No runtime predictive UQ
& Neural certificate learning \\

PNCBF~\cite{so_how_2024}
& Neural CBF via policy value function
& Native
& \xmark
& Post-training verification; No runtime predictive UQ
& Policy evaluation + CBF-QP \\

RBN~\cite{kim_reachability_2025}
& HJ-based neural CBF approximation
& Native
& \xmark
& Conformal safety calibration
& PINN approximation of HJ solutions \\

aCBF / RaCBF~\cite{taylor_adaptive_2020, black_fixedtime_2021}
& Unknown dynamics parameters
& No / Problem-dependent
& \checkmark
& Dynamics uncertainty only
& Lyapunov-based adaptation \\

Adaptive CBF~\cite{xiao_adaptive_2022}
& Time-varying penalty + Relaxation functions
& Explicit
& \checkmark
& Dynamics uncertainty only
& Auxiliary dynamics + HOCBF/CLF \\

Learned CBF-parameter tuning~\cite{ma_learning_2022, xiao_barriernet_2023, gao_online_2023, berducci_learning_2024, mohammad_soft_2025}
& CBF parameters tuned online
& No / Problem-dependent
& \checkmark
& No runtime predictive UQ
& Differentiable QP / RL / GNN policy \\
\midrule

\textbf{Ours (OA-CBF)}
& \textbf{ICCBF parameter $\bm{\alpha}=\{\alpha_1,\ldots,\alpha_r\}$}
& \textbf{Explicit via local validation}
& \checkmark
& \textbf{Epistemic gate + Aleatoric risk screening}
& \textbf{Probabilistic ensemble + Local validation} \\
\bottomrule
\end{tabular}%
}
\end{table*}

\subsubsection{Constructive and Reachability-Based Methods}

Constructive methods include sum-of-squares~(SOS) techniques, which certify nonnegativity of polynomial expressions encoding control and safety conditions through semidefinite programming~\cite{papachristodoulou_tutorial_2005}. Within the CBF literature, Xu \textit{et al.}~\cite{xu_correctness_2018} synthesize CBFs through a combination of SOS programming, physics-based modeling, and optimization, while Clark~\cite{clark_verification_2021} provides SOS-based conditions for verifying and synthesizing CBFs for polynomial control systems. Hamilton-Jacobi~(HJ) reachability offers a general model-based method: the resulting value function characterizes the safe set and, in the Control Barrier-Value Function~(CBVF) formulation, can be used directly in a safety filter with robustness to bounded disturbances~\cite{choi_robust_2021}. These approaches provide strong guarantees, but they suffer from the curse of dimensionality; in practice, direct offline full-grid computation is generally limited to five-dimensional systems~\cite{bansal_hamiltonjacobi_2017,bansal_deepreach_2021}. RefineCBF addresses this gap by warm-starting reachability from a candidate barrier and provably improving safety monotonically~\cite{tonkens_refining_2022}; a recent extension further localizes these updates for in-the-loop refinement under changing environments~\cite{tonkens_refining_2026}.

\subsubsection{Learning-Based Barrier Construction}

Learning-based methods improve computational scalability by approximating certificates or reachability solutions with neural networks. Robust Neural Control Lyapunov-Barrier Functions (CLBFs) learn certificates that retain robustness to modeled uncertainty~\cite{dawson_safe_2021}, and the broader landscape of neural Lyapunov, barrier, and contraction certificates is surveyed in~\cite{dawson_safe_2023}. For high-relative-degree systems under tight actuator limits, Policy Neural CBFs (PNCBFs) learn the value function of a nominal policy and show that the resulting maximum-over-time policy value function is itself a CBF~\cite{so_how_2024}. Reachability Barrier Networks (RBNs) instead use physics-informed neural networks~(PINNs) to approximate discounted HJ solutions and combine them with conformal prediction to obtain probabilistic safety guarantees for the learned approximation~\cite{kim_reachability_2025}. In perception-driven settings, In-Distribution Barrier Functions learn a self-supervised policy filter that keeps trajectories close to the distribution of safe demonstrations~\cite{castaneda_indistribution_2023}, while GCBF+ uses graph neural networks to parameterize graph control barrier functions for scalable distributed multi-agent safety~\cite{zhang_gcbf_2025}. Collectively, these methods substantially expand the practical reach of learned safety filters. During deployment, however, these methods typically rely on the offline learned certificate or calibrated approximation; in general, they neither support online reconfiguration nor explicitly address predictive uncertainty at runtime.

\subsubsection{Online Adaptation of Safety Filters}

A complementary direction keeps the barrier structure fixed and adapts parameters of the safety filter online. Relatedly, episodic learning with Control Lyapunov Functions~(CLFs) was studied for uncertain robotic systems with parametric uncertainty and unmodeled dynamics~\cite{taylor_episodic_2019}. Taylor and Ames~\cite{taylor_adaptive_2020} later proposed an adaptive CBF~(aCBF) formulation for safety under parametric model uncertainty, while Black \textit{et al.}~\cite{black_fixedtime_2021} introduced a fixed-time-stable parameter adaptation law within a robust adaptive CBFs (RaCBFs) framework, guaranteeing fixed-time convergence of the parameter estimates. In a distinct line of work that uses the same ``adaptive CBFs’’ terminology, Xiao \textit{et al.}~\cite{xiao_adaptive_2022} consider time-varying control bounds and noisy dynamics, introducing adaptive penalty or relaxation variables with auxiliary dynamics to address QP feasibility and conservativeness rather than parametric uncertainty estimation.

More closely related to our method, another line of work adapts class-$\mathcal{K}$ functions or related CBF hyperparameters to improve feasibility and reduce conservatism. Rate-Tunable CBFs (RT-CBFs)~\cite{parwana_ratetunable_2025} parameterize the class-$\mathcal{K}$ function and adapt its parameters online, 
deriving pointwise sufficient conditions on the parameter dynamics for compatibility of multiple CBF constraints. Yet, how to systematically synthesize such parameter adaptation laws remains a challenging open problem. Among learning-based approaches, differentiable CBF-QP layers have been used to learn environment- or state-dependent barrier tuning for improved generalization and reduced conservatism~\cite{ma_learning_2022,xiao_barriernet_2023}, while reinforcement-learning-based methods learn online adaptation policies for decentralized navigation~\cite{gao_online_2023}, interactive multi-agent settings~\cite{berducci_learning_2024}, and navigation in unknown environments~\cite{mohammad_soft_2025} by tuning class-$\mathcal{K}$ functions or related constraint hyperparameters. Collectively, these works show that adapting CBF hyperparameters can substantially improve feasibility and performance of the underlying CBF-based controller.
However, directly predicting CBF parameters does not establish whether the resulting closed-loop trajectory satisfies the CBF constraints over a finite prediction horizon under input constraints. This motivates an alternative approach that evaluates the predicted safety and performance of candidate CBF parameters while accounting for uncertainty in these predictions.

\section{PRELIMINARIES \label{sec:preliminaries}}

\subsection{Tangent Cone and Control Barrier Functions\label{subsec:nagumo}}

Consider a continuous-time nonlinear system:
\begin{equation}
\dot{\vx} = f(\vx) + g(\vx)\vu,
\label{eq:ct_dynamics}
\end{equation}
where $\vx \in \StateSpace \subset \Rn$ is the state and $\vu \in \ControlSpace \subset \Rm$ is the control input, with $\ControlSpace$ being the set of admissible controls for System~\eqref{eq:ct_dynamics}. 
We assume that $f: \StateSpace \to \Rn$ and $g: \StateSpace \to \Rnm$ are locally Lipschitz continuous and that, under a feedback law $\vu=\pi(\vx)$, closed-loop solutions exist and are unique on each interval where the policy is defined. The forward-invariance statements below additionally assume that these solutions exist for all $t\ge t_0$. The closed-loop dynamics are
\begin{equation}
\dot{\vx} = f_\textup{cl}(\vx) \coloneqq f(\vx) + g(\vx)\pi(\vx).
\label{eq:closed_loop}
\end{equation}

\begin{definition}[Tangent Cone \cite{blanchini_settheoretic_2008}]\label{def:tangent_cone}
Let $\calS \subset \Rn$ be a closed set. The tangent cone to $\calS$ at $\vx \in \calS$ is defined as
\begin{equation} \label{eq:tangent_cone}
\calT_{\calS}(\vx) \coloneqq \left\{ \vz \in \Rn \mid \liminf_{\tau \downarrow 0} \frac{\textup{dist}(\vx + \tau \vz, \calS)}{\tau} = 0 \right\}.
\end{equation}
\end{definition}
Here, $\textup{dist}(\vy, \calS)$ denotes the distance from a point $\vy \in \Rn$ to $\calS$. If $\calS$ is convex, the ``$\liminf$" can be replaced by a ``$\lim$" in~\eqref{eq:tangent_cone}, and $\calT_{\calS}(\vx)$ remains convex. Moreover, for any $\vx \in \textup{Int}(\calS)$ we have $\calT_{\calS}(\vx)=\Rn$, so that the tangent cone is only non-trivial on the boundary $\partial\calS$ of $\calS$.

\begin{definition}[Forward Invariance]\label{def:forward_invariance}
A set $\calS \subset \Rn$ is rendered forward invariant by a feedback controller $\pi: \calS \to \ControlSpace$ for the closed-loop system~\eqref{eq:closed_loop} if for all $\vx(t_{0}) \in \calS$, the solution $\vx(t) \in \calS$ for all $t \geq t_{0}$.
\end{definition}

\begin{theorem}[Nagumo's Theorem \cite{nagumo_uber_1942}]\label{thm:nagumo}
Consider the closed-loop system~\eqref{eq:closed_loop}\footnote{Although Nagumo's original theorem was stated for unforced systems, here it has been extended to the closed-loop controlled system.} and let $\calS \subset \Rn$ be a closed set. Then, $\calS$ is rendered forward invariant by $\pi$ if and only if
\begin{equation}
f(\vx) + g(\vx)\pi(\vx) \in \calT_\calS(\vx) \quad \forall \vx \in \calS.
\end{equation}
\end{theorem}

The safety guarantee provided by a CBF can be interpreted through the lens of the tangent cone as characterized by Nagumo's Theorem. For extra clarity in this paper, we first introduce the concept of a candidate CBF. From here, we denote the set $\calS$ as the set of safe states, which can often be encoded by a set of hard constraints provided to the system as requirements.
\begin{definition}[Candidate CBF]\label{def:candidate_cbf}
Let $\widetilde{h}: \StateSpace \to \RealSpace$ be a continuously differentiable function and let $\calC \coloneqq \{\vx \in \StateSpace \mid \widetilde{h}(\vx) \geq 0\}\subseteq\calS$. We call $\widetilde{h}$ a candidate CBF for $\calC$.
\end{definition}

 A \emph{candidate CBF} becomes a \emph{CBF} when it is paired with a properly chosen extended class $\calK_{\infty}$ function $\alpha$. More precisely, we have the following definition:
\begin{definition}[CBF \cite{ames_control_2019}] \label{def:cbf}
Let $h: \StateSpace \to \RealSpace$ be a continuously differentiable function and let $\calC \coloneqq \{\vx \in \StateSpace \mid h(\vx) \geq 0\}$. The function $h$ is a CBF on $\calC$ for System~\eqref{eq:ct_dynamics} if there exists an extended class-$\calK_{\infty}$ function $\alpha:\RealSpace\to\RealSpace$ such that
   \begin{equation}
       \sup_{\vu \in \ControlSpace} \left[\lieder_f h(\vx) + \lieder_g h(\vx)\vu\right] \geq -\alpha(h(\vx)) \quad \forall \vx \in \calC, 
   \label{eq:cbf-condition}
   \end{equation}
where $\lieder_f h(\vx)=\nabla h(\vx)^\top f(\vx)$ and $\lieder_g h(\vx)=\nabla h(\vx)^\top g(\vx)$ are the Lie derivatives of $h$ along $f$ and $g$, respectively.
\end{definition} 
Given a CBF~$h$ and its corresponding function $\alpha$, the set of all control inputs that render $\calC \subseteq \calS$ forward invariant, thus ensuring the system remains safe, is defined by the CBF constraint:
\begin{equation}
K_{\textup{cbf}}(\vx; \alpha) \coloneqq \left\{\vu \in \ControlSpace \mid \lieder_f h(\vx) + \lieder_g h(\vx)\vu \geq -\alpha(h(\vx)) \right\}.
\label{eq:k_cbf}
\end{equation}

We introduce two set-valued maps as follows. For each state $\vx \in \StateSpace$, define
\begin{equation}
F(\vx) \coloneqq \left\{ f(\vx) + g(\vx)\vu \mid \vu \in \ControlSpace \right\} \nonumber
\end{equation}
which describes all feasible state derivatives under any admissible control. Define also
\begin{equation}
G(\vx; \alpha) \coloneqq \left\{ f(\vx) + g(\vx)\vu \mid \vu \in K_{\textup{cbf}}(\vx; \alpha) \right\} \nonumber
\end{equation}
which captures the state derivatives that result from control inputs that satisfy the CBF constraint. Clearly, $G(\vx; \alpha)$ is a safety-restricted subset of $F(\vx)$.

The CBF constraint can be interpreted directly through \autoref{thm:nagumo}. For any feedback controller $\pi:\calC\to\ControlSpace$ satisfying $\pi(\vx)\in K_{\textup{cbf}}(\vx;\alpha)$, the closed-loop vector field satisfies
\[
L_fh(\vx)+L_gh(\vx)\pi(\vx)\ge -\alpha(h(\vx)).
\]
On $\partial\calC$, where $h(\vx)=0$, this condition gives $L_fh(\vx)+L_gh(\vx)\pi(\vx)\ge0$. If $\nabla h(\vx)\ne0$ on $\partial\calC$, this is the boundary tangency condition for $\calC$, and \autoref{thm:nagumo} gives forward invariance.

\subsection{Input Constrained CBFs \label{subsec:iccbf}}

The concept of CBFs has been generalized to HOCBFs in both continuous-time~\cite{xiao_control_2019} and discrete time~\cite{xiong_discretetime_2023} settings, which can be used for constraints of high relative degree. Input Constrained CBFs~(ICCBFs)~\cite{agrawal_safe_2021} further generalize HOCBFs, allowing for constraints of higher relative degrees and non-uniform relative degrees for control inputs.

The ICCBF construction addresses input constraints, i.e., $\ControlSpace \neq \Rm$, by restricting a candidate safe set to an inner set on which admissible control inputs satisfy the barrier constraint. Let $h$ be an $r$-times continuously differentiable candidate CBF. We assume that the recursion below produces continuously differentiable functions $b_i$ and that the class-$\calK_\infty$ functions are locally Lipschitz. Define $b_i:\StateSpace\to\RealSpace$, $i=0,\ldots,r-1$, as
{\setlength{\jot}{0.5ex} %
\begin{equation}\label{eq:iccbf-recursion}
\begin{split}
\raisetag{9.0ex} %
b_0(\vx;\bm{\alpha}) &\coloneqq h(\vx), \\
b_1(\vx;\bm{\alpha}) &\coloneqq \inf_{\vu \in \ControlSpace}\left[ \dot{b}_0(\vx, \vu;\bm{\alpha}) \right] + \alpha_1 ( b_0(\vx;\bm{\alpha}) ), \\
&\vdots \\
b_{r-1}(\vx;\bm{\alpha}) &\coloneqq \inf_{\vu \in \ControlSpace}\left[ \dot{b}_{r-2}(\vx, \vu;\bm{\alpha}) \right] + \alpha_{r-1} ( b_{r-2}(\vx;\bm{\alpha}) ),
\end{split}
\end{equation} }
and the constraint function $b_{r}:\StateSpace \times \ControlSpace \to \RealSpace$ as
\begin{equation} \label{eq:iccbf-constraint}
b_{r}(\vx, \vu;\bm{\alpha}) \coloneqq \dot{b}_{r-1}(\vx, \vu;\bm{\alpha}) + \alpha_{r}\left( b_{r-1}(\vx;\bm{\alpha}) \right) ,
\end{equation}
where $\bm{\alpha} =\{ \alpha_1,\ldots,\alpha_r \}$ is a set of extended class $\calK_{\infty}$ functions. We also define the following sets:
{\setlength{\jot}{0.5ex} %
\begin{equation}
\begin{split}
\raisetag{9.0ex} %
\calC_{0}(\bm{\alpha}) &\coloneqq \{\vx \in \StateSpace \mid b_{0}(\vx; \bm{\alpha}) \geq 0\} \subseteq \calS\\
\calC_{1}(\bm{\alpha}) &\coloneqq \{\vx \in \StateSpace \mid b_{1}(\vx; \bm{\alpha}) \geq 0\} \\
&\vdots \\
\calC_{r-1}(\bm{\alpha}) &\coloneqq \{\vx \in \StateSpace \mid b_{r-1}(\vx; \bm{\alpha}) \geq 0\} .
\end{split}
\label{eq:iccbf-sets}
\end{equation} }

\begin{definition}[Inner Safe Set~\cite{kim_learning_2025}]\label{def:inner-safe-set}
The set $\calC^{\star}(\bm{\alpha}) \subseteq \calS$,
\begin{equation} \label{eq:inner-safe-set}
\calC^{\star}(\bm{\alpha}) \coloneqq \bigcap_{i=0}^{r-1} \calC_i(\bm{\alpha}) ,
\end{equation}
where $\calC_{i}(\bm{\alpha})$ is defined in \eqref{eq:iccbf-sets}, is called an inner safe set for System~\eqref{eq:closed_loop}.
\end{definition}

\begin{definition}[ICCBF~\cite{agrawal_safe_2021}]\label{def:iccbf}
The function $h$ is an ICCBF\footnote{In \cite{agrawal_safe_2021}, $b_{r-1}$ is referred to as ICCBF; in this paper, for notational consistency with HOCBF~\cite{xiao_control_2019}, we designate $h$ as ICCBF, with all corresponding theoretical properties preserved.} on $\calC^{\star}(\bm{\alpha})$~\eqref{eq:inner-safe-set} for System~\eqref{eq:closed_loop} if there exist extended class $\calK_{\infty}$ functions $\bm{\alpha} = \{\alpha_1, \ldots, \alpha_r\}$ such that
\begin{equation}\label{eq:iccbf-condition}
\sup_{\vu \in \ControlSpace} \left[ b_r(\vx, \vu; \bm{\alpha}) \right] \geq 0 \quad \forall \vx \in \calC^{\star}(\bm{\alpha}) .
\end{equation}
\end{definition}
This condition allows the formulation of safety constraints for systems of relative degree at most $r$.%

For a fixed ICCBF parameter $\bm{\alpha}$, the set of admissible controls satisfying the ICCBF constraint is
\[
K_{\textup{iccbf}}(\vx;\bm{\alpha}) \coloneqq \{\vu\in\ControlSpace \mid b_r(\vx,\vu;\bm{\alpha})\ge0\}.
\]
The corresponding allowable state derivatives are
\[
G(\vx;\bm{\alpha}) \coloneqq \{f(\vx)+g(\vx)\vu \mid
\vu\in K_{\textup{iccbf}}(\vx;\bm{\alpha})\}.
\]
If $h$ is an ICCBF on $\calC^{\star}(\bm{\alpha})$ and the applied feedback satisfies $\pi(\vx)\in K_{\textup{iccbf}}(\vx;\bm{\alpha})$ there, then $\calC^{\star}(\bm{\alpha})\subseteq\calS$ is forward invariant~\cite{agrawal_safe_2021}.

Given a feedback controller $\vu_\textup{nom}(\vx)$ for System~\eqref{eq:ct_dynamics}, one can design a safety-critical controller using the following Quadratic Program~(QP) formulation that finds the minimally deviated safe control input $\vu^{\star}$ from the nominal control input:
\noindent\rule{\columnwidth}{0.4pt}
\textbf{CBF-QP:}
\begin{subequations}
\label{eq:cbf-qp-formulation}
\begin{align}
\vu^{\star} = \pi(\vx;\bm{\alpha}) & \coloneqq \argmin_{\vu \in \ControlSpace} \| \vu - \vu_\textup{nom}(\vx) \|^{2} \\
& \text{s.t.} \quad b_{r}(\vx,\vu;\bm{\alpha}) \geq 0 \label{eq:cbf-qp-constraint}
\end{align}
\end{subequations}
\noindent\rule{\columnwidth}{0.4pt}

The main text uses the continuous-time CBF-QP form for notational clarity. In sampled-data implementations, the same CBF parameter enters a discrete-time CBF or MPC-CBF constraint; the corresponding discrete-time ICCBF and MPC-CBF formulation is given in \discreteappendix.\footnote{Extended version: \url{https://arxiv.org/abs/2504.03038}.}

\section{OA-CBF PROBLEM FORMULATION\label{sec:problem}}
\subsection{Online Adaptation and Local Validation}

\begin{figure}[tbp]
\centering
\includegraphics[width=0.93 \linewidth]{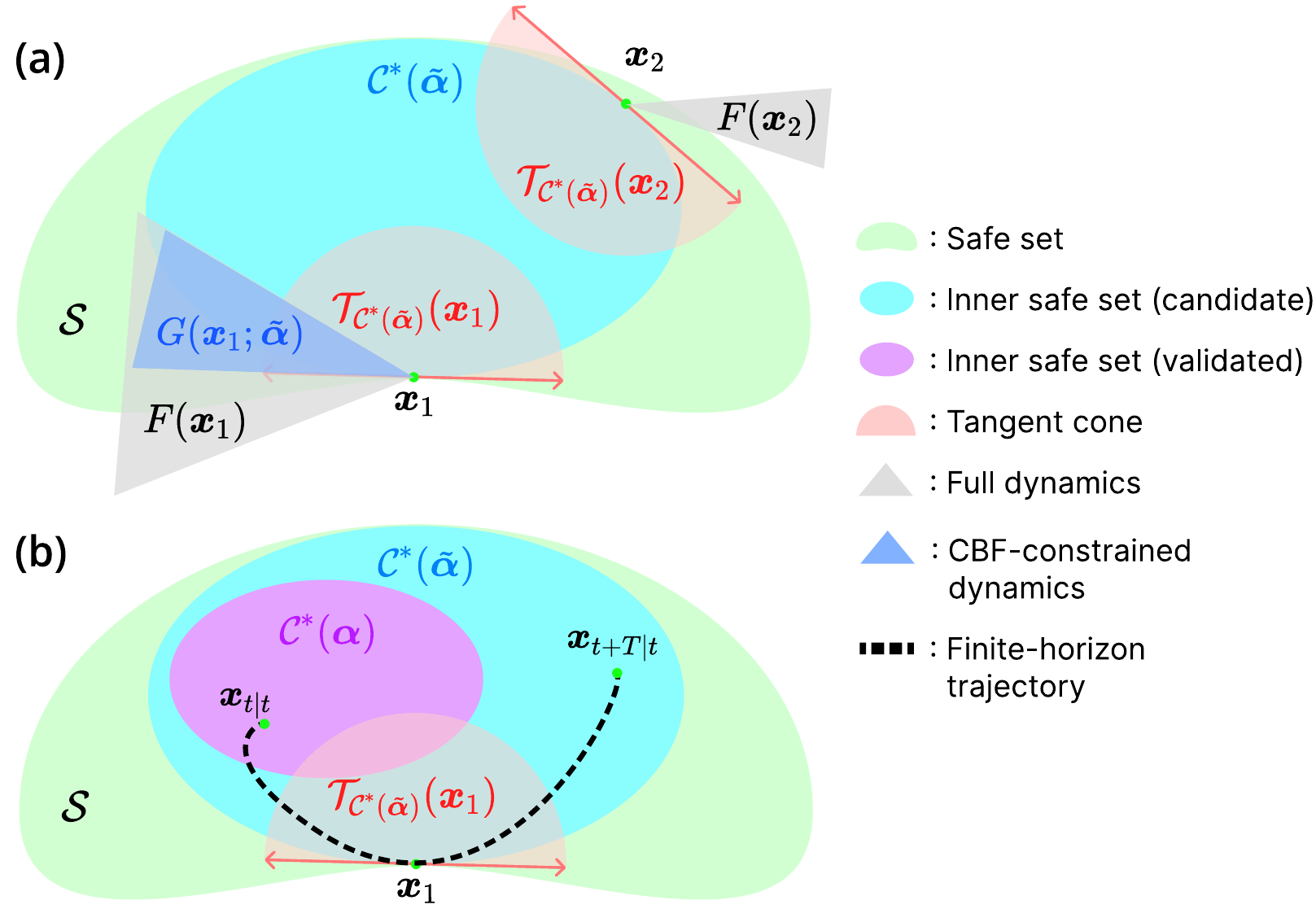}
\caption{
Conceptual illustration of the inner safe set and locally validated CBF parameter. (a) Candidate inner safe set~$\calC^{\star}(\widetilde{\bm{\alpha}})$ defined via an ICCBF with a CBF parameter~$\widetilde{\bm{\alpha}}$. At boundary point $\vx_1$, the CBF-constrained dynamics~$G(\vx_1; \widetilde{\bm{\alpha}})$ is non-empty, whereas at $\vx_2$ it is empty, indicating that the set cannot be rendered forward invariant. (b) With locally validated CBF parameters, the trajectory remains within the inner safe set shown in (a) over the finite-horizon, ensuring safety for that interval. By adapting the CBF parameters, the corresponding inner safe set is reshaped dynamically, alleviating conservatism by allowing the trajectory to extend beyond a fixed, globally verified inner safe set~$\calC^{\star}(\bm{\alpha})$.}
\label{fig:locally_validated}
\vspace{-5pt}
\end{figure}

We now formalize the CBF-parameter adaptation problem for input-constrained systems. In ICCBFs, the parameter $\bm{\alpha}=\{\alpha_1,\ldots,\alpha_r\}$ determines not only the decay rate in the safety constraint, but also the associated inner safe set $\calC^\star(\bm{\alpha})$. Therefore, adapting $\bm{\alpha}$ changes both the admissible-control set and the region that the controller can certify as safe. We refer to these extended class-$\calK_\infty$ functions collectively as the \textbf{\emph{CBF parameter}}. In the learned implementation, we use linear functions $\alpha_i(s)=\widetilde\alpha_i s$ with coefficients $\widetilde{\bm\alpha}\in\Theta\subseteq\RealSpace_{>0}^{r}$. Other parameterizations can also be used~\cite{xiao_feasibilityguided_2020}.

Let $\mathcal J=\{1,\ldots,M\}$ index the constraints and let $\calE=\{\mathbf o^j\}_{j\in\mathcal J}$ collect their parameters, with $\mathbf o^j\in\mathcal O\subseteq\RealSpace^{N_o}$. We consider differentiable constraint functions $h^j:\StateSpace\times\mathcal O\to\RealSpace$ that define
\begin{equation}
\calS=\bigcap_{j\in\mathcal J}\{\vx\in\StateSpace\mid h^j(\vx;\mathbf o^j)\ge0\}.
\label{eq:constraint-family}
\end{equation}
For example, $\mathbf o^j$ can contain an obstacle's position and size. The constraint parameters are fixed during a rollout, or their known evolution is included in the state. Applying the ICCBF construction to each constraint gives $b_i^j$ and $\calC_j^\star(\widetilde{\bm\alpha})$. For multiple constraints, we use $\calC^\star(\widetilde{\bm\alpha})\coloneqq\bigcap_{j\in\mathcal J}\calC_j^\star(\widetilde{\bm\alpha})$ and suppress the dependence on $\calE$ when it is clear from context.

\begin{remark}
Given \textbf{\emph{candidate ICCBFs}} \footnote{As with candidate CBFs in \autoref{def:candidate_cbf}, the ICCBF construction has been applied, but the ICCBF condition has not yet been verified.} $h^j$ and a specified parameter $\widetilde{\bm\alpha}$, we define, at each state $\vx\in\StateSpace$, the set of admissible control inputs satisfying all candidate ICCBF constraints as
\begin{equation}
K_{\textup{cand}}(\vx;\widetilde{\bm\alpha})
\coloneqq \left\{\vu\in\ControlSpace\mid
b_r^j(\vx,\vu;\widetilde{\bm\alpha})\geq 0,\ 
\forall j\in\mathcal J\right\}.
\label{eq:k_cand}
\end{equation}
The corresponding set of state derivatives induced by the control inputs in $K_{\textup{cand}}(\vx;\widetilde{\bm\alpha})$ is denoted by $G(\vx;\widetilde{\bm\alpha})$, as illustrated in Fig.~\ref{fig:locally_validated}(a).
For certain choices of $\widetilde{\bm\alpha}$, there may exist states $\vx\in\partial\calC^\star(\widetilde{\bm\alpha})$ at which $K_{\textup{cand}}(\vx;\widetilde{\bm\alpha})$ is empty.
At such states, no admissible control input satisfies all candidate ICCBF constraints.
\end{remark}

For comparison, define the set of \emph{globally valid} CBF parameters by
\begin{equation}
\calA_{\infty}\coloneqq\{\bm\alpha\mid K_{\textup{cand}}(\vx;\bm\alpha)\ne\varnothing,
\ \forall\vx\in\calC^\star(\bm\alpha)\}.
\label{eq:A_infty}
\end{equation}
A fixed parameter in $\calA_{\infty}$ allows an admissible CBF-based controller to preserve $\calC^\star(\bm\alpha)$ under the conditions in \autoref{subsec:iccbf}. Identifying such a parameter under input constraints can be difficult and may be overly conservative.

This paper introduces \textbf{\emph{Online Adaptive CBF~(OA-CBF)}} to address the following parameter adaptation problem:

\begin{problem}[Online Adaptive CBF]\label{prob:framework}
Given the closed-loop system~\eqref{eq:closed_loop} and a safe set $\calS$ (for which a valid ICCBF is unknown), design an online CBF parameter adaptation framework that yields an adapted parameter~$\widetilde{\bm{\alpha}}(t)$ such that a controller satisfying
\begin{equation}
\vu(t) \in K_{\textup{cand}}(\vx(t); \widetilde{\bm{\alpha}}(t)),
\end{equation}
where $K_{\textup{cand}}$ is given by \eqref{eq:k_cand}, yields a system trajectory satisfying $\vx(t) \in \calS$ for all $t \geq t_{0}$.
\end{problem}

The OA-CBF adaptation mechanism is applicable to both CBF-QP and sampled-data MPC-CBF controllers. In the CBF-QP formulation, OA-CBF replaces the fixed parameter~$\bm{\alpha}$ in~\eqref{eq:cbf-qp-constraint} with the adapted parameter~$\widetilde{\bm{\alpha}}(t)$. In the MPC-CBF formulation, the same replacement is made in the stage constraint; the discrete-time formulation is provided in \discreteappendix.

To address Problem~\ref{prob:framework}, we introduce a tractable, state-dependent approach based on \emph{locally validated CBF parameters} (see Fig.~\ref{fig:locally_validated}(b)). Rather than requiring a globally valid CBF parameter, we validate a candidate parameter over a finite prediction horizon $[t,t+T]$, where $T>0$. We denote by $\vx_{t+\tau|t}$ the predicted closed-loop trajectory of System~\eqref{eq:closed_loop} from the current state $\vx(t)$, with the candidate parameter $\widetilde{\bm{\alpha}}(t)$ held fixed over the prediction horizon:
\begin{equation}
\begin{aligned}
\frac{\mathrm d}{\mathrm d\tau}\vx_{t+\tau|t}
&=f(\vx_{t+\tau|t})+g(\vx_{t+\tau|t})\vu_{t+\tau|t},\\
\vu_{t+\tau|t}&=\pi(\vx_{t+\tau|t};\widetilde{\bm{\alpha}}(t)),
\qquad \vx_{t|t}=\vx(t),
\end{aligned}
\label{eq:rollout}
\end{equation}
for $\tau\in[0,T]$.

\begin{definition}[Locally Validated CBF Parameter\footnote{It was also termed a locally valid parameter in \cite{kim_learning_2025}.}]
\label{def:locally-validated}
Let $t \ge t_0$ and $T>0$. A CBF parameter $\widetilde{\bm{\alpha}}(t)$ is said to be \emph{locally validated} over the interval $[t,t+T]$ if:
\begin{enumerate}[leftmargin=5pt, labelindent=5pt, labelsep=0.0em, itemindent=*, align=left]
    \item the current state belongs to the inner safe set,
    \begin{equation}
    \vx_{t|t} \in \calC^{\star}(\widetilde{\bm{\alpha}}(t)),
    \label{eq:local-state-membership}
    \end{equation}

    \item for the rollout~\eqref{eq:rollout} $\vx_{t+\tau|t}$ generated under a control policy satisfying
    $\pi(\vx_{t+\tau|t}; \widetilde{\bm{\alpha}}(t)) \in K_{\textup{cand}}(\vx_{t+\tau|t};\widetilde{\bm{\alpha}}(t))$,
    the tangent-cone condition
    \begin{equation}
    f(\vx_{t+\tau|t}) + g(\vx_{t+\tau|t})\pi(\vx_{t+\tau|t}; \widetilde{\bm{\alpha}}(t))
    \in \calT_{\calC^{\star}(\widetilde{\bm{\alpha}}(t))}(\vx_{t+\tau|t})
    \label{eq:locally-validated}
    \end{equation}
    is satisfied for all $\tau \in [0,T]$.
\end{enumerate}
Similar to \autoref{thm:nagumo}, the condition in \eqref{eq:locally-validated} is only non-trivial on $\partial \calC^{\star}(\widetilde{\bm{\alpha}}(t))$.
\end{definition}

The following theorem states the safety property of OA-CBF under continued local validation.

\begin{theorem}[Safety under Continued Local Validation]
\label{thm:local-safety}
Consider System~\eqref{eq:closed_loop} and a fixed prediction horizon $T>0$. Let $t_0<t_1<\cdots$ be update times satisfying $t_k\to\infty$ as $k\to\infty$ and $t_{k+1}-t_k\le T$. Suppose that, at every update $t_k$, $k\ge0$, the selected parameter $\widetilde{\bm\alpha}_k$ is locally validated over $[t_k,t_k+T]$ according to \autoref{def:locally-validated}, with the rollout initialized at $\vx(t_k)$. If each $\widetilde{\bm\alpha}_k$ is held fixed on $[t_k,t_{k+1})$ and the control law $\vu(t)=\pi(\vx(t);\widetilde{\bm\alpha}_k)$ used in~\eqref{eq:rollout} is applied over this interval, then
\begin{equation}
\vx(t)\in\calS,\qquad \forall t\ge t_0.
\label{eq:local-safety}
\end{equation}
\end{theorem}

\begin{proof}
Since $t_k\to\infty$, any $t\ge t_0$ belongs to an update interval $[t_k,t_{k+1})$ for some $k\ge0$. Local validation at $t_k$ and the ICCBF inequalities~\eqref{eq:iccbf-recursion}--\eqref{eq:iccbf-constraint} ensure that the predicted closed-loop trajectory remains in $\calC^\star(\widetilde{\bm\alpha}_k)\subseteq\calS$ throughout $[t_k,t_k+T]$. Since the same control law is used for prediction and applied over
$[t_k,t_{k+1})\subseteq[t_k,t_k+T]$, it follows that
$\vx(t)\in\calS$ for all $t\in[t_k,t_{k+1}]$ and all $k\ge0$.
\end{proof}

Compared with choosing a fixed parameter in $\calA_\infty$, local validation allows the inner safe set to be reshaped around the current operating state, as illustrated in Fig.~\ref{fig:locally_validated}(b), thereby reducing the conservatism of relying on a single globally valid parameter. The remaining challenge is to efficiently identify locally validated CBF parameters at runtime:
\begin{problem}[Identifying Locally Validated CBF Parameters]
\label{prob:identifying}
At time $t$, given the current state $\vx_{t|t}$ and candidate ICCBFs, identify parameters from the locally validated set associated with the rollout and admissible policy in \autoref{def:locally-validated}:
\begin{equation}
\begin{aligned}
\calA(t)\coloneqq\bigl\{\widetilde{\bm\alpha}\in\Theta\mid {}&
\vx_{t|t}\in\calC^\star(\widetilde{\bm\alpha})\text{ and }\\
&\eqref{eq:locally-validated}\text{ holds over }[t,t+T]\bigr\}.
\end{aligned}
\label{eq:A_t}
\end{equation}
\end{problem}

\begin{figure}[!t]
\centering
\includegraphics[width=0.99\linewidth]{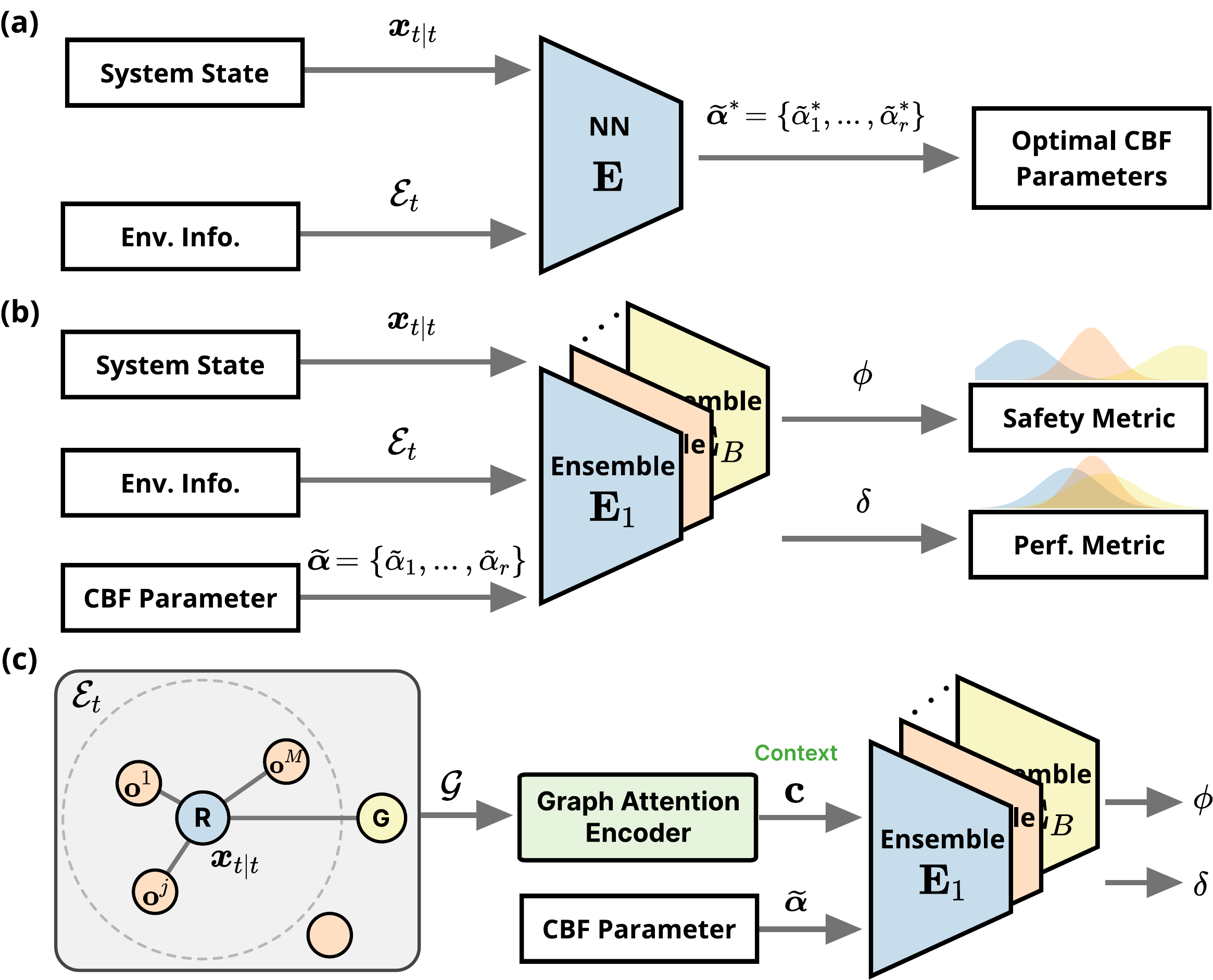}
\caption{
Neural network configurations for CBF parameter adaptation. (a) Direct-prediction approaches map the system state and environment information to a single adapted CBF parameter $\widetilde{\bm{\alpha}}^{\star}$.
(b) The proposed method instead takes a queried candidate CBF parameter $\widetilde{\bm{\alpha}}$ as input and predicts safety and performance metrics $(\phi,\delta)$ for that candidate.
(c) OA-CBF with graph attention first encodes the robot-goal-obstacle graph $\mathcal G$ into a context feature $\vc$, then evaluates the concatenated input $[\vc^\top,\widetilde{\bm{\alpha}}^\top]^\top$ using the PENN.
}
\label{fig:nn-config}
\label{fig:nn}
\vspace{-5pt}
\end{figure}

\subsection{Limitations of Direct Parameter Prediction \label{sec:basic}}

The core idea behind OA-CBF is to identify \emph{online} locally validated CBF parameters $\widetilde{\bm{\alpha}}(t_{k})$, $k \in \NaturalNumber \cup \{0\}$. A na\"{i}ve approach is to resort to an exhaustive search by forward propagating the closed-loop system trajectory over all possible parameters; however, this method is computationally intractable. An alternative is to collect data samples offline and use prediction models for online estimation. Neural networks can approximate the dependence of local validity on the state, constraint parameters, and CBF parameter while permitting efficient online evaluation.

One approach is to train a neural network to predict CBF parameters online and then directly adapt those parameters. Let $\widebar{X}$ denote an input representation of the system state $\vx_{t|t}$, the task goal, and the collection of constraint parameters $\calE_t=\{\mathbf o_t^j\}_{j\in\mathcal J}$ at time $t$ (e.g., obstacle positions and sizes in a collision avoidance task). The output $\widebar{Y} = \widetilde{\bm{\alpha}}^{\star}$ represents the predicted CBF parameter. Then, a deterministic neural network~$\vE$, parameterized by $\bm{\theta}$, defines the mapping (see Fig.~\ref{fig:nn}(a)):
\begin{equation} \label{eq:prior}
\widetilde{\bm{\alpha}}^{\star} = \vE(\widebar{X};\bm{\theta}) .
\end{equation}

The approaches in~\cite{ma_learning_2022, xiao_barriernet_2023, gao_online_2023, berducci_learning_2024, mohammad_soft_2025} fall into this category and have demonstrated promising results in reducing conservatism through online adaptation of CBF parameters using neural networks. However, a deterministic parameter prediction alone provides no assessment of the uncertainty in that prediction. In particular, the predicted CBF parameter $\widetilde{\bm{\alpha}}^{\star}$ in~\eqref{eq:prior} may not belong to the locally validated set $\calA(t)$.

\subsection{Identification under Predictive Uncertainty}
\label{subsec:uncertainty-problems}
For a learned approach to \autoref{prob:identifying}, two sources of uncertainty must be distinguished~\cite{kendall_what_2017,chua_deep_2018}. \emph{Epistemic uncertainty} arises from limited knowledge of the prediction model, particularly for queried states, constraint parameters, or CBF parameters poorly represented in the training data. Such queries may be out-of-distribution (OOD). \emph{Aleatoric uncertainty} describes variability in the predicted outcome that can remain for well-represented inputs, for example due to observation noise.

Let $X$ denote an input to the prediction model. We formulate the following two subproblems.

\begin{subproblem}[Identification under Epistemic Uncertainty]
\label{prob:epistemic}
Given a trained prediction model and held-out in-distribution (ID) trajectories, identify and reject potentially out-of-distribution (OOD) inputs while limiting the probability of incorrectly rejecting an ID trajectory, i.e.,
\begin{equation}
\Prob\left(\mathcal X_{\textup{test}}\subseteq
\mathcal X_{\textup{ID}}\right)
\geq 1-\delta_{\textup{cal}},
\label{eq:problem-epistemic}
\end{equation}
where $\mathcal X_{\textup{test}}$ denotes the set of inputs from an ID test trajectory, $\mathcal X_{\textup{ID}}$ denotes the set of inputs classified as ID, and $\delta_{\textup{cal}}\in(0,1)$ is a prescribed tolerance.
\end{subproblem}

\begin{subproblem}[Identification under Aleatoric Uncertainty]
\label{prob:aleatoric}
Among the candidates retained by \autoref{prob:epistemic}, identify CBF parameters predicted to satisfy the local-validation conditions in \autoref{def:locally-validated} with probability at least $1-\varepsilon$, i.e.,
\begin{equation}
\inf_{P\in\mathfrak P(X)}
\Prob^{P}\left(\widehat{\mathcal V}(X)\right)
\geq 1-\varepsilon,
\label{eq:problem-aleatoric}
\end{equation}
where $\widehat{\mathcal V}(X)$ denotes the event that the predicted outcome satisfies a sufficient condition for local validation, $\mathfrak P(X)$ denotes the family of predictive distributions associated with $X$, and $\varepsilon\in(0,1)$.
\end{subproblem}

\section{PROBABILISTIC VERIFICATION OF CBF PARAMETERS}
\label{sec:offline}
\label{sec:online}

To address \autoref{prob:epistemic} and \autoref{prob:aleatoric}, we propose a probabilistic verification method for queried CBF parameters, as illustrated in Fig.~\ref{fig:nn-config}(b)--(c). Instead of outputting a single CBF parameter $\widetilde{\bm{\alpha}}^{\star}$ as in~\eqref{eq:prior}, the model evaluates queried parameters by predicting distributions over a safety metric $\phi$ and a performance metric $\delta$. The key novelty is that this approach enables us to assess the local validity of candidate CBF parameters while explicitly accounting for uncertainty in neural network predictions. This shifts the burden of identifying locally validated parameters from direct prediction by the trained model to a verification process that accounts for predictive uncertainty. We first introduce the prediction model and define the safety and performance metrics, followed by the verification methods under epistemic and aleatoric uncertainty.

\subsection{Probabilistic Ensemble Neural Network Model}
\label{subsec:graph}
\label{subsec:penn}
The graph-attention encoder maps the system state, task goal, and constraint parameters $\calE$ to a context vector $\vc\in\RealSpace^{n_c}$. It aggregates a variable number of constraint features, allowing the same predictor to evaluate scenarios with different numbers of observed constraints without selecting a fixed number of neighbors. The candidate CBF parameter remains an explicit input to the model through
\begin{equation}
X\coloneqq[\vc^\top,\widetilde{\bm\alpha}^\top]^\top\in\RealSpace^{n_c+r}.
\label{eq:network_embedding}
\end{equation}
A shared context can therefore be used to evaluate a batch of candidates. In the navigation experiments, the graph represents the robot state, goal, and observed obstacles; its construction and attention mechanism are given in Appendix~\appendixref{app:graph}{\mbox  {A-A}}.

Given the embedding $X$, we employ a probabilistic ensemble neural network~(PENN), denoted by $\vF$, to predict the random output $\widehat Y=[\widehat\phi,\widehat\delta]^\top$. These quantities predict the safety and performance metrics defined in \autoref{subsec:safety-loss}. Each of its $B$ members specifies a Gaussian predictive law
\begin{equation}
\vE_b(X;\bm\theta_b)=\mathcal N\!\left(\vmu_{\bm\theta_b}(X),\VSigma_{\bm\theta_b}(X)\right),
\quad b=1,\ldots,B,
\label{eq:gaussian}
\end{equation}
with mean $\vmu_{\bm\theta_b}(X)\in\RealSpace^2$ and positive diagonal covariance $\VSigma_{\bm\theta_b}(X)\in\RealSpace^{2\times2}$. Uniform weighting of the ensemble $\mathfrak E=\{\vE_1,\ldots,\vE_B\}$ yields
\begin{equation}
\widehat Y\sim\vF(X;\bm\theta_{1:B})
=\frac{1}{B}\sum_{b=1}^{B}\vE_b(X;\bm\theta_b).
\label{eq:penn}
\end{equation}
The family $\mathfrak P(X)$ in \autoref{prob:aleatoric} consists of the safety-loss marginals of these member distributions.

Disagreement among members provides the epistemic score, while each member's output variance describes the aleatoric variability represented by the model~\cite{chua_deep_2018,kendall_what_2017}. The resulting Gaussian mixture~\eqref{eq:penn} supports both verification rules below. Network architecture and training details are provided in Appendices~\appendixref{app:penn}{\mbox  {A-B}} and~\appendixref{app:training}{\mbox  {A-D}}.

\subsection{Characterization of Safety and Performance Metrics \label{subsec:safety-loss}}
Let $\mathcal D_T\subseteq\StateSpace\times\Theta$ contain the state--parameter pairs for which the rollout~\eqref{eq:rollout} exists with admissible inputs over $[t,t+T]$. We use a safety-loss function and a performance-cost function
\begin{equation}
\phi_T:\mathcal D_T\to\RealSpace_{\ge0},
\qquad
\delta_T:\mathcal D_T\to[0,T],
\label{eq:metric-maps}
\end{equation}
whose definitions are given below in~\eqref{eq:finite-horizon-risk} and~\eqref{eq:performance-metric}. Their dependence on the constraint parameters, task goal, and control policy is suppressed. For complete admissible rollouts, the training labels are the evaluated values $Y=[\phi_T,\delta_T]^\top$; the random outputs $\widehat\phi$ and $\widehat\delta$ of~\eqref{eq:penn} predict these labels. The safety threshold $\widebar\phi$ is chosen so that $\phi_T(\vx_t;\widetilde{\bm\alpha})\le\widebar\phi$, together with~\eqref{eq:local-state-membership}, implies local validation; \autoref{prop:loss-certificate} establishes this implication.

For a fixed update state, define the \emph{safety margin} $\sigma$ as the infimum of the ICCBF constraint values along~\eqref{eq:rollout}:
\begin{equation}
\sigma(\widetilde{\bm{\alpha}}) \coloneqq
\inf_{\tau\in[0,T]}\min_{j\in\mathcal J}
b_r^j(\vx_{t+\tau|t},\vu_{t+\tau|t};\widetilde{\bm{\alpha}}).
\label{eq:safety_margin}
\end{equation}
Here, $\mathcal J$ is the constraint index set in~\eqref{eq:constraint-family}. Thus, $\sigma(\widetilde{\bm{\alpha}})\ge0$ enforces all candidate ICCBF inequalities throughout the prediction horizon. Together with~\eqref{eq:local-state-membership}, this is sufficient for local validation.

To construct a scalar learning target, define, for a fixed candidate $\widetilde{\bm\alpha}$, the function $\lambda^j:\StateSpace\times\ControlSpace\to\RealSpace_{>0}$ by
\begin{equation}
\lambda^j(\vx,\vu)\coloneqq
\lambda_1\exp\!\left(-\lambda_2 b_r^j(\vx,\vu;\widetilde{\bm\alpha})\right),
\label{eq:lambda_j}
\end{equation}
where $\lambda_1,\lambda_2>0$. Let $D_j:\StateSpace\to[1,\infty)$ be a weighting function for constraint $j$. It can encode application-specific proximity and direction information; the obstacle-avoidance construction is given in Appendix~\appendixref{app:safety-weighting}{\mbox  {A-C}}. Define the safety-loss density $\Phi:\StateSpace\times\ControlSpace\times\Theta\to\RealSpace_{\ge0}$ as
\begin{equation}
\Phi(\vx,\vu;\widetilde{\bm\alpha})
\coloneqq\max_{j\in\mathcal J}\max\!\left\{
\lambda^j(\vx,\vu),\,
\eta\frac{\lambda^j(\vx,\vu)}{D_j(\vx)}\right\},
\quad\eta\ge1.
\label{eq:safety_density}
\end{equation}
The first term ensures that every negative ICCBF margin yields a loss greater than $\lambda_1$. The second term permits additional weighting without weakening this implication. For an empty constraint set, we use the conventions $\Phi=0$ and $\sigma=+\infty$.

The finite-horizon safety loss is the function $\phi_T$ in~\eqref{eq:metric-maps}, evaluated along~\eqref{eq:rollout}:
\begin{equation}
\phi_T(\vx_t;\widetilde{\bm\alpha})\coloneqq
\sup_{\tau\in[0,T]}
\Phi(\vx_{t+\tau|t},\vu_{t+\tau|t};\widetilde{\bm\alpha}).
\label{eq:finite-horizon-risk}
\end{equation}

The following result states the certificate property of this loss.
\begin{proposition}[Sufficient Finite-Horizon Certificate]
\label{prop:loss-certificate}

Consider the admissible rollout~\eqref{eq:rollout} of System~\eqref{eq:ct_dynamics} under $\pi(\cdot;\widetilde{\bm\alpha})$ over $[t,t+T]$. Suppose that~\eqref{eq:local-state-membership} holds.

 Then
\begin{equation}
\phi_T(\vx_t;\widetilde{\bm{\alpha}})
\leq \widebar{\phi}\coloneqq\lambda_1
\quad\Longrightarrow\quad
\sigma(\widetilde{\bm{\alpha}})\geq0.
\label{eq:condition2}
\end{equation}
This condition therefore ensures that the candidate parameter $\widetilde{\bm{\alpha}}$ satisfies \autoref{def:locally-validated}.
\end{proposition}

\begin{proof}
The threshold condition yields the following implication chain, where the intermediate inequalities hold for %
every $j\in\mathcal J$ and every $\tau\in[0,T]$:
\begin{align*}
&\, \phi_T(\vx_t;\widetilde{\bm{\alpha}})
\leq \widebar{\phi}=\lambda_1
\\
\overset{\text{(a)}}{\Longrightarrow}
&\, \lambda^j(
\vx_{t+\tau|t},
\vu_{t+\tau|t})
\leq \lambda_1
\\
\overset{\text{(b)}}{\Longrightarrow}
&\, \lambda_1
\exp\!\left(
-\lambda_2 b_r^j(
\vx_{t+\tau|t},
\vu_{t+\tau|t};
\widetilde{\bm{\alpha}})
\right)
\leq \lambda_1
\\
\overset{\text{(c)}}{\Longrightarrow}
&\, b_r^j(
\vx_{t+\tau|t},
\vu_{t+\tau|t};
\widetilde{\bm{\alpha}})
\geq 0
\\
\overset{\text{(d)}}{\Longrightarrow}
&\, \sigma(\widetilde{\bm{\alpha}})\geq0.
\end{align*}
Here, (a) follows because, for %
every $j\in\mathcal J$ and every $\tau\in[0,T]$,
\begin{align*}
\lambda^j(\vx_{t+\tau|t},\vu_{t+\tau|t})
&\leq
\Phi(\vx_{t+\tau|t},\vu_{t+\tau|t};
\widetilde{\bm{\alpha}})
\\
&\leq
\phi_T(\vx_t;\widetilde{\bm{\alpha}}).
\end{align*}
The first inequality follows from the maximum in~\eqref{eq:safety_density}, and %
the second follows from~\eqref{eq:finite-horizon-risk}.
Step~(b) substitutes~\eqref{eq:lambda_j}.
Step~(c) uses $\lambda_1,\lambda_2>0$ and the strict monotonicity of the exponential function.
Finally, (d) follows from the definition of $\sigma$ in~\eqref{eq:safety_margin} %
as the infimum of the ICCBF margins over %
all constraints and the prediction horizon.

Together with~\eqref{eq:local-state-membership}, the recursive ICCBF inequalities~\eqref{eq:iccbf-recursion}--\eqref{eq:iccbf-constraint} preserve $b_{r-1}^j,\ldots,b_0^j\ge0$ over $[t,t+T]$, as in the proof of \autoref{thm:local-safety}. The predicted trajectory therefore remains in $\calC^\star(\widetilde{\bm\alpha})$~\eqref{eq:inner-safe-set} and satisfies~\eqref{eq:locally-validated}.

\end{proof}

For performance, let $v_{\mathrm{prog}}(\vx)\ge0$ denote the task-relevant translational speed and let $v_{\mathrm{min}}>0$ be the low-speed threshold. Let $T_{\mathrm g}\le T$ be the elapsed time from the rollout start to first arrival at the goal region, or $T$ if the goal is not reached within the horizon. We define the accumulated low-speed duration by
\begin{equation}
\delta_T(\vx_t;\widetilde{\bm\alpha})\coloneqq
\int_0^{T_{\mathrm g}}\!\mathbf 1_{\{v_{\mathrm{prog}}(\vx_{t+\tau|t})<v_{\mathrm{min}}\}}\,\mathrm d\tau.
\label{eq:performance-metric}
\end{equation}
Here, $\mathbf 1$ is the indicator function. Lower values of $\delta_T$ indicate less time spent in slow motion or deadlock. The PENN~\eqref{eq:penn} is trained on these metrics as described in Appendix~\appendixref{app:training}{\mbox  {A-D}}.

The trained model~\eqref{eq:penn} is used below to verify candidate parameters under the uncertainty requirements in \autoref{prob:epistemic} and \autoref{prob:aleatoric}.

\subsection{Verification Under Epistemic Uncertainty \label{subsec:eu}}

We evaluate epistemic uncertainty in the PENN predictions~\eqref{eq:penn} by quantifying disagreement among the ensemble members $\mathfrak E=\{\vE_1,\ldots,\vE_B\}$. With equal ensemble weights, the quadratic Jensen-R\'enyi divergence~(JRD) has the closed-form expression~\cite{wang_closedform_2009,kim_bridging_2023}
\begin{equation}
\begin{aligned}
D(X;\vE_{1:B})
&=-\log\left(\frac{1}{B^2}\sum_{b=1}^{B}\sum_{c=1}^{B}\mathfrak D_{bc}\right)\\
&\quad+\frac{1}{B}\sum_{b=1}^{B}\log\mathfrak D_{bb},
\end{aligned}
\label{eq:jrd}
\end{equation}
where
\[
\mathfrak D_{bc}
\coloneqq\frac{1}{|\bm{\mathfrak V}_{bc}|^{1/2}}
\exp\!\left(-\frac{1}{2}\bm\Delta_{bc}^{\top}
\bm{\mathfrak V}_{bc}^{-1}\bm\Delta_{bc}\right),
\]
$\bm{\mathfrak V}_{bc}\coloneqq\VSigma_{\bm\theta_b}(X)+\VSigma_{\bm\theta_c}(X)$,
$\bm\Delta_{bc}\coloneqq\vmu_{\bm\theta_b}(X)-\vmu_{\bm\theta_c}(X)$,
and $|\bm{\mathfrak V}_{bc}|$ denotes the determinant. In essence, the JRD~$D(X) \in \RealSpace$ measures how \emph{``spread out”} or inconsistent the ensemble predictions are. We write $D(X)$ when the trained ensemble is fixed.

We retain a candidate prediction only when its ensemble disagreement is below the calibrated threshold $D_{\textup{thr}}$. For an embedded input $X$~\eqref{eq:network_embedding}, the epistemic gate is
\begin{equation}
D\left(X;\vE_{1:B}\right) \le D_{\textup{thr}}.
\label{eq:epistemic_filter}
\end{equation}
If this condition is violated, the candidate prediction is discarded as epistemically unreliable. This gate screens queries using ensemble disagreement before the predicted risk is assessed. The calibration below specifies its ID recall under an exchangeability assumption.

\subsection{Conformal Calibration of the Epistemic Gate \label{subsec:cp}}

The threshold $D_{\textup{thr}}$ determines the tradeoff between rejecting ID predictions and retaining potentially unreliable queries. Existing approaches often choose such thresholds heuristically or tune them as hyperparameters~\cite{yang_generalized_2024, yang_rtdiff_2024, qing_bitrajdiff_2025, sun_socially_2025, kim_how_2025, kim_learning_2025}. To address \autoref{prob:epistemic}, we calibrate the gate using held-out ID trajectories without requiring representative OOD examples. Conformal prediction provides a finite-sample ID recall guarantee for this one-class screening problem~\cite{angelopoulos_conformal_2023, ding_classconditional_2023}.

Let $\calD_{\textup{calib}}=\{\uptau_i\}_{i=1}^{N_{\textup{traj}}}$ be a calibration dataset not used to train the PENN~\eqref{eq:penn}. Each $\uptau_i$ contains a state rollout~\eqref{eq:rollout} and the associated constraint parameters, task goal, and candidate CBF parameters used to construct the model inputs~\eqref{eq:network_embedding}. The trajectories are generated independently by sampling the initial system state, constraint parameters $\calE$, and candidate CBF parameters. States within a trajectory are temporally dependent, so complete trajectories are the units of exchangeability~\cite{ren_robots_2023}.

For each trajectory, we evaluate $D(X_{\tau,i})$ over $\tau\in[0,T]$ and define the trajectory score
\begin{equation}
Q_i\coloneqq\sup_{\tau\in[0,T]}D(X_{\tau,i}).
\label{eq:trajectory-score}
\end{equation}
The calibration and test scores use the same trained model and evaluation rule. As required in \autoref{prob:epistemic}, the test trajectory must be exchangeable with the calibration trajectories conditional on the trained model.

Let $Q_{(1)}\le\cdots\le Q_{(N_{\textup{traj}})}$ denote the nonconformity scores of the calibration trajectories, sorted in ascending order. For a prescribed $\delta_{\textup{cal}}\in(0,1)$, the conformal threshold is defined as
\begin{equation}
\begin{aligned}
k_{\textup{cal}}&\coloneqq
\left\lceil(N_{\textup{traj}}+1)(1-\delta_{\textup{cal}})\right\rceil,\\
D_{\textup{thr}}&\coloneqq Q_{(k_{\textup{cal}})},
\end{aligned}
\label{eq:d_thr_calibration}
\end{equation}
with the convention $Q_{(N_{\textup{traj}}+1)}=+\infty$.

\begin{lemma}[In-Distribution Recall Guarantee]
\label{lem:epistemic}
\label{thm:cp_guarantee}
For the trained prediction model~\eqref{eq:penn}, suppose that the calibration trajectories and an ID test trajectory $\uptau_{\textup{test}}$ are exchangeable and their scores are computed by~\eqref{eq:trajectory-score}. Then the threshold~\eqref{eq:d_thr_calibration} satisfies
\begin{equation}
\Prob\!\left(Q_{\textup{test}}\le D_{\textup{thr}}\right)
\ge1-\delta_{\textup{cal}},
\label{eq:cp_guarantee}
\end{equation}
where $Q_{\textup{test}}=\sup_{\tau\in[0,T]}D(X_{\tau,\textup{test}})$. Consequently, the gate~\eqref{eq:epistemic_filter} satisfies~\eqref{eq:problem-epistemic}, the trajectory-level ID recall requirement of \autoref{prob:epistemic}.
\end{lemma}

\begin{proof}
For $k_{\textup{cal}}\leq N_{\textup{traj}}$, the definition of the empirical quantile implies that $Q_{\textup{test}}\leq D_{\textup{thr}}$ whenever $Q_{\textup{test}}$ is among the $k_{\textup{cal}}$ smallest scores in $\{Q_1,\ldots,Q_{N_{\textup{traj}}},Q_{\textup{test}}\}$, including all scores tied at the cutoff. At least $k_{\textup{cal}}$ of these $N_{\textup{traj}}+1$ scores satisfy this criterion. Exchangeability therefore gives
\[
\Prob\!\left(Q_{\textup{test}}\leq D_{\textup{thr}}\right)
\geq
\frac{k_{\textup{cal}}}{N_{\textup{traj}}+1}
\geq 1-\delta_{\textup{cal}},
\]
which proves~\eqref{eq:cp_guarantee}. If $k_{\textup{cal}}=N_{\textup{traj}}+1$, then $D_{\textup{thr}}=+\infty$ and the claim holds directly.
\end{proof}

\autoref{lem:epistemic} ensures that the JRD does not exceed $D_{\textup{thr}}$ throughout an exchangeable ID trajectory with marginal probability at least $1-\delta_{\textup{cal}}$.

\subsection{Verification Under Aleatoric Uncertainty \label{subsec:au}}

For prediction results that pass the epistemic filter, we next examine the aleatoric uncertainty in the predicted safety metric. For notational simplicity, we focus only on the first output component of $Y=[\phi,\delta]^\top$, namely the safety metric~$\phi$, and denote its prediction by~$\widehat{\phi}$; the performance metric~$\delta$ will be used later in \autoref{subsec:select}. For the $b$-th ensemble member, let $ \mu^{\phi}_{\bm{\theta}_{b}}(X)$ and $ \sigma^{\phi}_{\bm{\theta}_{b}}(X)$ denote the mean and standard deviation of this scalar prediction. Then, under $\vE_b(X)$, the marginal prediction of~$\widehat{\phi}$ is
\begin{equation}
\widehat{\phi} \sim \mathcal{N} \left( \mu^{\phi}_{\bm{\theta}_{b}}(X), \bigl( \sigma^{\phi}_{\bm{\theta}_{b}}(X) \bigr)^2 \right).
\label{eq:phi-marginal}
\end{equation}

Recall from~\eqref{eq:condition2} that the deterministic condition $\phi_T(\vx_t;\widetilde{\bm{\alpha}})\leq\widebar{\phi}$, with $\widebar{\phi}=\lambda_1$, is sufficient for the finite-horizon ICCBF margin condition $\sigma(\widetilde{\bm{\alpha}})\geq0$ along the predicted trajectory. Since larger values of $\phi$ correspond to higher predicted risk, the relevant aleatoric uncertainty lies in the right tail of $\widehat{\phi}$. We therefore employ the right-tail Value at Risk~(VaR) and Conditional Value at Risk~(CVaR)~\cite{cai_evora_2024, rockafellar_optimization_2000}.

For a single ensemble member's prediction, the right-tail VaR with tail probability $\varepsilon\in(0,1)$ is defined as

\begin{equation}
\textnormal{VaR}_{\varepsilon}^{\vE_b(X)}\left(\widehat{\phi}\right)
\coloneqq
\inf_{\nu \in \RealSpace}
\left\{
\nu \mid \textnormal{Prob}^{\vE_b(X)} \left(\widehat{\phi} \le \nu\right) \ge 1-\varepsilon
\right\}. \nonumber
\label{eq:right-var}
\end{equation}
Thus, $\textnormal{VaR}_{\varepsilon}^{\vE_b(X)}(\widehat{\phi})$ is the $(1-\varepsilon)$-quantile of the predicted risk level.

Building on VaR, we define the corresponding right-tail CVaR as
\begin{equation}
\textnormal{CVaR}_{\varepsilon}^{\vE_b(X)} \left(\widehat{\phi}\right)
\coloneqq
\inf_{\eta \in \RealSpace}
\left\{
\eta + \frac{1}{\varepsilon}\mathbb{E}_{\vE_b(X)}
\left[(\widehat{\phi}-\eta)^+\right]
\right\}, \nonumber
\label{eq:right-cvar}
\end{equation}
where $(\cdot)^+ \coloneqq \max\{\cdot,0\}$. The right-tail CVaR represents the expected value of the worst $\varepsilon$-fraction of predicted risk levels, and therefore accounts for both tail probability and tail severity.

For the Gaussian prediction in \eqref{eq:phi-marginal}, both quantities admit closed-form expressions:
\begin{equation}
\textnormal{VaR}_{\varepsilon}^{\vE_b(X)} \left(\widehat{\phi}\right)
=
 \mu^{\phi}_{\bm{\theta}_{b}}(X)
+
 \sigma^{\phi}_{\bm{\theta}_{b}}(X)\,F_{\mathcal N}^{-1}(1-\varepsilon), \nonumber
\label{eq:right-var-gaussian}
\end{equation}
\begin{equation}
\textnormal{CVaR}_{\varepsilon}^{\vE_b(X)} \left(\widehat{\phi}\right)
=
 \mu^{\phi}_{\bm{\theta}_{b}}(X)
+
 \sigma^{\phi}_{\bm{\theta}_{b}}(X)\,
\frac{f_{\mathcal N} \left(F_{\mathcal N}^{-1}(1-\varepsilon)\right)}{\varepsilon}, \nonumber
\label{eq:right-cvar-gaussian}
\end{equation}
where $F_{\mathcal N}$ and $f_{\mathcal N}$ denote the cumulative distribution function and probability density function of the standard normal distribution, respectively.

To robustly account for aleatoric uncertainty, we treat the ensemble set~$\mathfrak{E}=\{\vE_{1},\ldots,\vE_{B}\}$ as an ambiguity set and impose the following distributionally robust CVaR constraint~\cite{rahimian_distributionally_2022}:
\begin{equation}
\sup_{\vE_b \in \mathfrak{E}}
\textnormal{CVaR}_{\varepsilon}^{\vE_b(X)}\left(\widehat{\phi}\right)
\le \widebar{\phi}.
\label{eq:distributionally-robust-cvar}
\end{equation}
This condition requires the most conservative ensemble member to predict a right-tail risk no larger than the sufficient threshold~$\widebar{\phi}$.

\begin{lemma}[Probabilistic Verification under Aleatoric Uncertainty]
\label{lem:aleatoric}
Consider a candidate CBF parameter $\widetilde{\bm{\alpha}}$ satisfying~\eqref{eq:local-state-membership}. If the distributionally robust CVaR condition~\eqref{eq:distributionally-robust-cvar} holds, then
\begin{equation}
\inf_{\vE_b\in\mathfrak{E}}
\textnormal{Prob}^{\vE_b(X)}\!
\left(\widehat{\phi}\leq\widebar{\phi}\right)
\geq1-\varepsilon.
\label{eq:prob-cvar-implication}
\end{equation}
Consequently, the CVaR condition is sufficient to satisfy~\eqref{eq:problem-aleatoric}, the predictive risk requirement of \autoref{prob:aleatoric}.
\end{lemma}

\begin{proof}
For every ensemble member $\vE_b\in\mathfrak{E}$, the relation between VaR and CVaR together with~\eqref{eq:distributionally-robust-cvar} gives
\[
\textnormal{VaR}_{\varepsilon}^{\vE_b(X)}(\widehat{\phi})
\leq
\textnormal{CVaR}_{\varepsilon}^{\vE_b(X)}(\widehat{\phi})
\leq \widebar{\phi}.
\]
By the definition of VaR, this implies
\[
\textnormal{Prob}^{\vE_b(X)}\!
\left(\widehat{\phi}\leq\widebar{\phi}\right)
\geq1-\varepsilon.
\]
Taking the infimum over $\vE_b\in\mathfrak{E}$ proves~\eqref{eq:prob-cvar-implication}.
\end{proof}

\autoref{lem:aleatoric} ensures that the predicted safety loss does not exceed $\widebar{\phi}$ with probability at least $1-\varepsilon$ under every ensemble member.

Let $\calA_{\textup{cand}}$ denote the set of sampled CBF parameters satisfying~\eqref{eq:local-state-membership}. Applying the epistemic gate and the distributionally robust CVaR condition yields the retained candidate set
\begin{equation}
\begin{aligned}
\widehat{\calA}
=\bigl\{\widetilde{\bm{\alpha}}\in\calA_{\textup{cand}}
\,\big|\,&
D(X;\vE_{1:B})\leq D_{\textup{thr}}
\textnormal{ and }\\
&\sup_{\vE_b\in\mathfrak{E}}
\mathrm{CVaR}_{\varepsilon}^{\vE_b(X)}
(\widehat{\phi})
\leq\widebar{\phi}\bigr\}.
\end{aligned}
\label{eq:verified-set}
\end{equation}
The epistemic gate and the CVaR condition provide calibrated ID recall and a bound on predicted risk, respectively, as established in \autoref{lem:epistemic} and \autoref{lem:aleatoric}.

\subsection{Performance-Guided CBF Parameter Selection \label{subsec:select}}

\begin{algorithm}[t]\footnotesize
\caption{Online Adaptive CBF (OA-CBF)}
\label{alg:online_adaptation}
\textbf{Given:} initial parameter $\widetilde{\bm{\alpha}}_0^{\star}$; CBF-based controller $\pi(\cdot;\widetilde{\bm{\alpha}})$; trained PENN $\vF(\cdot;\bm{\theta}_{1:B})$; calibrated epistemic threshold $D_{\textup{thr}}$. \\
\For{$t=1,\ldots,T_{\textup{max}}$}
{
    $\vx_t,\calE_t \leftarrow \texttt{Sensor()}$; \algovspace\\
    $\calA_{\textup{sample}} \leftarrow \texttt{SampleCandidates}(\widetilde{\bm{\alpha}}_{t-1}^{\star})$; \algovspace\\

    \tcp*[h]{--- Parallelizable block ---} \algovspace \algovspace \\

    $\calA_{\textup{cand}} \leftarrow \{ \widetilde{\bm{\alpha}} \in \calA_{\textup{sample}} \mid \vx_t \in \calC^{\star}(\widetilde{\bm{\alpha}}) \}$; \algovspace\\
    $\vX =\{X^{(k)} \}^{K}_{k=1} \leftarrow \texttt{CreateInputs}(\vx_t,\calE_t,\calA_{\textup{cand}})$; \algovspace\\
    $\mathfrak{E}, \widehat{\vY} =\{\widehat{Y}^{(k)} \}^{K}_{k=1} \leftarrow \vF(\vX;\bm{\theta}_{1:B})$; \algovspace\\
    $\widehat{\calA} \leftarrow \{\widetilde{\bm{\alpha}} \in \calA_{\textup{cand}} \mid D(X^{(k)};\vE_{1:B}) \le D_{\textup{thr}}$ \algovspace\\
    $\qquad \qquad \qquad \qquad \text{and } \sup_{\vE_b \in \mathfrak{E}} \textnormal{CVaR}_\varepsilon^{\vE_b(X^{(k)})}(\widehat{\phi}) \leq \widebar{\phi} \}$; \algovspace\\

    \tcp*[h]{--- End parallelizable block ---} \algovspace \algovspace \\

    \If{$\widehat{\calA}=\varnothing$}{
        $\widetilde{\bm{\alpha}}_t^\star\leftarrow\texttt{ApplyContingency}(\vx_t,\widetilde{\bm{\alpha}}_{t-1}^\star)$;\\
    }
    $\widetilde{\bm{\alpha}}_t^{\star} \leftarrow \argmin_{\widetilde{\bm{\alpha}} \in \widehat{\calA}} \hat{\delta}$; \algovspace\\
    $\vu_t \leftarrow \pi(\vx_t;\widetilde{\bm{\alpha}}_t^{\star})$; \algovspace\\
    $\texttt{ApplyControlInput}(\vu_t)$;
}
\end{algorithm}

After the two screening stages, $\widehat{\calA}$~\eqref{eq:verified-set} may contain multiple candidate parameters. When this set is nonempty, we select among them using the second output of the prediction model, the predicted performance metric $\hat{\delta}$, where smaller values indicate more efficient progress toward the goal:
\begin{equation}
\widetilde{\bm{\alpha}}_t^{\star}
=\argmin_{\widetilde{\bm{\alpha}}\in\widehat{\calA}}\hat{\delta}.
\label{eq:performance_selection}
\end{equation}
Because the optimization is restricted to $\widehat{\calA}$, performance-based selection does not weaken either screening condition. In practice, when multiple candidates attain nearly identical values of $\hat{\delta}$, one may additionally prefer the candidate with smaller deviation from the previously applied parameter to reduce chattering.

\subsection{Online CBF Parameter Adaptation \label{subsec:online-cbf}}

Integrating the verification and selection steps above, \autoref{alg:online_adaptation} implements the proposed OA-CBF framework. At each decision instant, a batch of candidate parameters is sampled around the previously applied parameter. Candidates satisfying~\eqref{eq:local-state-membership} are evaluated by the PENN, passed through the epistemic and aleatoric verification stages, and ranked using the predicted performance metric.

The following theorem combines the two verification results and states the guarantees of the parameter-selection procedure.

\begin{theorem}[Probabilistic Verification of OA-CBF Parameters]
\label{thm:main-result}
Consider \autoref{alg:online_adaptation} with the model~\eqref{eq:penn}. Under the exchangeability assumptions of \autoref{lem:epistemic}, the epistemic gate~\eqref{eq:epistemic_filter} satisfies the ID recall requirement~\eqref{eq:problem-epistemic} of \autoref{prob:epistemic}.

Whenever $\widehat{\calA}$~\eqref{eq:verified-set} is nonempty, the parameter selected by~\eqref{eq:performance_selection} satisfies~\eqref{eq:local-state-membership} and the predictive risk requirement~\eqref{eq:problem-aleatoric} of \autoref{prob:aleatoric}. It minimizes the predicted performance cost over $\widehat{\calA}$. Thus, the two verification steps satisfy \autoref{prob:epistemic} and \autoref{prob:aleatoric}.
\end{theorem}

\begin{proof}
The ID recall bound follows from \autoref{lem:epistemic}. Every parameter in $\widehat{\calA}$ satisfies~\eqref{eq:local-state-membership} and the CVaR condition~\eqref{eq:distributionally-robust-cvar}; \autoref{lem:aleatoric} therefore gives~\eqref{eq:problem-aleatoric}. Selection by~\eqref{eq:performance_selection} preserves these conditions and minimizes the predicted performance cost within the retained set.
\end{proof}

Since the probabilistic guarantees in \autoref{thm:main-result} are based on the predictions of the trained model rather than exact local validation, \autoref{thm:main-result} does not guarantee recursive feasibility of OA-CBF.

\begin{remark}[Contingency When the Parameter Set Is Empty]
\label{rem:empty-screening-set}
The set $\widehat{\calA}(t)$ may be empty if the training data do not adequately represent the current operating conditions, the sampled candidates are unsuitable, or the user-specified tolerances for epistemic uncertainty and aleatoric risk are overly restrictive. The selection conclusion of \autoref{thm:main-result} then does not apply. A practical contingency in \autoref{alg:online_adaptation} is to gradually reduce the previously applied CBF-parameter coefficients to encourage more conservative behavior. Parameter decay alone does not ensure local validity or feasibility under input constraints and carries no safety guarantee.

A certified backup controller can be used when the current state lies within its certified domain. For example, a backup-CBF construction requires an admissible backup policy, a terminal set invariant under that policy, and safe reachability of that set~\cite{chen_backup_2021}. Establishing these properties can be as challenging as constructing a valid CBF for a nonlinear system with input constraints. Backup design is therefore a related research problem beyond the scope of this paper. Such a backup is not assumed to be available, and the terminal neighborhoods used for offline data generation do not establish the required backup certificate.
\end{remark}

For the discrete-time realization in \discreteappendix, the corresponding interval-by-interval argument applies at the modeled sampling instants. Continuous-time safety between samples additionally requires the corresponding sampled-data conditions.

From a computational standpoint, the proposed CBF-parameter verification pipeline is fully batch-parallelizable. At each update step, we perform batch-wise PENN inference over the candidate set $\calA_{\textup{cand}}$ and across all ensemble members, which can be executed efficiently on GPU. Moreover, the JRD and the Gaussian CVaR both admit closed-form solutions, so the epistemic gate and aleatoric verification can be computed directly from the predicted moments in parallel, without Monte Carlo approximation. We emphasize that there is \textbf{\emph{no single iteration loop}} in the entire CBF parameter verification and adaptation procedure, ensuring efficient real-time implementation.

\section{RESULTS \label{sec:simulation}}

We evaluate the proposed OA-CBF method through simulation experiments on multiple nonlinear, input-constrained dynamical systems. The experiments are designed to answer four key questions. \textbf{(Q1)} Can OA-CBF serve as an effective solution to \autoref{prob:identifying}, i.e., can it identify locally validated CBF parameters online for nonlinear systems with input constraints? \textbf{(Q2)} Does the proposed online adaptation reduce the conservatism of fixed-parameter CBF controllers while preserving safety? \textbf{(Q3)} Does the proposed graph-attention-based environment encoding allow OA-CBF to generalize to obstacle configurations and obstacle counts that differ from those used during training? \textbf{(Q4)} Is OA-CBF applicable beyond classical signed-distance-function~(SDF)-type CBF candidates?

We first summarize the experimental setup and CBF candidates, then compare OA-CBF against fixed-parameter controllers, online-decay CBF baselines, and learning-based baselines, including BarrierNet~\cite{xiao_barriernet_2023} and an OA-CBF variant with fully connected environment encoding~\cite{kim_learning_2025,kim_how_2025}. The system equations and implementation settings are provided in Appendix~\appendixref{app:implementation}{A}.

\subsection{Experimental Setup}
\label{subsec:experimental_setup}

The safety-critical controller~$\pi$ is instantiated differently across baselines. The fixed-parameter baselines use the sampled-data MPC-CBF realization~\cite{zeng_safetycritical_2021} summarized in \discreteappendix, unless stated otherwise. OA-CBF w/ GAT and OA-CBF w/ FC use CBF-QP safety filters on all four randomized systems, while the VTOL case study uses MPC-CBF. The OD-CBF with QP baseline \cite{zeng_safetycritical_2021_decay} follows the standard CBF-QP formulation, while the BarrierNet baseline uses a CBF-QP as the last layer, as in the original algorithm~\cite{xiao_barriernet_2023}. In contrast, OD-CBF with MPC~\cite{zeng_enhancing_2021} uses the same MPC-CBF structure as the MPC-based baselines, but optimizes the decay parameters online inside the MPC problem. Thus, all methods are evaluated on the same dynamics, obstacle configurations, and input constraints, but they may differ in whether the underlying safety filter is realized as a CBF-QP or as a finite-horizon MPC-CBF. Details of the compared baselines are provided in \autoref{subsec:baselines}. %%

The controller update period is $\Delta t=0.05$~s, which is also the MPC discretization step. The admissible input set is enforced in each optimization problem, with system-specific limits listed in \autoref{tab:input_limits}. The training and testing environments are generated independently, and test configurations differ from those used for training. Each evaluated combination of method and system uses 1,000 cases. All methods on a case share its initial conditions, obstacle motion, route, and sensor-noise draws. The test layouts contain 24, 36, 48, or 60 obstacles.

OA-CBF uses a PENN with $B=4$ ensemble members, a CVaR upper-tail probability of $\varepsilon=0.01$, and a conformal coverage level of $1-\delta_{\mathrm{cal}}=0.95$. The epistemic threshold is calibrated from $200{,}000$ trajectories for each system. Network settings, data generation, and controller-cost choices are given in Appendix~\appendixref{app:experimental_setup}{\mbox  {A-E}}.

The reported metrics are goal-reaching rate (Success), infeasibility rate (Inf.), collision rate (Col.), and average reach time (Time [s]). Inf. counts trials terminated due to solver infeasibility, while Col. includes both collisions and workspace constraint violations. Success is the percentage of trials that successfully reach the goal within the maximum simulation time. The remaining unsuccessful trials are attributed to timeouts. Average reach time is computed only over successful goal-reaching trials; failed, infeasible, and collision trials are excluded from this average.

\begin{table*}[t]
\centering
\caption{Input constraints used in the simulation studies.}
\label{tab:input_limits}
{\renewcommand{\arraystretch}{1.15}
\setlength{\tabcolsep}{5pt}
\begin{tabular}{@{} p{0.2\textwidth} p{0.74\textwidth} @{}}
\toprule
\textbf{System} & \textbf{Input Constraints} \\
\midrule
Dynamic unicycle
& $a \in [-0.5,0.5]~\mathrm{m/s^2}$, $\omega \in [-0.5,0.5]~\mathrm{rad/s}$. \\

Quad2D
& $\vu = [u_{\textup{r}},u_{\textup{l}}]^\top$, with $u_{\textup{r}},u_{\textup{l}} \in [2.5,5.5]$. \\

Quad3D 
& $\vu=[u_1,u_2,u_3,u_4]^\top$, with $u_i\in[-2.0,5.0]$. \\

Kinematic bicycle
& $a \in [-5.0,5.0]~\mathrm{m/s^2}$ and $\beta \in [-\beta_{\max},\beta_{\max}]$, where $\beta_{\max}=\tan^{-1}\!\left(\frac{\ell_r}{\ell_f+\ell_r}\tan(32^\circ)\right) \approx 0.303~\mathrm{rad}$ ($17.4^\circ$).\\

VTOL quadplane
& $\delta_f,\delta_r,\delta_p \in [0,1]$, $\delta_e \in [-0.5,0.5]$. \\
\bottomrule
\end{tabular}}
\end{table*}

\begin{table*}[!t]
\centering
\caption{CBF candidates used in simulation. The entry $r$ denotes the relative degree of the continuous-time CBF candidate.}
\label{tab:cbf_relative_degree}
\resizebox{\textwidth}{!}{
\begin{tabular}{lccccc}
\toprule
CBF candidate 
& Dynamic Unicycle 
& Quad2D 
& Quad3D 
& Kinematic Bicycle 
& VTOL Quadplane \\
\midrule
Distance / planar obstacle CBF \appendixeqref{eq:distance_cbf_candidate}{56}
& $r=2$, ICCBF 
& $r=2$, ICCBF 
& $r=4$ in continuous time; RK4 discrete-time CBF used
& N/A
& $r=2$, ICCBF \\
DPCBF \appendixeqref{eq:dpcbf_candidate}{57}
& N/A
& N/A
& N/A
& $r=1$, CBF
& N/A \\
\bottomrule
\end{tabular}
}
\end{table*}

\subsection{Applied Dynamical Systems}
\label{subsec:applied_dynamical_systems}

We evaluate OA-CBF on five dynamical systems: a dynamic unicycle, a planar quadrotor (Quad2D), a three-dimensional quadrotor (Quad3D), a kinematic bicycle, and a VTOL quadplane (Fig.~\ref{fig:quadplane}). The first four systems are used for randomized benchmark comparisons, while the VTOL quadplane is used for a cruise-to-hover transition and landing case study. The Quad3D model follows~\cite{annaswamy_integration_2023}, the bicycle benchmark uses the Dynamic Parabolic CBF (DPCBF) formulation of~\cite{park_collision_2026}, and the VTOL aerodynamic model follows~\cite{willis_pitch_2022,beard_small_2012}. The state and input definitions, dynamics, and model assumptions are given in Appendix~\appendixref{app:system_models}{\mbox  {A-F}}.

\begin{figure}[!t]
\centering
\includegraphics[width=0.99\linewidth]{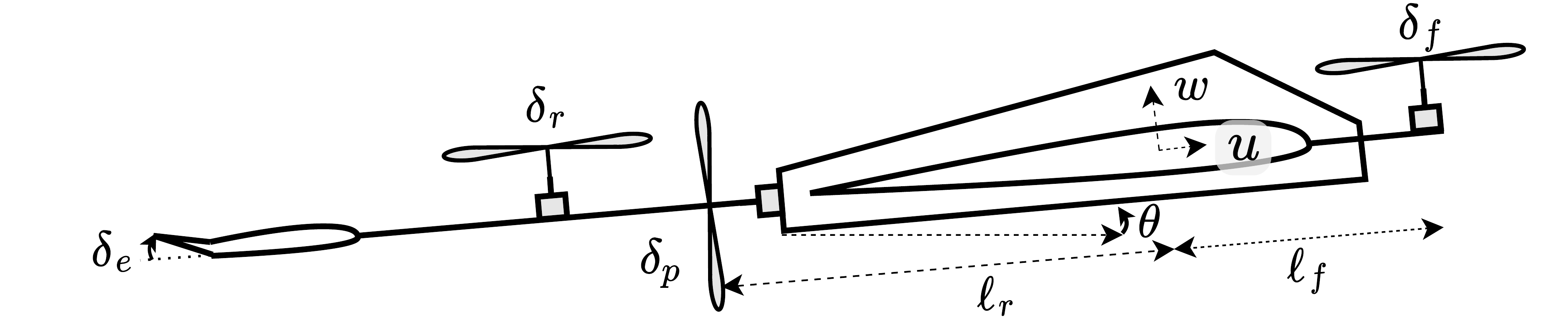}
\caption{
Illustration of a VTOL quadplane aircraft.}
\label{fig:quadplane}
\vspace{-5pt}
\end{figure}

\subsection{CBF Candidates and Relative-Degree Structure}
\label{subsec:cbf_candidates}

The dynamic unicycle, Quad2D, Quad3D, and VTOL quadplane use the distance-based obstacle-avoidance candidate in~\appendixeqref{eq:distance_cbf_candidate}{56}. The dynamic unicycle, Quad2D, and VTOL candidate constraints have relative degree two; their ICCBFs use two linear class-$\calK_\infty$ functions, $\alpha_i(s)=\widetilde{\alpha}_i s$, $i\in\{1,2\}$. For Quad3D, the planar avoidance constraint has relative degree four, and the MPC-based methods use the sampled-data RK4 CBF formulation of~\cite{taylor_safety_2022}.

The kinematic bicycle benchmark uses the relative-degree-one DPCBF~\cite{park_collision_2026} in~\appendixeqref{eq:dpcbf_candidate}{57}, which evaluates safety in relative-velocity space. OA-CBF is coupled with a CBF-QP for this benchmark. The candidate constructions are detailed in Appendix~\appendixref{app:cbf_candidates}{\mbox  {A-G}}, and their relative-degree structures are summarized in \autoref{tab:cbf_relative_degree}.

\begin{table*}[t]
\centering
\caption{Performance comparison across three system dynamics using 1,000 trials for each evaluated combination of method and system. Success, Inf., and Col. are percentages; Time is the mean simulated reach time over successful trials. The metrics are defined in \autoref{subsec:experimental_setup}. A dash indicates no successful trials. $^\dagger$ For Quad3D, the planar obstacle-avoidance constraint has relative degree four in continuous time; for the signed-clearance CBF-QP realization, the obstacle function is not four-times differentiable, so the corresponding continuous-time fourth-order CBF-QP constraint is not constructed.}
\label{tab:sdf_results}
{\small\setlength{\tabcolsep}{3pt}
\begin{tabular}{l|rrrr|rrrr|rrrr}
\toprule
 & \multicolumn{4}{c|}{\textbf{Dynamic Unicycle}} & \multicolumn{4}{c|}{\textbf{Quad2D}} & \multicolumn{4}{c}{\textbf{Quad3D}} \\
\cmidrule(r){2-5} \cmidrule(lr){6-9} \cmidrule(l){10-13}
\textbf{Method} & \textbf{Success} & \textbf{Inf.} & \textbf{Col.} & \textbf{Time [s]} & \textbf{Success} & \textbf{Inf.} & \textbf{Col.} & \textbf{Time [s]} & \textbf{Success} & \textbf{Inf.} & \textbf{Col.} & \textbf{Time [s]} \\
\midrule
Fixed Low & 22.7 & 20.1 & 42.9 & 23.90 & 0.0 & 96.5 & 2.6 & \textemdash & 14.7 & 81.2 & 0.0 & 68.08 \\
Fixed High & 16.7 & 82.5 & 0.3 & 10.70 & 1.6 & 37.5 & 1.3 & 66.21 & 46.4 & 46.7 & 0.0 & 46.69 \\
OD-CBF w/ QP~\cite{zeng_safetycritical_2021_decay} & 0.4 & 98.5 & 0.6 & 12.89 & 0.0 & 27.7 & 72.3 & \textemdash & \multicolumn{4}{c}{N/A (not $C^4$)$^\dagger$} \\
OD-CBF w/ MPC~\cite{zeng_enhancing_2021} & 26.1 & 3.0 & 59.4 & 26.78 & 0.0 & 5.2 & 45.1 & \textemdash & 93.1 & 0.1 & 0.0 & 57.88 \\
BarrierNet~\cite{xiao_barriernet_2023} & 15.1 & 6.7 & 65.9 & 19.50 & 0.0 & 95.6 & 4.4 & \textemdash & 0.0 & 0.0 & 100.0 & \textemdash \\
OA-CBF w/ FC \cite{kim_learning_2025, kim_how_2025} & 93.5 & 0.0 & 0.0 & 23.41 & 96.0 & 0.0 & 0.0 & 23.68 & 94.6 & 5.4 & 0.0 & 27.21 \\
\midrule
\rowcolor{gray!20} OA-CBF w/ GAT & \textbf{93.6} & 0.0 & 0.0 & 23.33 & \textbf{96.7} & 0.0 & 0.0 & 24.84 & \textbf{99.7} & 0.3 & 0.0 & 41.73 \\
\bottomrule
\end{tabular}}
\end{table*}

\subsection{Baselines}
\label{subsec:baselines}

We compare OA-CBF against fixed-parameter, optimal-decay, and learning-based baselines. \textbf{Fixed Low} uses conservative fixed CBF parameters, while \textbf{Fixed High} uses aggressive fixed parameters. \textbf{OD-CBF with QP} implements the optimal-decay CBF-QP method~\cite{zeng_safetycritical_2021_decay,ong_properties_2025}, and \textbf{OD-CBF with MPC} implements the optimal-decay MPC-CBF method of~\cite{zeng_enhancing_2021}. For relative-degree-two CBFs, both CBF parameters are optimized in optimal-decay baselines. \textbf{BarrierNet}~\cite{xiao_barriernet_2023} is used as a differentiable CBF-QP-layer baseline when the corresponding continuous-time CBF-QP constraint is well defined. We also compare against OA-CBF with fully connected closest-obstacle encoding, denoted \textbf{OA-CBF w/ FC}, to isolate the effect of replacing heuristic obstacle selection with the proposed graph-attention encoder. This baseline follows the implementation provided by \cite{kim_learning_2025,kim_how_2025}. Our proposed method is denoted \textbf{OA-CBF w/ GAT}.  Detailed baseline settings and implementation details are provided in Appendix~\appendixref{app:baseline_details}{\mbox  {A-I}}.

For OA-CBF, candidate parameters are sampled from a system-specific adaptation interval $\widetilde{\bm{\alpha}}\in[\widetilde{\alpha}_{\min},\widetilde{\alpha}_{\max}]$, and the initially applied parameter is set to the lower endpoint unless stated otherwise. The candidate is accepted only if it passes both the epistemic uncertainty gate and the distributionally robust CVaR constraint described in \autoref{sec:online}. Among the accepted candidates, the parameter with the lowest predicted performance cost is applied.

\begin{figure*}[t]
\centering
\includegraphics[width=0.99\textwidth]{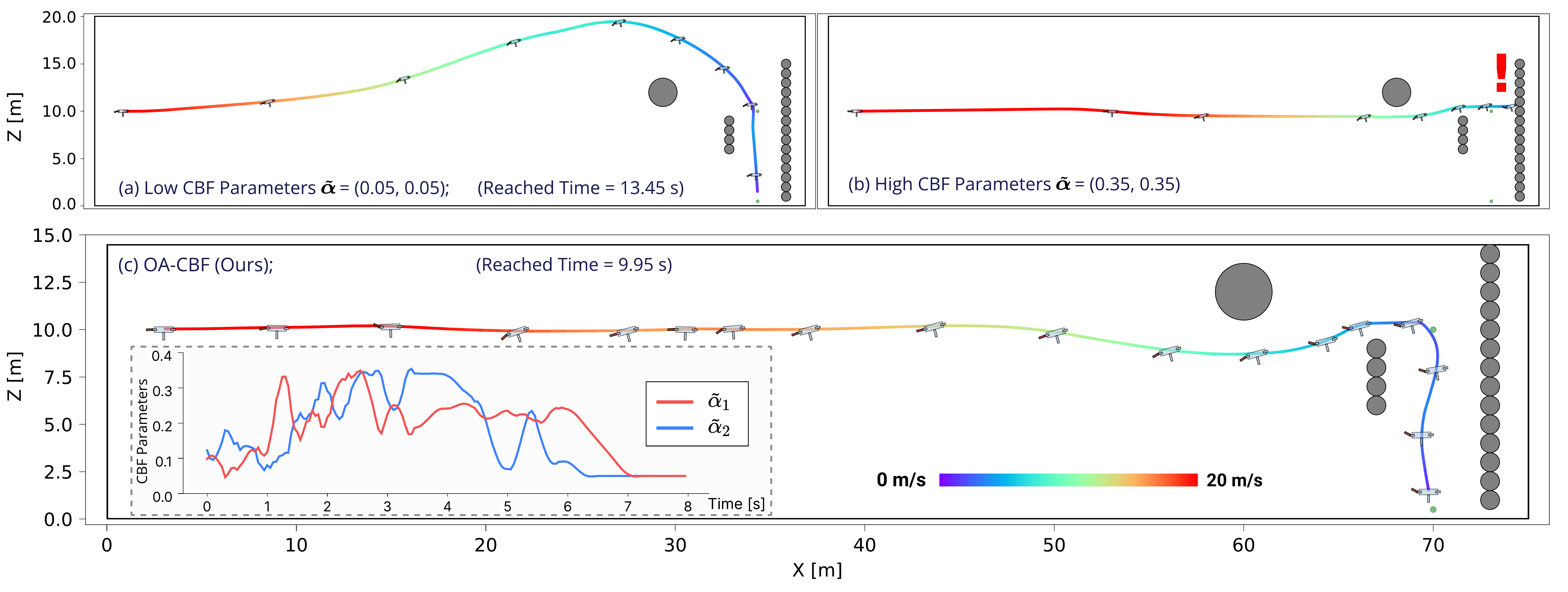}
\caption{
The figure shows the aircraft trajectory along with its speed profile, rigid-body pose, and elevator angle $\delta_e$, with obstacles depicted in gray. (a) With fixed low parameters, the trajectory shows a significant altitude detour as the ICCBF constraints force the aircraft to pitch up to decelerate in response to the obstacle. (b) With fixed high parameters, the controller becomes infeasible, ultimately resulting in a collision. (c) Our adaptive approach dynamically adjusts the CBF parameters based on the aircraft's speed and position. Initially, due to the high speed, it maintains low CBF parameters, prompting the elevator to pitch up and generate additional drag. As the aircraft slows down, the parameters increase to enhance performance, enabling the aircraft to fly safely beneath the obstacle and executing a smooth transition afterward. \textbf{Inset}: Evolution of the adapted CBF parameters over time.}
\label{fig:vtol_results}
\vspace{-5pt}
\end{figure*}

\subsection{Results and Discussion}
\label{subsec:results_discussion}

\subsubsection{\textbf{(Q1)-(Q2)}: Safety and Conservatism Across Nonlinear Systems}

\autoref{tab:sdf_results} compares the methods on the dynamic unicycle, Quad2D, and Quad3D benchmarks using distance-based obstacle-avoidance CBFs. Representative trajectories are shown in Fig.~\appendixref{fig:traj}{9} in Appendix~\appendixref{app:additional_results}{\mbox  {A-H}}, illustrating how each method behaves in the same obstacle-rich environments.

The fixed-low controller achieves goal-reaching rates of $22.7\%$, $0\%$, and $14.7\%$ on the dynamic unicycle, Quad2D, and Quad3D benchmarks, respectively. It has high infeasibility rates on Quad2D and Quad3D ($96.5\%$ and $81.2\%$), while $42.9\%$ of the unicycle trials end in collision. Thus, small fixed CBF parameters alone do not ensure reliable task completion in these benchmarks. %%

The fixed-high controller also has limited goal-reaching performance. Its goal-reaching rates are $16.7\%$, $1.6\%$, and $46.4\%$ on the dynamic unicycle, Quad2D, and Quad3D benchmarks, respectively, with high infeasibility rates of $82.5\%$, $37.5\%$, and $46.7\%$. These results show that a single aggressive CBF parameter cannot reliably handle all operating conditions under input constraints. %%

The optimal-decay baselines partially address this issue, but their performance is not uniform across systems. OD-CBF with QP has a $98.5\%$ infeasibility rate on the dynamic unicycle and fails to reach the goal in the Quad2D benchmark due to input constraints. OD-CBF with MPC achieves a $93.1\%$ goal-reaching rate on Quad3D, but reaches the goal in only $26.1\%$ of unicycle trials and has no successful Quad2D trials. %%

BarrierNet exhibits a different failure mode. Its average reach time is $19.50$~s on the dynamic unicycle, but this value is conditional on the $15.1\%$ of trials that successfully reach the goal. It has no successful trials on Quad2D or Quad3D: $95.6\%$ of Quad2D trials end in solver infeasibility, while all Quad3D trials end in collision. Thus, a short conditional reach time does not imply reliable task completion under the input constraints. %%

In contrast, OA-CBF with graph-attention encoding has no observed collisions or state-bound violations and achieves goal-reaching rates of $93.6\%$, $96.7\%$, and $99.7\%$ on the dynamic unicycle, Quad2D, and Quad3D benchmarks, respectively. Its average reach times are $23.33$~s, $24.84$~s, and $41.73$~s. OA-CBF w/ GAT gives the highest observed goal-reaching rate on all three benchmarks. %

These results support \textbf{(Q1)} and \textbf{(Q2)}: in these trials, OA-CBF improves goal-reaching performance over the fixed-parameter controllers with no observed collisions or workspace constraint violations.

\begin{figure*}[!t]
\centering
\includegraphics[width=0.99\textwidth]{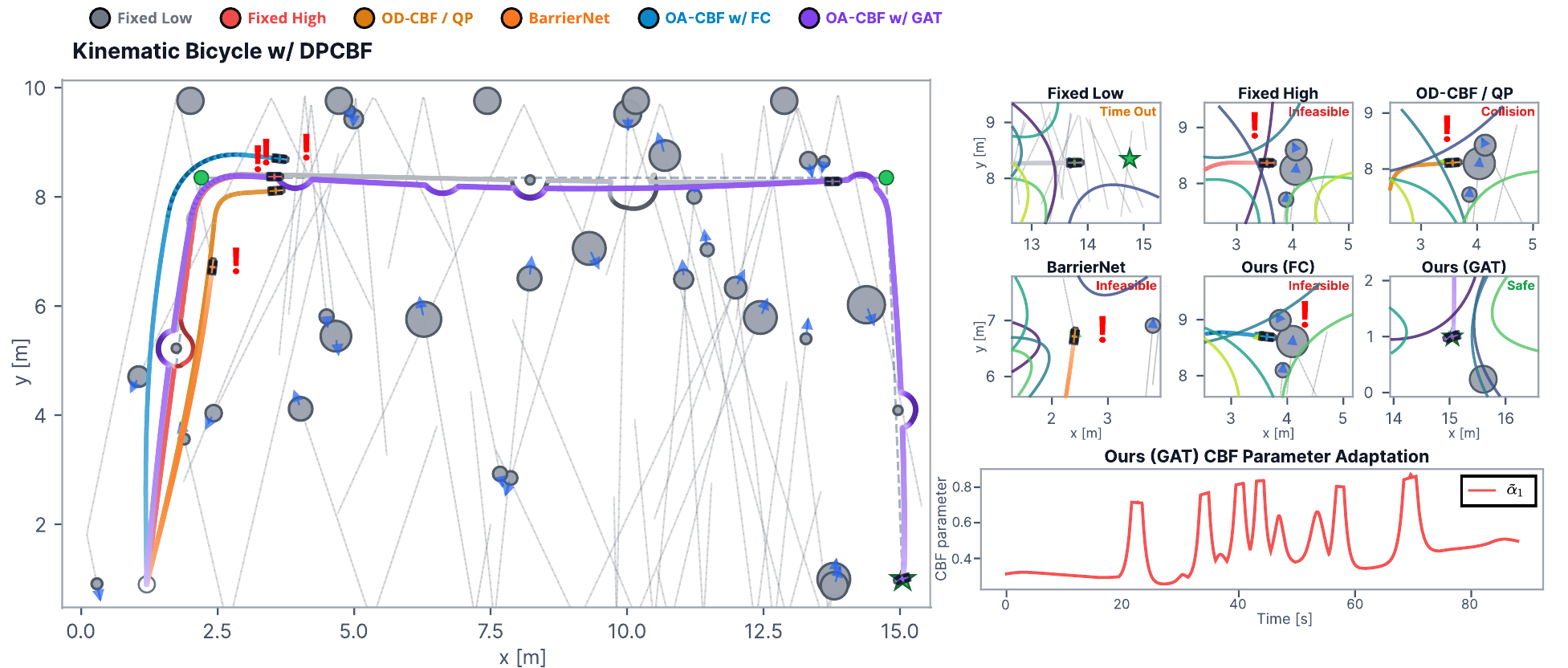}
\caption{
Representative trajectories for the kinematic bicycle benchmark with DPCBF in Table~\ref{tab:dpcbf_results}. The DPCBF safety boundary associated with each obstacle is also visualized on the right panels.}
\label{fig:dpcbf_traj}
\vspace{-5pt}
\end{figure*}

\subsubsection{VTOL Quadplane Case Study}

Fig.~\ref{fig:vtol_results} shows the VTOL quadplane case study. The quadplane starts in cruise mode with a high initial horizontal velocity, transitions toward hover mode near the intermediate waypoint, and then lands while avoiding obstacles. The initial position and horizontal velocity are $(p_{x,0},p_{z,0})=(2,10)~\text{m}, v_{x,0}=20~\text{m/s}$. The intermediate waypoint is $(70,10)$~m and the landing goal is $(70,0)$~m.

With fixed low CBF parameters, the controller remains safe but reacts conservatively to the obstacle. The aircraft pitches upward and takes a large altitude detour to reduce speed before passing the obstacle. With fixed high CBF parameters, the controller is less conservative initially but becomes infeasible near the obstacle, resulting in collision. OA-CBF adapts the CBF parameters based on the current speed and position of the aircraft. At high speed, it keeps the parameters small enough to preserve feasibility and induce early deceleration. Once the aircraft slows down, the parameters increase, reducing conservatism and allowing the aircraft to pass below the obstacle before transitioning smoothly to hover and landing.

In the simulation shown in Fig.~\ref{fig:vtol_results}, OA-CBF reaches the goal in approximately $9.95$~s, compared with approximately $13.45$~s for the fixed-low baseline, while the fixed-high baseline collides. %

This case study further supports \textbf{(Q1)} and \textbf{(Q2)} for a nonlinear system with aerodynamic effects and strict input constraints: the proposed parameter selection reduces the detour and transition time in the simulated VTOL maneuver.

\begin{table}[t]
\centering
\caption{Performance comparison on the kinematic bicycle model with DPCBF using 1,000 trials per method. Success, Inf., and Col. are percentages; Time is the mean simulated reach time over successful trials. A dash indicates no successful trials.}
\label{tab:dpcbf_results}
{\small\setlength{\tabcolsep}{3pt}
\begin{tabular}{l|rrrr}
\toprule
 & \multicolumn{4}{c}{\textbf{Kinematic Bicycle w/ DPCBF}} \\
\cmidrule(r){2-5}
\textbf{Method} & \textbf{Success} & \textbf{Inf.} & \textbf{Col.} & \textbf{Time [s]} \\
\midrule
Fixed Low & 3.6 & 3.3 & 93.0 & 45.49 \\
Fixed High & 0.0 & 100.0 & 0.0 & \textemdash \\
OD-CBF w/ QP~\cite{zeng_safetycritical_2021_decay} & 0.0 & 0.2 & 99.8 & \textemdash \\
BarrierNet~\cite{xiao_barriernet_2023} & 1.0 & 0.5 & 98.5 & 19.03 \\
OA-CBF w/ FC \cite{kim_learning_2025, kim_how_2025} & 93.2 & 0.4 & 0.0 & 23.25 \\
\midrule
\rowcolor{gray!20} OA-CBF w/ GAT & \textbf{93.4} & 0.6 & 0.0 & 27.28 \\
\bottomrule
\end{tabular}}
\end{table}

\subsubsection{\textbf{(Q3)}: Effect of Graph-Attention Encoding}

The comparison between OA-CBF w/ FC and OA-CBF w/ GAT serves as the encoder ablation in this section: both variants use the same offline labels, PENN ensemble size, uncertainty filtering, and online candidate-selection rule, but differ in whether the environment is represented by the closest-obstacle feature vector or by graph-attention aggregation.

In \autoref{tab:sdf_results}, OA-CBF w/ FC achieves goal-reaching rates of $93.5\%$, $96.0\%$, and $94.6\%$ on the dynamic unicycle, Quad2D, and Quad3D benchmarks, respectively. The corresponding rates for OA-CBF w/ GAT are $93.6\%$, $96.7\%$, and $99.7\%$. Both variants have no observed collisions. The largest difference occurs on Quad3D, where solver infeasibility decreases from $5.4\%$ to $0.3\%$.
These results support \textbf{(Q3)}: the graph-attention encoder handles varying obstacle counts without a fixed-size obstacle vector or heuristic obstacle selection, while its observed performance advantage over FC depends on the dynamics. %

\subsubsection{\textbf{(Q4)}: Adaptation with a Non-Distance-Based CBF}

\autoref{tab:dpcbf_results} evaluates OA-CBF on the kinematic bicycle model with the DPCBF candidate in \appendixeqref{eq:dpcbf_candidate}{57}. This experiment tests whether OA-CBF is tied to classical distance-based CBFs or can also be applied to CBF formulations whose safety semantics are defined in the relative-velocity space. Representative trajectories for the DPCBF benchmark are shown in Fig.~\ref{fig:dpcbf_traj}, illustrating the behavior of each method in the dynamic-obstacle setting.

The fixed-low controller reaches the goal in only $3.6\%$ of trials, with an average reach time of $45.49$~s, while $93.0\%$ of trials end in collision. The fixed-high controller has no successful trials: all terminate with solver infeasibility. These outcomes show that neither fixed-parameter choice reliably completes the DPCBF benchmark.

The QP-based adaptive baselines also have limited success in this benchmark. OD-CBF with QP has no successful trials, with $99.8\%$ ending in collision due to input constraints. BarrierNet succeeds in only $1.0\%$ of trials, with a conditional mean reach time of $19.03$~s and a collision rate of $98.5\%$. The short BarrierNet reach time therefore represents only its small successful subset.

OA-CBF improves goal-reaching performance without observed collisions. OA-CBF w/ FC and OA-CBF w/ GAT achieve goal-reaching rates of $93.2\%$ and $93.4\%$, with average reach times of $23.25$~s and $27.28$~s, respectively. For GAT, only $0.6\%$ of the trials were unsuccessful because no locally validated CBF parameter was identified. These results support \textbf{(Q4)}, demonstrating that the proposed uncertainty-aware parameter validation mechanism is not restricted to signed-distance-function-type CBF candidates.

Taken together, \autoref{tab:sdf_results}, \autoref{tab:dpcbf_results}, and Fig.~\ref{fig:vtol_results} support \textbf{(Q1)}-\textbf{(Q4)}. In these benchmarks, OA-CBF improves goal-reaching performance relative to fixed-parameter controllers, achieves the highest observed goal-reaching rate with graph-attention encoding on all four randomized systems, and applies beyond distance-based CBF candidates.

\section{CONCLUSION \label{sec:conclusion}}

In this paper, we studied online CBF-parameter adaptation for nonlinear robotic systems subject to input constraints. CBF parameters affect safety-filter feasibility, the shape of the inner safe set, and the conservatism of the resulting closed-loop behavior; fixed conservative parameters can slow task execution, whereas aggressive parameters can lead to collision or controller infeasibility. To address this issue, we proposed OA-CBF, which predicts a finite-horizon safety loss and a performance metric for queried parameters. A graph-attention-based probabilistic ensemble model represents the system state and constraint information. Conformal calibration bounds rejection of exchangeable in-distribution trajectories, while a distributionally robust CVaR test bounds predictive tail risk. OA-CBF selects the retained candidate with the lowest predicted performance cost. The framework can be coupled with both MPC-CBF and CBF-QP safety filters and is shown in simulation to reduce conservatism while maintaining low safety-failure rates across distance-based ICCBF and DPCBF benchmarks.

\FloatBarrier
\bibliographystyle{IEEEtran}
\typeout{}
\bibliography{references.bib}

% Appendices are included only in the arXiv version.
\ifarxivversion
\appendices
\section{Implementation and Additional Experimental Material}
\label{app:implementation}

\subsection{Graph-Attention Encoder for the Navigation Case Studies}
\label{app:graph}
The main formulation uses the constraint parameters $\calE$ in~\eqref{eq:constraint-family}. For the navigation experiments, $\mathbf o^j$ contains the position, velocity, and size of the $j$-th observed obstacle within sensing range $l_{\mathrm{range}}$. The number $M$ of observed obstacles can vary, so we encode their features using a graph-attention network~\cite{li_graph_2019,zhang_gcbf_2025}. This avoids the fixed-neighbor padding or truncation used in~\cite{kim_learning_2025,kim_how_2025}.

\subsubsection{Graph Construction}
Let $\calG=(\calV,\calE_{\mathrm{edge}})$ have $M+2$ nodes representing the robot state, the goal, and the obstacle parameters. For each node $i$, let $\vp_i$ and $\vv_i$ denote its position and translational velocity in the task space, and let $R_i$ denote its geometric radius. The robot features are obtained from $\vx_t$; obstacle features are obtained from $\mathbf o^j$; the goal is specified by its position $\vp_{\mathrm{goal}}$, with zero velocity and radius. A separate one-hot vector $\mathbf t_i$ encodes node type: $[1,0,0]^\top$ for the robot, $[0,1,0]^\top$ for an obstacle, and $[0,0,1]^\top$ for the goal.

An edge $e_{j\to i}\in\calE_{\mathrm{edge}}$ transmits features from a neighboring node $j$ to receiver $i$. Its feature vector is
\begin{equation}
\ve_{ji}\coloneqq[(\vp_j-\vp_i)^\top,(\vv_j-\vv_i)^\top,
\operatorname{dist}(\vp_j,\vp_i)]^\top,
\label{eq:graph-edge}
\end{equation}
where the surface-to-surface distance is
\begin{equation}
\operatorname{dist}(\vp_j,\vp_i)\coloneqq
\max(0,\|\vp_j-\vp_i\|-R_j-R_i).
\label{eq:graph-distance}
\end{equation}

\subsubsection{Aggregation via Graph Attention}
For each edge, a multilayer perceptron (MLP) encoder maps $\vz_{ji}\coloneqq[\mathbf t_i^\top,\mathbf t_j^\top,\ve_{ji}^\top]^\top$ to a latent feature $\vq_{ji}=\psi_{\mathrm{enc}}(\vz_{ji})$. A gate network computes normalized attention weights
\begin{equation}
w_{ji}=\frac{\exp(\psi_{\mathrm{gate}}(\vq_{ji}))}
{\sum_{k\in\calN_i}\exp(\psi_{\mathrm{gate}}(\vq_{ki}))},
\label{eq:graph-attention}
\end{equation}
where $\calN_i$ is the neighbor set of node $i$. Thus $w_{ji}\ge0$ and $\sum_{j\in\calN_i}w_{ji}=1$. The context vector is
\begin{equation}
\vc_i=\sum_{j\in\calN_i}w_{ji}\,\psi_{\mathrm{mes}}(\vq_{ji}),
\label{eq:graph-context}
\end{equation}
where $\psi_{\mathrm{mes}}$ is an MLP. We use the robot node's context $\vc$ and concatenate the queried CBF parameter according to~\eqref{eq:network_embedding}.

\subsection{Probabilistic Ensemble Architecture}
\label{app:penn}
The PENN consists of $B$ independently initialized neural networks with parameters $\bm\theta_b$. Each network predicts the mean and diagonal covariance in~\eqref{eq:gaussian}, and uniform weighting yields the mixture~\eqref{eq:penn}. The member-specific variances represent heteroscedastic aleatoric variability, while disagreement among member predictions is used for epistemic verification. Ensemble inference can be evaluated in parallel~\cite{kim_bridging_2023}. The architecture used in the experiments is specified in Appendix~\ref{app:experimental_setup}.

\subsection{Obstacle-Avoidance Weighting in the Safety Loss}
\label{app:safety-weighting}
For the navigation case studies, the generic weighting function in~\eqref{eq:safety_density} is chosen as
\begin{align}
\beta^j(\vx)&\coloneqq\beta_1\exp\!\left(\beta_2(1-\cos\Delta\theta_j)\right),
\label{eq:beta_j}\\
D_j(\vx)&\coloneqq1+\beta^j(\vx)\operatorname{dist}(\vp,\vp_{\mathrm{obs},j})^2,
\label{eq:D_j}
\end{align}
where $\beta_1,\beta_2>0$, $\vp$ is the robot position, $\vp_{\mathrm{obs},j}$ is the obstacle center, and $\operatorname{dist}$ is defined by~\eqref{eq:graph-distance}. The angle $\Delta\theta_j$ is measured between the robot's motion direction and the direction toward the obstacle center. When the speed is zero, the prescribed robot heading supplies the direction. At zero surface clearance, set $D_j=1$, so no angle is needed at coincident positions. This construction gives $D_j\ge1$ and places greater weight on nearby obstacles in the direction of travel. It is one application-specific choice; the certificate in \autoref{prop:loss-certificate} depends on the unweighted term $\lambda^j$ in~\eqref{eq:safety_density}, not on circular-obstacle geometry.

\begin{figure}[tbp]
\centering
\includegraphics[width=0.99\linewidth]{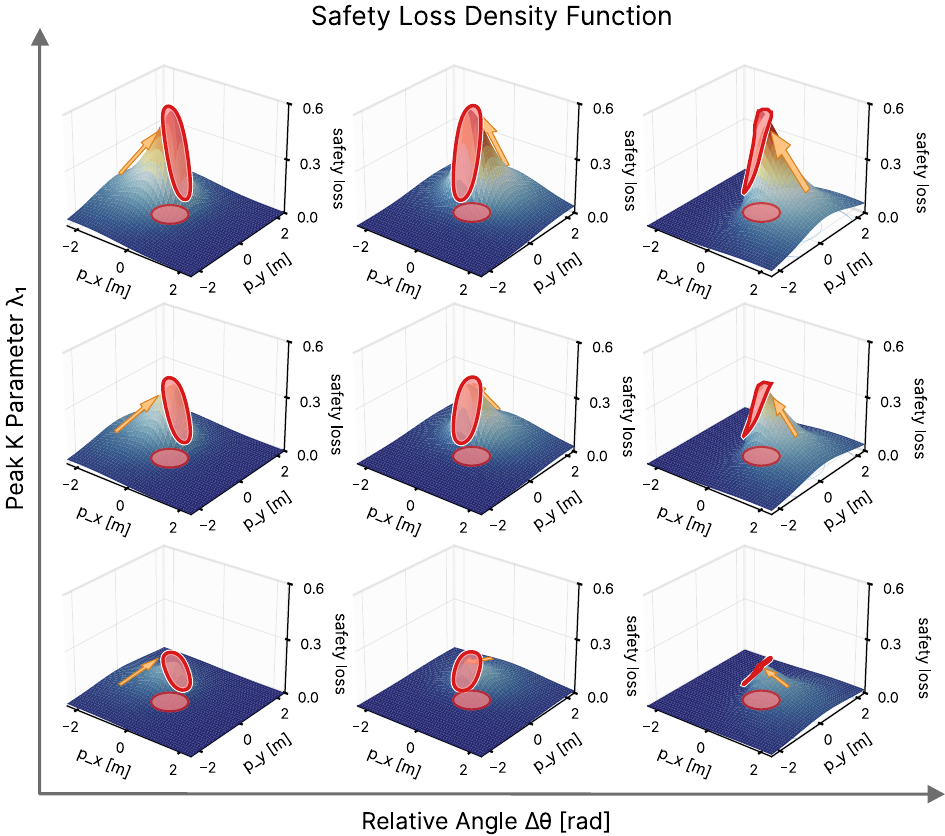}
\caption{
Safety-loss density $\Phi(\vx,\vu;\widetilde{\bm{\alpha}})$ in a canonical single-obstacle setting. A circular obstacle with radius $1\,\mathrm{m}$ is placed at the origin, $\vu$ is fixed to zero, and $\Phi(\vx,\vu;\widetilde{\bm{\alpha}})$ is evaluated over planar robot positions $\vp=[p_x,p_y]^\top$ while varying the peak parameter $\lambda_1$ and the relative angle $\Delta\theta$. The visualization shows how the loss concentrates near the obstacle and is reshaped by the heading-dependent dissipation term, assigning larger loss to configurations that are more directly aligned with the robot's approach direction. The red disk marks the obstacle 
collision region, and the orange arrow indicates an example robot motion direction 
used to interpret $\Delta\theta$.
}
\label{fig:safety-loss-density}
\end{figure}

\subsection{Training Data Generation and Loss Function}
\label{app:training}
\label{subsec:training}
To train the model~\eqref{eq:penn}, we generate a dataset $\calD=\{(\calG_k,\widetilde{\bm\alpha}_k,Y_k)\}_{k=1}^{N_{\mathrm{data}}}$ from simulated closed-loop rollouts. For each sample, we draw the initial state $\vx_0$, the number of constraints $M$, their parameters $\mathbf o^j$, and the candidate CBF parameter $\widetilde{\bm\alpha}_k$. The navigation graph is constructed as in Appendix~\ref{app:graph}, and the controller $\pi(\cdot;\widetilde{\bm\alpha}_k)$ generates the rollout~\eqref{eq:rollout}.

For each admissible rollout, we record $Y_k=[\phi_k,\delta_k]^\top$ using~\eqref{eq:finite-horizon-risk} and~\eqref{eq:performance-metric}. The horizon $T$ is chosen sufficiently long to reach the goal when feasible. In training scenarios with a terminal controlled-invariant neighborhood $\calX_f\subset\calS$ and an admissible stabilizing controller, a trajectory that reaches $\calX_f$ can be continued safely under that controller~\cite{agrawal_gatekeeper_2024,kim_backupbased_2026}. Such an offline construction does not establish recursive feasibility during deployment. If the rollout becomes infeasible or violates a constraint before time $T$, its safety label is assigned a large positive penalty. The perturbation convention used for the reported experiments is specified in Appendix~\ref{app:experimental_setup}.

The dataset is split into training and validation sets in a 7:3 ratio. Test cases are generated separately, with different constraint counts and configurations as described in \autoref{sec:simulation}. The calibration trajectories used in \autoref{subsec:cp} are also held out from model training.

Each ensemble member is trained by minimizing a Gaussian negative log-likelihood loss. Let $\mathcal I_{\mathrm{train}}$ index the training split and let $X^{(k)}$ be the embedding associated with sample $k$. The loss, up to an additive constant and positive scale, is
\begin{equation}
\begin{aligned}
\calL_{\mathrm{train}}(\bm\theta_b)
={}&\sum_{k\in\mathcal I_{\mathrm{train}}}\Bigl[
\log\det\VSigma_{\bm\theta_b}(X^{(k)})\\
&+(\vmu_{\bm\theta_b}(X^{(k)})-Y_k)^\top
\VSigma_{\bm\theta_b}(X^{(k)})^{-1}\\
&\hspace{2em}\cdot(\vmu_{\bm\theta_b}(X^{(k)})-Y_k)\Bigr].
\end{aligned}
\label{eq:train_loss}
\end{equation}

\begin{figure}[tbp]
\centering
\includegraphics[width=0.99\linewidth]{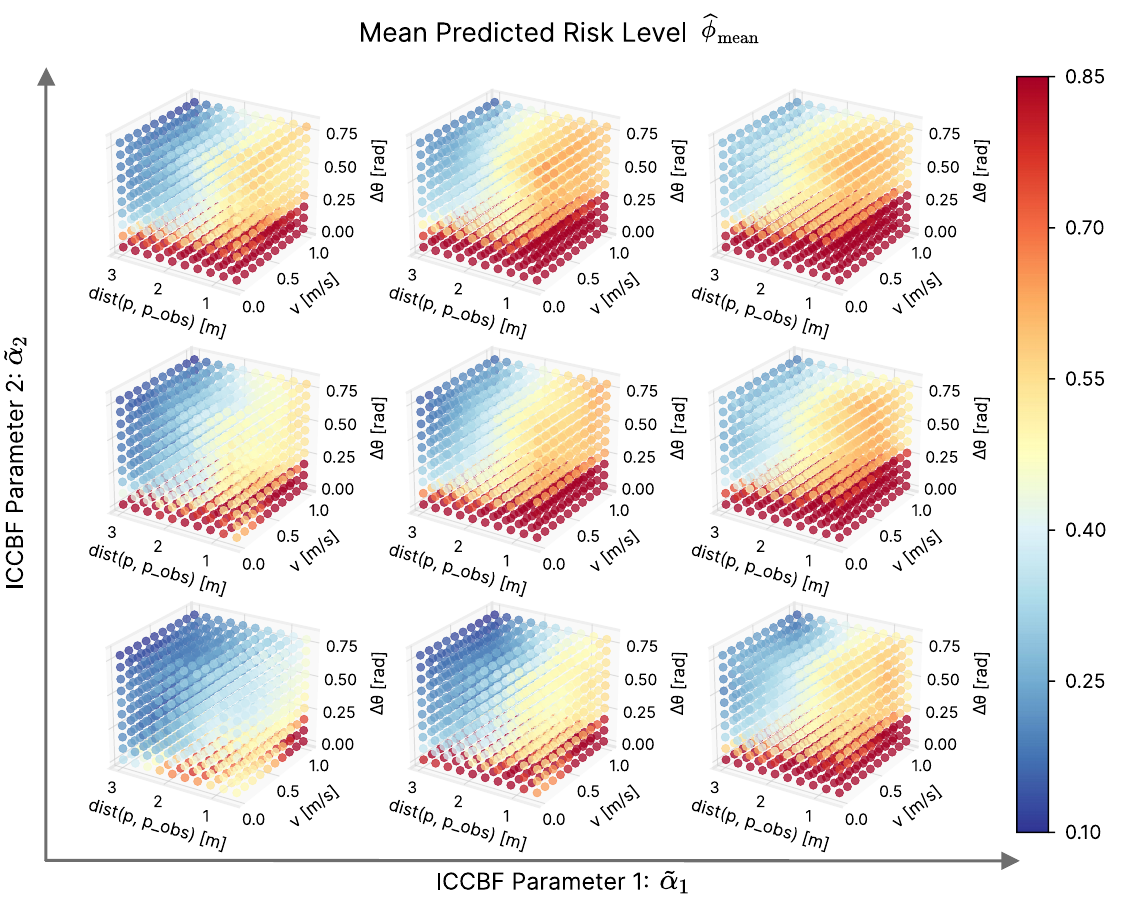}
\caption{
Sensitivity visualization of the ensemble-mean predicted safety-risk level 
$\widehat{\phi}_{\mathrm{mean}}$ from the trained GAT-PENN model used in the 
dynamic-unicycle experiments. Each subplot corresponds to one candidate parameter pair $(\widetilde{\alpha}_1,\widetilde{\alpha}_2)$, while the axes show obstacle distance, robot speed, and relative angle to the obstacle.
}
\label{fig:predicted-risk-level}
\end{figure}

\textbf{Visualization of the learned risk level predictor:} We additionally inspected one-obstacle slices of the trained GAT-PENN predictor, evaluated using the dynamic unicycle model, and observed that the ensemble-mean predicted risk increases for smaller obstacle distances, larger robot speeds, and headings more directly oriented toward the obstacle (see Fig.~\ref{fig:predicted-risk-level}). Online verification uses the full predictive distribution, including epistemic gating and the CVaR condition, rather than the ensemble mean alone.

\subsection{Experimental Implementation}
\label{app:experimental_setup}

OA-CBF w/ GAT and OA-CBF w/ FC use CBF-QP safety filters on the dynamic unicycle, Quad2D, Quad3D, and kinematic bicycle benchmarks. The same finite-horizon verification and parameter-selection procedure is used across these systems. The VTOL case study uses the MPC-CBF realization. All implementations use JAX.

For MPC-CBF controllers, the stage and terminal costs are quadratic tracking costs that drive the robot toward a goal state or waypoint. The cost matrices are tuned once for each dynamical system so that, in obstacle-free environments, the nominal MPC controller reaches the goal without excessive overshoot. For QP-based baselines, the nominal input is generated by the hand-tuned PD tracking controller and minimally modified through the CBF-QP safety layer. In all cases, the adapted or fixed CBF parameter enters only through the CBF/ICCBF safety constraint. %%

The offline dataset is generated following the procedure in \autoref{subsec:training}. For each sampled scenario, we sample an initial robot state, an obstacle environment, and a candidate CBF parameter~$\widetilde{\bm{\alpha}}$. We then roll out the corresponding, potentially unsafe, closed-loop system using this candidate parameter and record the resulting safety and performance labels. Infeasible rollouts and collision rollouts are assigned a large risk penalty so that they are separated from marginally safe cases. The training and testing environments are generated independently. In particular, the test environments include obstacle configurations that are not matched to the training distribution, which allows us to evaluate the scalability of the graph-attention encoding.

The PENN model used by OA-CBF consists of $B=4$ ensemble members. Each ensemble member is implemented as a five-layer stochastic multilayer perceptron with hidden widths $40$, $80$, $120$, and $40$ and ReLU activations. For the graph-attention variant, the graph encoder described in \autoref{subsec:graph} is trained jointly with the PENN prediction head and produces a $16$-dimensional robot-context embedding that is concatenated with the candidate CBF parameter. To model observation and dynamics uncertainty during offline data generation, we add $3\%$ i.i.d. Gaussian noise to the neural-network input features, excluding the candidate CBF parameters. The CVaR upper-tail probability is set to $\varepsilon=0.01$. The epistemic uncertainty threshold $D_{\mathrm{thr}}$ is calibrated at coverage level $1-\delta_{\mathrm{cal}}=0.95$ using graph-level Jensen-Renyi divergence scores computed from $200{,}000$ calibration trajectories
for each system.

The randomized tests use 250 cases at each sensor-noise multiplier $0$, $0.5$, $1$, and $2$, giving 1,000 cases per system. Initial conditions, obstacle motion, routes, and noise realizations are shared across methods, and the four test layouts contain 24, 36, 48, or 60 obstacles. The final comparison contains 26,000 controller evaluations.

\subsection{System Models}
\label{app:system_models}

We evaluate OA-CBF on five dynamical systems: a dynamic unicycle, a planar quadrotor (Quad2D), a three-dimensional quadrotor (Quad3D), a kinematic bicycle model, and a VTOL quadplane. The first four systems are used for randomized benchmark comparisons, while the VTOL quadplane is used as a high-dimensional nonlinear case study.

\subsubsection{Dynamic Unicycle}

The dynamic unicycle state is $\vx = [p_x,p_y,\theta,v]^\top$ where $(p_x,p_y)$ is the planar position, $\theta$ is the heading angle, and $v$ is the forward speed. The input is $\vu = [a,\omega]^\top$, where $a$ is the longitudinal acceleration and $\omega$ is the angular velocity. The dynamics are
\begin{equation}
\begin{alignedat}{2}
    \dot p_x &= v\cos\theta,
    &\qquad
    \dot p_y &= v\sin\theta, \\
    \dot\theta &= \omega,
    &\qquad
    \dot v &= a.
\end{alignedat}
\label{eq:dynamic_unicycle}
\end{equation}

\subsubsection{Quad2D}
The Quad2D benchmark uses a planar thrust-controlled rotor model with state $\vx=[p_x,p_z,\theta,v_x,v_z,\dot\theta]^\top$, where $(p_x,p_z)$ is the planar position, $\theta$ is the pitch angle, and $(v_x,v_z)$ is the planar velocity. The input is $\vu=[u_{\textup{r}},u_{\textup{l}}]^\top$, where $u_{\textup{r}}$ and $u_{\textup{l}}$ are the right and left rotor thrusts. The implemented dynamics are
\[
\begin{aligned}
\dot p_x &= v_x, &
\dot p_z &= v_z, \\
\dot v_x &= -\frac{\sin\theta}{m}(u_{\textup{r}}+u_{\textup{l}}), &
\dot v_z &= -g+\frac{\cos\theta}{m}(u_{\textup{r}}+u_{\textup{l}}), \\
\ddot\theta &= \frac{r}{I}(u_{\textup{r}}-u_{\textup{l}}).
\end{aligned}
\]

\subsubsection{Quad3D}
The Quad3D benchmark uses a linearized 6-DOF quadrotor model with 12 states and 4 motor-force inputs~\cite{annaswamy_integration_2023}. The state is $\vx =
[\,p_x,p_y,p_z,\phi,\theta,\psi,v_x,v_y,v_z,\dot\phi,\dot\theta,\dot\psi\,]^\top$, where $(p_x,p_y,p_z)$ is the position, $(\phi,\theta,\psi)$ are the roll, pitch, and yaw angles, and $(v_x,v_y,v_z)$ is the translational velocity. The control input is $\vu=[u_1,u_2,u_3,u_4]^\top$, where $u_i$ denotes the $i$-th motor-force deviation from the hover equilibrium. These inputs are mapped to the incremental vertical force and body torques by
\[
\begin{bmatrix}
F\\ \tau_\phi\\ \tau_\theta\\ \tau_\psi
\end{bmatrix}
=
\begin{bmatrix}
1 & 1 & 1 & 1\\
L & 0 & -L & 0\\
0 & L & 0 & -L\\
\nu & -\nu & \nu & -\nu
\end{bmatrix}
\begin{bmatrix}
u_1\\u_2\\u_3\\u_4
\end{bmatrix}.
\]
Here $F$ is the incremental vertical force about hover, $\tau_\theta$, $\tau_\phi$, and $\tau_\psi$ are the pitch, roll, and yaw torques, respectively, $L$ is the quadrotor arm length, and $\nu$ is the rotor yaw-torque coefficient.

The dynamics are
\begin{equation}
\begin{aligned}
    \dot p_x &= v_x, &
    \dot p_y &= v_y, &
    \dot p_z &= v_z,\\
    \dot v_x &= g\theta, &
    \dot v_y &= -g\phi, &
    \dot v_z &= \frac{1}{m}F,\\
    \ddot\phi &= \frac{1}{I_x}\tau_\phi, &
    \ddot\theta &= \frac{1}{I_y}\tau_\theta, &
    \ddot\psi &= \frac{1}{I_z}\tau_\psi .
\end{aligned}
\label{eq:quad3d}
\end{equation}
In these equations, $m$ is the quadrotor mass, $g$ is the gravitational acceleration, and $I_x$, $I_y$, and $I_z$ are the moments of inertia about the body-fixed roll, pitch, and yaw axes, respectively.

\subsubsection{Kinematic Bicycle}
The kinematic bicycle model is used to evaluate OA-CBF with a non-SDF-type CBF candidate. The state is $\vx=[p_x,p_y,\theta,v]^\top$, where $(p_x,p_y)$ denotes the center-of-mass position, $\theta$ is the heading, and $v$ is the forward speed. The physical inputs are the longitudinal acceleration $a$ and the front-wheel steering angle $\delta$. Let $\ell_f$ and $\ell_r$ denote the distances from the center of mass to the front and rear axles. The slip angle is $\beta = \tan^{-1}\!\left( \frac{\ell_r}{\ell_f+\ell_r}\tan\delta \right)$.
Using the small-slip approximation, the model can be written in control-affine form with input $\vu=[a,\beta]^\top$:
\begin{equation}
\underbrace{\begin{bmatrix}
\dot p_x\\
\dot p_y\\
\dot\theta\\
\dot v
\end{bmatrix}}_{\dot{\vx}}
=
\underbrace{\begin{bmatrix}
v\cos\theta\\
v\sin\theta\\
0\\
0
\end{bmatrix}}_{f(\vx)}
+
\underbrace{\begin{bmatrix}
0 & -v\sin\theta\\
0 & \phantom{-}v\cos\theta\\
0 & v/\ell_r\\
1 & 0
\end{bmatrix}}_{g(\vx)}
\underbrace{\begin{bmatrix}
a\\
\beta
\end{bmatrix}}_{\vu}.
\label{eq:kinematic_bicycle}
\end{equation}
The dynamic obstacles in this benchmark follow constant-velocity unicycle dynamics. For obstacle $j$, $\vx_{\mathrm{obs}}^j = [p_{\mathrm{obs},x}^j, p_{\mathrm{obs},y}^j, \theta_{\mathrm{obs}}^j, v_{\mathrm{obs}}^j]^\top$,
with $\dot p_{\mathrm{obs},x}^j = v_{\mathrm{obs}}^j\cos\theta_{\mathrm{obs}}^j$, $\dot p_{\mathrm{obs},y}^j = v_{\mathrm{obs}}^j\sin\theta_{\mathrm{obs}}^j$. Obstacle states are assumed to be observable.

\subsubsection{VTOL Quadplane}
We also consider a two-dimensional VTOL quadplane as a case study (see Fig.~\ref{fig:quadplane}). This platform has nonlinear, high-dimensional dynamics and strict input limits, making CBF-parameter tuning particularly challenging during transitions between cruise and hover flight.

The state is $\vx=[\vp^\top,\vv^\top,\theta,q]^\top$, where $\vp=[p_x,p_z]^\top\in\RealSpace^2$ is the inertial-frame position, $\vv\in\RealSpace^2$ is the inertial-frame velocity, $\theta$ is the pitch angle, and $q=\dot\theta$ is the pitch rate. The control input is $\vu=[\delta_f,\delta_r,\delta_p,\delta_e]^\top$, where $\delta_f,\delta_r\in[0,1]$ are the front and rear rotor throttles, $\delta_p\in[0,1]$ is the pusher throttle, and $\delta_e\in[-0.5,0.5]$ is the elevator deflection.

The rigid-body equations are
\begin{equation}
\begin{alignedat}{2}
    \dot{\vp} &= \vv,
    &\qquad
    \dot\theta &= q, \\
    m\dot{\vv} &= -mg\ve_z + \mathbf R(\theta)\vF_{\mathrm{total}},
    &\qquad
    J_y\dot q &= M_{\mathrm{total}}.
\end{alignedat}
\label{eq:vtol_rigid_body}
\end{equation}
where $m$ is the mass, $J_y$ is the pitch moment of inertia, $\ve_z=[0,1]^\top$, $\vF_{\mathrm{total}}$ is the total body-frame force, and $M_{\mathrm{total}}$ is the total pitching moment.

Let $[u,w]^\top=\mathbf R^{-1}(\theta)\vv$ denote the body-frame velocity, $V=\sqrt{u^2+w^2}$ the airspeed, and $\alpha_a=\tan^{-1}(-w/u)$ the angle of attack. We use $\alpha_a$ for angle of attack to avoid conflict with the CBF parameter $\widetilde{\bm{\alpha}}$. The lift coefficient is modeled as a smooth blend between a linear lift curve and a flat-plate model~\cite{willis_pitch_2022},
\begin{equation}
\begin{aligned}
C_L(\alpha_a)
&=
(1-\sigma(\alpha_a))(C_{L_0}+C_{L_\alpha}\alpha_a) \\
& \qquad \qquad
+\sigma(\alpha_a)(2\sin\alpha_a\cos\alpha_a) , \nonumber
\end{aligned}
\label{eq:vtol_lift_coefficient}
\end{equation}
where
\begin{equation}
\sigma(\alpha_a)
=
\frac{
1+e^{-w_s(\alpha_a-\alpha_s)}
 +e^{w_s(\alpha_a+\alpha_s)}
}{
\left(1+e^{-w_s(\alpha_a-\alpha_s)}\right)
\left(1+e^{w_s(\alpha_a+\alpha_s)}\right)
}. \nonumber
\end{equation}
The drag coefficient is
\[
    C_D(\alpha_a)=C_{D_0}+C_{D_{\alpha^2}}\alpha_a^2. 
\]
The baseline aerodynamic force is
\begin{equation}
\vF_0(V,\alpha_a)
=
\frac{1}{2}\rho V^2 S\mathbf R(\alpha_a)
\begin{bmatrix}
-C_D(\alpha_a)\\
C_L(\alpha_a)
\end{bmatrix}, \nonumber
\end{equation}
and the elevator force is
\begin{equation}
\vF_{\delta_e}(V,\alpha_a)\delta_e
=
\frac{1}{2}\rho V^2 S\mathbf R(\alpha_a)
\begin{bmatrix}
-C_{D_{\delta_e}}(\alpha_a)\\
C_{L_{\delta_e}}(\alpha_a)
\end{bmatrix}
\delta_e. \nonumber
\end{equation}
The rotor thrusts are modeled as $T_f=k_f\delta_f$, $T_r=k_r\delta_r$, and $T_p=k_p\delta_p$. Thus,
\begin{equation}
\vF_{\mathrm{total}}
=
\vF_0(V,\alpha_a)
+
\vF_{\delta_e}(V,\alpha_a)\delta_e
+
\begin{bmatrix}
T_p\\
T_f+T_r
\end{bmatrix}. \nonumber
\end{equation}
The total pitching moment $M_{\mathrm{total}}$ is computed from the rotor, elevator, and aerodynamic moment contributions~\cite{beard_small_2012}.

\begin{figure*}[!t]
\centering
\includegraphics[width=0.95\textwidth]{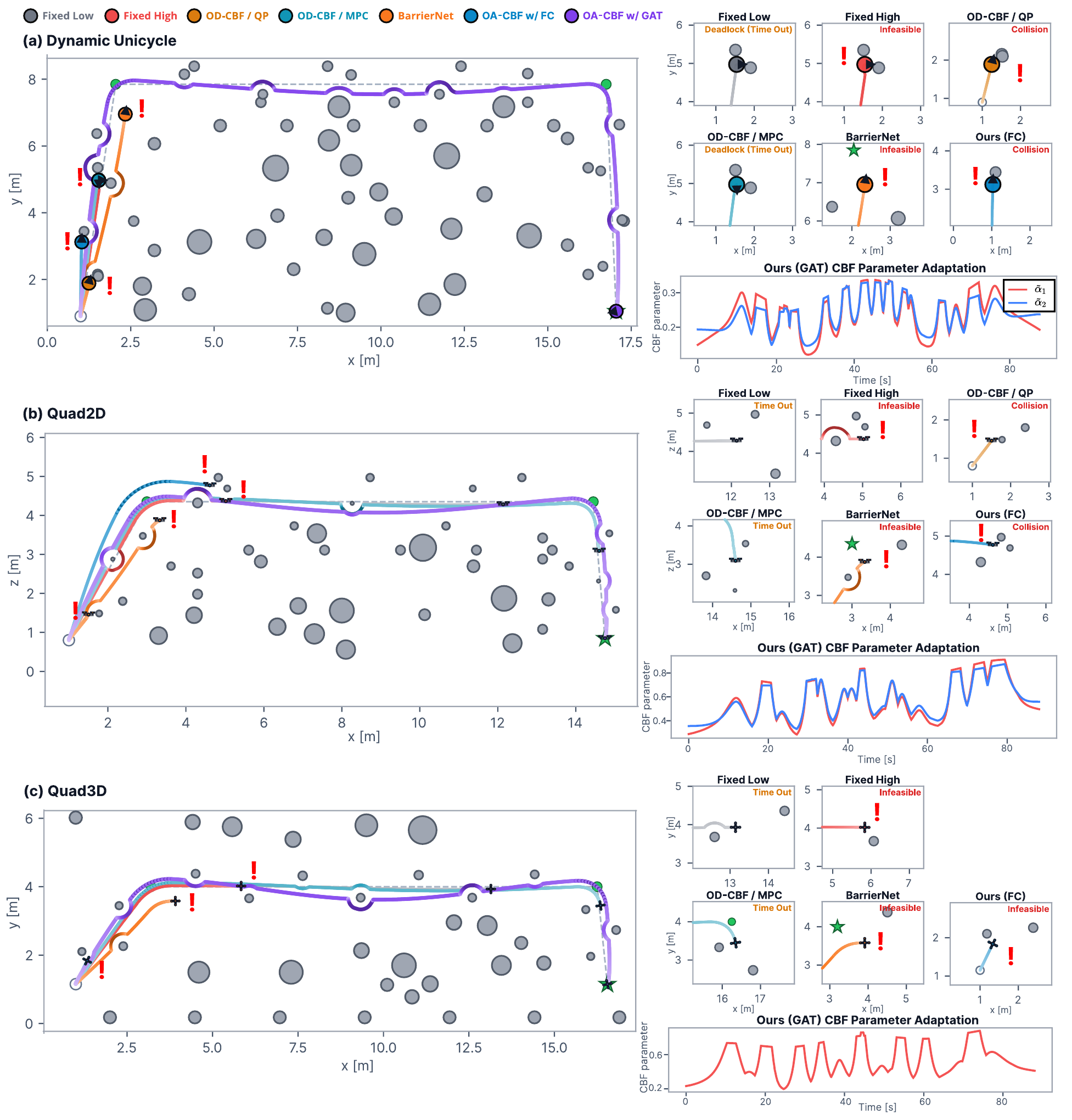}
\caption{Representative trajectories for the distance-based ICCBF benchmarks in \autoref{tab:sdf_results}: (a) dynamic unicycle, (b) Quad2D, and (c) Quad3D. Gray circles denote obstacles, and trajectory color encodes robot speed, with more saturated colors indicating faster motion. Green circles and stars denote reached and next waypoints, respectively. The right panels show representative terminal snapshots for each baseline, labeled by the corresponding outcome: safe, infeasible, collision, or timeout. The bottom plots show the adapted CBF parameters for OA-CBF with graph attention.}
\label{fig:traj}
\vspace{-5pt}
\end{figure*}

\subsection{CBF Candidate Construction}
\label{app:cbf_candidates}

For all systems except the kinematic bicycle benchmark, we use a distance-based collision-avoidance candidate CBF. For obstacle $j$, let $\vp_{\mathrm{obs},j}$ be the obstacle position and let $\ell_{\mathrm{robot}}$ and $\ell_{\mathrm{obs},j}$ be the robot and obstacle radii. The candidate CBF is
\begin{equation}
    \tilde h_j(\vx)
    =
    \|\mathbf P\vp-\vp_{\mathrm{obs},j}\|^2
    -
    (\ell_{\mathrm{robot}}+\ell_{\mathrm{obs},j})^2,
    \label{eq:distance_cbf_candidate}
\end{equation}
where $\mathbf P$ selects the coordinates used for collision avoidance. For planar systems, $\mathbf P$ is the identity. For Quad3D, $\mathbf P$ projects the three-dimensional position onto the horizontal $(x,y)$ plane, so the unsafe set is a planar zone independent of altitude. The CBF-based controllers enforce these inequalities jointly for all obstacles.

For the dynamic unicycle, Quad2D, and VTOL quadplane, the distance-based candidate has relative degree two. Therefore, we use an ICCBF with two linear class-$\mathcal K_\infty$ functions,
\[
    \alpha_i(s)=\widetilde{\alpha}_i s,
    \qquad i\in\{1,2\},
\]
and adapt $\widetilde{\bm{\alpha}} =[\widetilde{\alpha}_1,\widetilde{\alpha}_2]^\top$ online. %

For Quad3D, the same planar obstacle-avoidance constraint has relative degree four under the full rigid-body actuation model. The OA-CBF variants use CBF-QP safety filters, as in the other randomized benchmarks. The MPC-based methods use the sampled-data RK4 CBF formulation of \cite{taylor_safety_2022}, which produces a discrete-time CBF constraint that can be imposed directly in the MPC controller.

To address \textbf{(Q4)}, we also evaluate OA-CBF with a non-distance-based CBF on the kinematic bicycle model. Specifically, we use a Dynamic Parabolic CBF (DPCBF) \cite{park_collision_2026}. Let
\[
    \vp_{\mathrm{rel}}
    =
    \begin{bmatrix}
    p_{\mathrm{obs},x}-p_x\\
    p_{\mathrm{obs},y}-p_y
    \end{bmatrix},
\]
and let $\vv_{\mathrm{rel}}$ be the relative velocity between the obstacle and the robot. Define the line-of-sight angle
\[
    \psi_{\mathrm{los}}
    =
    \operatorname{atan2}(p_{\mathrm{rel},y},p_{\mathrm{rel},x})
\]
and rotate the relative velocity into the line-of-sight frame:
\[
    \tilde{\vv}_{\mathrm{rel}}
    =
    \begin{bmatrix}
    \tilde v_{\mathrm{rel},x}\\
    \tilde v_{\mathrm{rel},y}
    \end{bmatrix}
    =
    \mathbf R(-\psi_{\mathrm{los}})\vv_{\mathrm{rel}}.
\]

Let $r_j=\ell_{\mathrm{robot}}+\ell_{\mathrm{obs},j}$ and
\[
    d_j(\vx)
    =
    \sqrt{\|\vp_{\mathrm{rel}}\|^2-r_j^2},
\]
which is evaluated on the collision-free domain
$\|\vp_{\mathrm{rel}}\|>r_j$. The DPCBF uses two state-dependent design maps,
\[
    \lambda_j(\vx)
    =
    k_\lambda
    \frac{d_j(\vx)}{\|\vv_{\mathrm{rel}}\|},
    \qquad
    \mu_j(\vx)
    =
    k_\mu d_j(\vx),
\]
where $k_\lambda,k_\mu>0$. The CBF candidate is
\begin{equation}
    h_j^{\mathrm{DPCBF}}(\vx)
    =
    \tilde v_{\mathrm{rel},x}
    +
    \lambda_j(\vx)\tilde v_{\mathrm{rel},y}^2
    +
    \mu_j(\vx).
    \label{eq:dpcbf_candidate}
\end{equation}
Unlike the distance-based CBF in \eqref{eq:distance_cbf_candidate}, this candidate evaluates safety in the relative-velocity plane. The parabolic boundary shifts and changes curvature as the relative distance and relative speed vary, so the controller does not treat all motion toward an obstacle as immediately unsafe. For the kinematic bicycle model, this DPCBF has relative degree one. The CBF candidates and their corresponding relative-degree structures are summarized in~\autoref{tab:cbf_relative_degree}.

\subsection{Additional Trajectory Plots}
\label{app:additional_results}

Figure~\ref{fig:traj} provides the trajectory comparisons for the distance-based ICCBF results in \autoref{tab:sdf_results}.

\subsection{Baseline Implementation Details} \label{app:baseline_details}

This appendix provides additional implementation details for the baselines used in \autoref{sec:simulation}. For the MPC-based benchmarks, the fixed-parameter baselines use linear CBF parameters in $(0,1)$ for the discrete-time CBF or ICCBF constraints. The OA-CBF variants use CBF-QP safety filters on the four randomized systems and MPC-CBF in the VTOL case study. The DPCBF benchmark follows the original CBF-QP realization in~\cite{park_collision_2026}.

\textbf{(a) Fixed Low:}
    This baseline uses a fixed low CBF parameter for all time. Low CBF parameters typically make the controller more conservative because the safety constraint is activated earlier and the resulting inner safe set is smaller. We use $\widetilde{\alpha}_1=\widetilde{\alpha}_2=0.01$ for the dynamic unicycle and Quad2D, $\widetilde{\alpha}=0.01$ for Quad3D, $\widetilde{\alpha}_1=\widetilde{\alpha}_2=0.05$ for the VTOL case study, and $\widetilde{\alpha}=0.1$ for the kinematic bicycle.

\textbf{(b) Fixed High.}
    This baseline uses a fixed high CBF parameter. High CBF parameters generally allow more aggressive behavior when the system is far from obstacles, but they can also cause loss of feasibility near obstacles under input constraints. We use $\widetilde{\alpha}_1=\widetilde{\alpha}_2=0.35$ for the dynamic unicycle and VTOL case study, $\widetilde{\alpha}_1=\widetilde{\alpha}_2=0.99$ for Quad2D, and $\widetilde{\alpha}=0.99$ for Quad3D. For the kinematic-bicycle DPCBF benchmark, we set $\widetilde{\alpha}=70.0$ because the closed-loop behavior under the DPCBF-QP formulation empirically saturates around this parameter value.

\textbf{(c) OD-CBF with QP.}
    This baseline uses the optimal-decay CBF-QP method \cite{zeng_safetycritical_2021_decay, ong_properties_2025}. The CBF parameters are optimized online inside an instantaneous CBF-QP and penalized in the objective. For relative-degree-two CBFs, both CBF coefficients are optimized. This method is included when the corresponding continuous-time CBF-QP is well defined.

\textbf{(d) OD-CBF with MPC.}
    This baseline uses the optimal-decay MPC-CBF method \cite{zeng_enhancing_2021}, which incorporates online decay-rate optimization inside the MPC-CBF formulation. This baseline uses the same dynamics, constraints, prediction horizon, and nominal cost as the other MPC-based controllers.

\textbf{(e) BarrierNet.}
    We include BarrierNet~\cite{xiao_barriernet_2023} as a differentiable CBF-QP-layer baseline. Its upstream neural network predicts task-directed reference-control and CBF-related quantities from the current state and environment, and the final layer solves a multi-constraint CBF-QP. At each control step, it selects the five closest obstacles by signed clearance and constructs a fixed-size neural input from per-obstacle features $(\Delta p_x,\Delta p_y,\Delta\theta,d_{\mathrm{clr}},v)$, where $\Delta p_x$ and $\Delta p_y$ denote the obstacle position relative to the robot in the robot body frame, $\Delta\theta$ denotes the relative bearing angle to the obstacle, $d_{\mathrm{clr}}$ denotes the signed clearance to the obstacle boundary, and $v$ denotes the robot speed. The raw QP context contains the reduced robot state, the goal, and the selected obstacle states. A shared per-obstacle encoder predicts obstacle-wise CBF parameters, while a control head predicts a nominal correction around the reference input. In our implementation, the differentiable QP layer is trained with slack variables and without hard input bounds to avoid degenerate training batches; during deployment, the QP is evaluated with input bounds, and the computed input is clipped at the end. The network is trained by imitation on safe, goal-reaching CBF-QP rollouts generated from the same obstacle distribution, while failed, colliding, or timeout rollouts are discarded. Because BarrierNet relies on a continuous-time CBF-QP layer, it is evaluated only when the corresponding CBF-QP constraint is well defined.

\textbf{(f) OA-CBF with fully connected encoding (OA-CBF w/ FC).}
    This baseline follows the fully connected environment encoding implementation provided by \cite{kim_learning_2025,kim_how_2025}. In that implementation, the neural-network input contains the robot state, the candidate CBF parameter, and the closest-obstacle information only. This produces a fixed-dimensional input vector and avoids padding, but it can discard safety-relevant information when multiple obstacles are simultaneously relevant. This baseline is included to isolate the effect of replacing the fixed-dimensional closest-obstacle encoding with the proposed graph-attention encoding.

\textbf{(g) OA-CBF with graph-attention encoding (OA-CBF w/ GAT).}
    This is the proposed method. The environment is represented as a graph consisting of the robot, the goal, and all observed obstacles. Attention-based aggregation produces a fixed-dimensional context vector without requiring padding, truncation, or heuristic obstacle selection. The resulting context is concatenated with the candidate CBF parameter and evaluated by the PENN.

\fi % arxivversion: Appendix A

\else
  \begin{abstract}
Control barrier functions (CBFs) provide a tractable mechanism for enforcing safety constraints, but the system response depends strongly on their class-$\calK$ function parameters. Under input constraints, these parameters affect both the feasibility of the underlying optimization problem and the tradeoff between conservative and aggressive responses. This paper proposes Online Adaptive CBF (OA-CBF), a framework for adapting CBF parameters at runtime to improve task performance under input constraints. We introduce locally validated CBF parameters as a finite-horizon validation target and train a probabilistic ensemble neural network to evaluate queried parameters rather than directly predict them. A graph-attention encoder provides a fixed-dimensional representation of variable-size environments, while the probabilistic ensemble characterizes epistemic uncertainty through disagreement among its members and aleatoric uncertainty through their individual predictive variances. To account for these uncertainties, we use conformal prediction to reject potentially out-of-distribution predictions, and distributionally robust conditional value-at-risk to assess the predicted safety risk across ensemble members. Among the retained candidates, OA-CBF selects the parameter with the best predicted progress metric and applies it through either an MPC-CBF or CBF-QP safety filter. Simulation studies on dynamic unicycle, planar and three-dimensional quadrotor, kinematic bicycle, and VTOL quadplane benchmarks show that OA-CBF reduces the conservatism of fixed-parameter CBF controllers while maintaining low collision and infeasibility rates. \href{https://github.com/tkkim-robot/online_adaptive_cbf}{\textcolor{red}{[Code]}}  \href{https://www.taekyung.me/oa-cbf}{\textcolor{red}{[Project Page]}}\footnote{Project page: \href{https://www.taekyung.me/oa-cbf}{https://www.taekyung.me/oa-cbf}}
\end{abstract} %%

\section{INTRODUCTION}
\subsection{Background}

\begin{figure*}[t]
\centering
\includegraphics[width=0.99\textwidth]{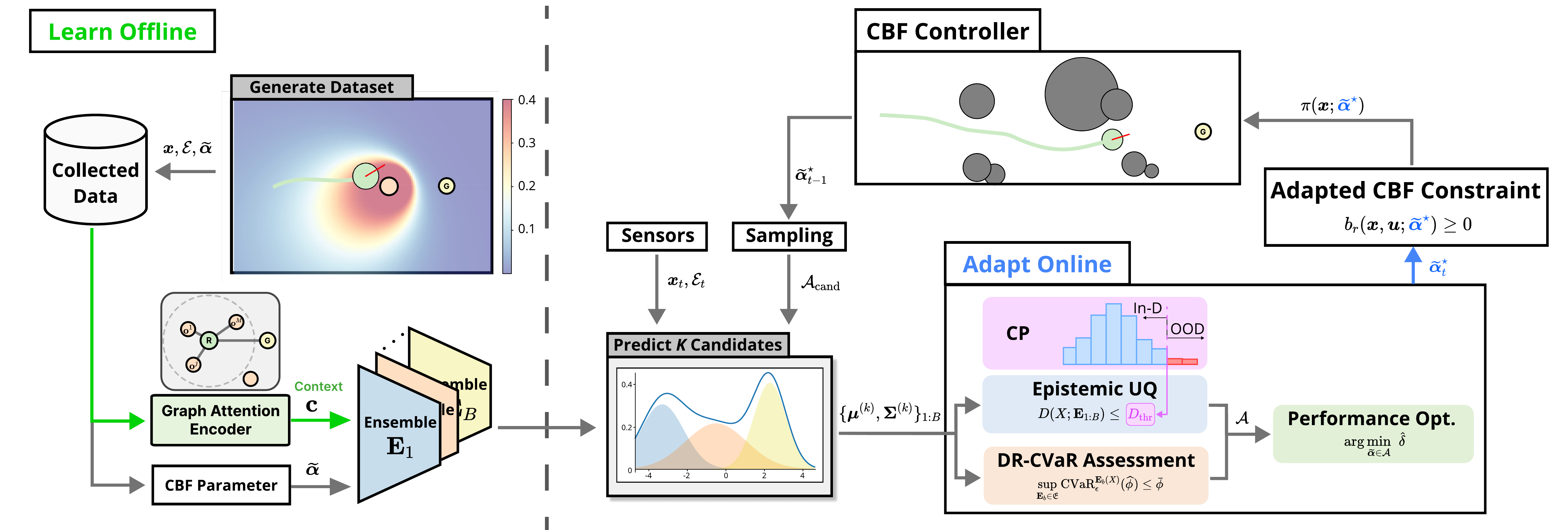}
\caption{
Overview of OA-CBF. Offline rollouts are used to train a graph-attention encoder and probabilistic ensemble neural networks that predict safety and 
performance metrics for queried CBF parameters. During deployment, candidate 
parameters are sampled online, evaluated in batch, filtered by an epistemic uncertainty gate calibrated using conformal prediction and an aleatoric uncertainty gate evaluated by distributionally robust CVaR, and selected according to predicted task performance before updating the CBF-based controller.
}
\label{fig:overview}
\label{fig:main_diagram}
\vspace{-5pt}
\end{figure*}

Robotic and autonomous systems operating in cluttered or dynamic environments must satisfy hard safety constraints while still making efficient progress toward their goals. Control Barrier Functions (CBFs) have become a standard tool for safety-critical control because they synthesize safety filters that minimally modify a nominal controller while enforcing forward invariance of a safe set~\cite{ames_control_2019}. For robotic systems with high-relative-degree constraints or hard actuator limits, this idea has been extended through higher-order CBFs~\cite{xiao_control_2019}, input-constrained CBFs~\cite{agrawal_safe_2021}, and sampled-data MPC-CBF formulations~\cite{zeng_safetycritical_2021, zeng_enhancing_2021}.

A persistent practical difficulty is that the behavior of a CBF-based controller depends strongly on the chosen class-$\mathcal K$ functions, or equivalently on the decay parameters used in the CBF constraint~\cite{xiao_feasibilityguided_2020, parwana_ratetunable_2025, kim_your_2026}. These parameters do not merely tune performance; under input constraints, they also influence feasibility and the size of the corresponding safe set that the controller can render forward invariant~\cite{agrawal_safe_2021, kim_how_2025}. Conservative parameters typically activate the safety constraint early and preserve feasibility, but they can lead to unnecessarily slow or circuitous motion. Aggressive parameters may improve nominal progress, but they can make the safety filter infeasible near obstacles or under limited control authority.

This tradeoff is especially restrictive when a single fixed parameter is used across all states and environments. A parameter that is safe for one operating condition may be overly conservative in another, while a parameter that performs well far from obstacles may render the safety filter infeasible when the system approaches the boundary of the safe set with insufficient actuation margin. Therefore, effective CBF deployment requires a mechanism that can adapt the CBF parameter online while rejecting unsafe or unreliable updates before they are applied to the controller.

Learning-based adaptation offers a way to reduce this conservatism, but direct prediction of CBF parameters introduces a new challenge: the learned model may be overconfident on states, constraint functions, or candidate parameters that are poorly represented in the training data. Such overconfident extrapolation can be more harmful than conservative behavior, because an invalid parameter update directly changes the feasible-control set of the safety filter. This motivates a different viewpoint: rather than treating the learned model as a direct parameter predictor, we use it as an uncertainty-aware evaluator of candidate CBF parameters.

In this paper, we develop Online Adaptive CBF~(OA-CBF), a framework that verifies and adapts CBF parameters using predictions over a finite horizon. The proposed method evaluates the predicted safety and performance of candidate parameters, applies epistemic and aleatoric uncertainty tests, and then selects among the retained candidates according to predicted task performance. The epistemic test controls the rejection probability of in-distribution trajectories, while the aleatoric test bounds risk under the predictive model. This provides a probabilistic verification step before each parameter update; the overall offline-training and online-adaptation workflow is summarized in Fig.~\ref{fig:overview}.
\subsection{Our Contributions}

This paper makes the following contributions.
\begin{itemize}
    \item We formulate online CBF-parameter adaptation under input constraints through the notion of locally validated CBF parameters. Instead of requiring a globally valid CBF parameter for all future states, the proposed formulation verifies candidate parameters over a finite prediction horizon around the current operating condition. We show that repeated local validation over overlapping update intervals preserves safety of the resulting adaptive CBF controller.

    \item We develop an uncertainty-aware verification procedure for CBF-parameter adaptation. A probabilistic ensemble neural network predicts safety-risk and performance metrics for queried candidate parameters, while epistemic uncertainty is quantified through ensemble disagreement and calibrated using conformal prediction. For candidates that pass this epistemic gate, a distributionally robust CVaR condition is used to screen aleatoric risk before the parameter is accepted.

    \item We demonstrate that the OA-CBF framework can be coupled with both MPC-CBF and CBF-QP safety filters. The proposed method is evaluated on dynamic unicycle, Quad2D, Quad3D, kinematic bicycle, and VTOL quadplane systems, covering both distance-based and non-distance-based CBF candidates. The results show that OA-CBF reduces the conservatism of fixed-parameter controllers and reduces safety-failure rates relative to aggressive fixed-parameter, optimal-decay, and learning-based CBF-QP baselines.
\end{itemize}

\subsection{Related Work}

CBF-based safety methods relevant to this paper can be grouped into
constructive synthesis, learned certificate construction, and online adaptation of safety filters. A key distinction is not only \emph{which safety certificate or safety-filter component is constructed or adapted}, but also whether \emph{deployment-time reconfiguration under input constraints} is supported together with an explicit validation mechanism. \autoref{tab:related_work} provides a summary comparison of these methods.

\begin{table*}[t]
\centering
\caption{Comparison of representative directions for constructing or adapting barrier-based safety filters.}
\label{tab:related_work}
\renewcommand{\arraystretch}{1.3}
\resizebox{\textwidth}{!}{%
\begin{tabular}{lcccccc}
\toprule
\textbf{Method} & \textbf{Design Target} & \textbf{Control Limit Treatment} & \makecell[c]{\textbf{Runtime} \\ \textbf{Reconfiguration}} & \textbf{Runtime Uncertainty Treatment} & \textbf{Core Mechanism} \\
\midrule
SOS / Algebraic CBF synthesis~\cite{xu_correctness_2018, clark_verification_2021}
& CBF itself
& Explicit
& \xmark
& N/A
& SOS program / Positivstellensatz \\

CBVF / HJ reachability~\cite{choi_robust_2021}
& Value function used as CBF
& Native
& \xmark
& Bounded-disturbance robustness
& HJI variational inequality \\

RefineCBF / HJ-Patch~\cite{tonkens_refining_2022, tonkens_refining_2026}
& Candidate CBF / Almost-safe value function
& Native
& \checkmark
& N/A
& Warm-started / Local HJ refinement \\

Robust neural CLBF~\cite{dawson_safe_2021}
& NN approximating CLBF
& Problem-dependent
& \xmark
& Modeled uncertainty; No runtime predictive UQ
& Neural certificate learning \\

PNCBF~\cite{so_how_2024}
& Neural CBF via policy value function
& Native
& \xmark
& Post-training verification; No runtime predictive UQ
& Policy evaluation + CBF-QP \\

RBN~\cite{kim_reachability_2025}
& HJ-based neural CBF approximation
& Native
& \xmark
& Conformal safety calibration
& PINN approximation of HJ solutions \\

aCBF / RaCBF~\cite{taylor_adaptive_2020, black_fixedtime_2021}
& Unknown dynamics parameters
& No / Problem-dependent
& \checkmark
& Dynamics uncertainty only
& Lyapunov-based adaptation \\

Adaptive CBF~\cite{xiao_adaptive_2022}
& Time-varying penalty + Relaxation functions
& Explicit
& \checkmark
& Dynamics uncertainty only
& Auxiliary dynamics + HOCBF/CLF \\

Learned CBF-parameter tuning~\cite{ma_learning_2022, xiao_barriernet_2023, gao_online_2023, berducci_learning_2024, mohammad_soft_2025}
& CBF parameters tuned online
& No / Problem-dependent
& \checkmark
& No runtime predictive UQ
& Differentiable QP / RL / GNN policy \\
\midrule

\textbf{Ours (OA-CBF)}
& \textbf{ICCBF parameter $\bm{\alpha}=\{\alpha_1,\ldots,\alpha_r\}$}
& \textbf{Explicit via local validation}
& \checkmark
& \textbf{Epistemic gate + Aleatoric risk screening}
& \textbf{Probabilistic ensemble + Local validation} \\
\bottomrule
\end{tabular}%
}
\end{table*}

\subsubsection{Constructive and Reachability-Based Methods}

Constructive methods include sum-of-squares~(SOS) techniques, which certify nonnegativity of polynomial expressions encoding control and safety conditions through semidefinite programming~\cite{papachristodoulou_tutorial_2005}. Within the CBF literature, Xu \textit{et al.}~\cite{xu_correctness_2018} synthesize CBFs through a combination of SOS programming, physics-based modeling, and optimization, while Clark~\cite{clark_verification_2021} provides SOS-based conditions for verifying and synthesizing CBFs for polynomial control systems. Hamilton-Jacobi~(HJ) reachability offers a general model-based method: the resulting value function characterizes the safe set and, in the Control Barrier-Value Function~(CBVF) formulation, can be used directly in a safety filter with robustness to bounded disturbances~\cite{choi_robust_2021}. These approaches provide strong guarantees, but they suffer from the curse of dimensionality; in practice, direct offline full-grid computation is generally limited to five-dimensional systems~\cite{bansal_hamiltonjacobi_2017,bansal_deepreach_2021}. RefineCBF addresses this gap by warm-starting reachability from a candidate barrier and provably improving safety monotonically~\cite{tonkens_refining_2022}; a recent extension further localizes these updates for in-the-loop refinement under changing environments~\cite{tonkens_refining_2026}.

\subsubsection{Learning-Based Barrier Construction}

Learning-based methods improve computational scalability by approximating certificates or reachability solutions with neural networks. Robust Neural Control Lyapunov-Barrier Functions (CLBFs) learn certificates that retain robustness to modeled uncertainty~\cite{dawson_safe_2021}, and the broader landscape of neural Lyapunov, barrier, and contraction certificates is surveyed in~\cite{dawson_safe_2023}. For high-relative-degree systems under tight actuator limits, Policy Neural CBFs (PNCBFs) learn the value function of a nominal policy and show that the resulting maximum-over-time policy value function is itself a CBF~\cite{so_how_2024}. Reachability Barrier Networks (RBNs) instead use physics-informed neural networks~(PINNs) to approximate discounted HJ solutions and combine them with conformal prediction to obtain probabilistic safety guarantees for the learned approximation~\cite{kim_reachability_2025}. In perception-driven settings, In-Distribution Barrier Functions learn a self-supervised policy filter that keeps trajectories close to the distribution of safe demonstrations~\cite{castaneda_indistribution_2023}, while GCBF+ uses graph neural networks to parameterize graph control barrier functions for scalable distributed multi-agent safety~\cite{zhang_gcbf_2025}. Collectively, these methods substantially expand the practical reach of learned safety filters. During deployment, however, these methods typically rely on the offline learned certificate or calibrated approximation; in general, they neither support online reconfiguration nor explicitly address predictive uncertainty at runtime.

\subsubsection{Online Adaptation of Safety Filters}

A complementary direction keeps the barrier structure fixed and adapts parameters of the safety filter online. Relatedly, episodic learning with Control Lyapunov Functions~(CLFs) was studied for uncertain robotic systems with parametric uncertainty and unmodeled dynamics~\cite{taylor_episodic_2019}. Taylor and Ames~\cite{taylor_adaptive_2020} later proposed an adaptive CBF~(aCBF) formulation for safety under parametric model uncertainty, while Black \textit{et al.}~\cite{black_fixedtime_2021} introduced a fixed-time-stable parameter adaptation law within a robust adaptive CBFs (RaCBFs) framework, guaranteeing fixed-time convergence of the parameter estimates. In a distinct line of work that uses the same ``adaptive CBFs’’ terminology, Xiao \textit{et al.}~\cite{xiao_adaptive_2022} consider time-varying control bounds and noisy dynamics, introducing adaptive penalty or relaxation variables with auxiliary dynamics to address QP feasibility and conservativeness rather than parametric uncertainty estimation.

More closely related to our method, another line of work adapts class-$\mathcal{K}$ functions or related CBF hyperparameters to improve feasibility and reduce conservatism. Rate-Tunable CBFs (RT-CBFs)~\cite{parwana_ratetunable_2025} parameterize the class-$\mathcal{K}$ function and adapt its parameters online, 
deriving pointwise sufficient conditions on the parameter dynamics for compatibility of multiple CBF constraints. Yet, how to systematically synthesize such parameter adaptation laws remains a challenging open problem. Among learning-based approaches, differentiable CBF-QP layers have been used to learn environment- or state-dependent barrier tuning for improved generalization and reduced conservatism~\cite{ma_learning_2022,xiao_barriernet_2023}, while reinforcement-learning-based methods learn online adaptation policies for decentralized navigation~\cite{gao_online_2023}, interactive multi-agent settings~\cite{berducci_learning_2024}, and navigation in unknown environments~\cite{mohammad_soft_2025} by tuning class-$\mathcal{K}$ functions or related constraint hyperparameters. Collectively, these works show that adapting CBF hyperparameters can substantially improve feasibility and performance of the underlying CBF-based controller.
However, directly predicting CBF parameters does not establish whether the resulting closed-loop trajectory satisfies the CBF constraints over a finite prediction horizon under input constraints. This motivates an alternative approach that evaluates the predicted safety and performance of candidate CBF parameters while accounting for uncertainty in these predictions.

\section{PRELIMINARIES \label{sec:preliminaries}}

\subsection{Tangent Cone and Control Barrier Functions\label{subsec:nagumo}}

Consider a continuous-time nonlinear system:
\begin{equation}
\dot{\vx} = f(\vx) + g(\vx)\vu,
\label{eq:ct_dynamics}
\end{equation}
where $\vx \in \StateSpace \subset \Rn$ is the state and $\vu \in \ControlSpace \subset \Rm$ is the control input, with $\ControlSpace$ being the set of admissible controls for System~\eqref{eq:ct_dynamics}. 
We assume that $f: \StateSpace \to \Rn$ and $g: \StateSpace \to \Rnm$ are locally Lipschitz continuous and that, under a feedback law $\vu=\pi(\vx)$, closed-loop solutions exist and are unique on each interval where the policy is defined. The forward-invariance statements below additionally assume that these solutions exist for all $t\ge t_0$. The closed-loop dynamics are
\begin{equation}
\dot{\vx} = f_\textup{cl}(\vx) \coloneqq f(\vx) + g(\vx)\pi(\vx).
\label{eq:closed_loop}
\end{equation}

\begin{definition}[Tangent Cone \cite{blanchini_settheoretic_2008}]\label{def:tangent_cone}
Let $\calS \subset \Rn$ be a closed set. The tangent cone to $\calS$ at $\vx \in \calS$ is defined as
\begin{equation} \label{eq:tangent_cone}
\calT_{\calS}(\vx) \coloneqq \left\{ \vz \in \Rn \mid \liminf_{\tau \downarrow 0} \frac{\textup{dist}(\vx + \tau \vz, \calS)}{\tau} = 0 \right\}.
\end{equation}
\end{definition}
Here, $\textup{dist}(\vy, \calS)$ denotes the distance from a point $\vy \in \Rn$ to $\calS$. If $\calS$ is convex, the ``$\liminf$" can be replaced by a ``$\lim$" in~\eqref{eq:tangent_cone}, and $\calT_{\calS}(\vx)$ remains convex. Moreover, for any $\vx \in \textup{Int}(\calS)$ we have $\calT_{\calS}(\vx)=\Rn$, so that the tangent cone is only non-trivial on the boundary $\partial\calS$ of $\calS$.

\begin{definition}[Forward Invariance]\label{def:forward_invariance}
A set $\calS \subset \Rn$ is rendered forward invariant by a feedback controller $\pi: \calS \to \ControlSpace$ for the closed-loop system~\eqref{eq:closed_loop} if for all $\vx(t_{0}) \in \calS$, the solution $\vx(t) \in \calS$ for all $t \geq t_{0}$.
\end{definition}

\begin{theorem}[Nagumo's Theorem \cite{nagumo_uber_1942}]\label{thm:nagumo}
Consider the closed-loop system~\eqref{eq:closed_loop}\footnote{Although Nagumo's original theorem was stated for unforced systems, here it has been extended to the closed-loop controlled system.} and let $\calS \subset \Rn$ be a closed set. Then, $\calS$ is rendered forward invariant by $\pi$ if and only if
\begin{equation}
f(\vx) + g(\vx)\pi(\vx) \in \calT_\calS(\vx) \quad \forall \vx \in \calS.
\end{equation}
\end{theorem}

The safety guarantee provided by a CBF can be interpreted through the lens of the tangent cone as characterized by Nagumo's Theorem. For extra clarity in this paper, we first introduce the concept of a candidate CBF. From here, we denote the set $\calS$ as the set of safe states, which can often be encoded by a set of hard constraints provided to the system as requirements.
\begin{definition}[Candidate CBF]\label{def:candidate_cbf}
Let $\widetilde{h}: \StateSpace \to \RealSpace$ be a continuously differentiable function and let $\calC \coloneqq \{\vx \in \StateSpace \mid \widetilde{h}(\vx) \geq 0\}\subseteq\calS$. We call $\widetilde{h}$ a candidate CBF for $\calC$.
\end{definition}

 A \emph{candidate CBF} becomes a \emph{CBF} when it is paired with a properly chosen extended class $\calK_{\infty}$ function $\alpha$. More precisely, we have the following definition:
\begin{definition}[CBF \cite{ames_control_2019}] \label{def:cbf}
Let $h: \StateSpace \to \RealSpace$ be a continuously differentiable function and let $\calC \coloneqq \{\vx \in \StateSpace \mid h(\vx) \geq 0\}$. The function $h$ is a CBF on $\calC$ for System~\eqref{eq:ct_dynamics} if there exists an extended class-$\calK_{\infty}$ function $\alpha:\RealSpace\to\RealSpace$ such that
   \begin{equation}
       \sup_{\vu \in \ControlSpace} \left[\lieder_f h(\vx) + \lieder_g h(\vx)\vu\right] \geq -\alpha(h(\vx)) \quad \forall \vx \in \calC, 
   \label{eq:cbf-condition}
   \end{equation}
where $\lieder_f h(\vx)=\nabla h(\vx)^\top f(\vx)$ and $\lieder_g h(\vx)=\nabla h(\vx)^\top g(\vx)$ are the Lie derivatives of $h$ along $f$ and $g$, respectively.
\end{definition} 
Given a CBF~$h$ and its corresponding function $\alpha$, the set of all control inputs that render $\calC \subseteq \calS$ forward invariant, thus ensuring the system remains safe, is defined by the CBF constraint:
\begin{equation}
K_{\textup{cbf}}(\vx; \alpha) \coloneqq \left\{\vu \in \ControlSpace \mid \lieder_f h(\vx) + \lieder_g h(\vx)\vu \geq -\alpha(h(\vx)) \right\}.
\label{eq:k_cbf}
\end{equation}

We introduce two set-valued maps as follows. For each state $\vx \in \StateSpace$, define
\begin{equation}
F(\vx) \coloneqq \left\{ f(\vx) + g(\vx)\vu \mid \vu \in \ControlSpace \right\} \nonumber
\end{equation}
which describes all feasible state derivatives under any admissible control. Define also
\begin{equation}
G(\vx; \alpha) \coloneqq \left\{ f(\vx) + g(\vx)\vu \mid \vu \in K_{\textup{cbf}}(\vx; \alpha) \right\} \nonumber
\end{equation}
which captures the state derivatives that result from control inputs that satisfy the CBF constraint. Clearly, $G(\vx; \alpha)$ is a safety-restricted subset of $F(\vx)$.

The CBF constraint can be interpreted directly through \autoref{thm:nagumo}. For any feedback controller $\pi:\calC\to\ControlSpace$ satisfying $\pi(\vx)\in K_{\textup{cbf}}(\vx;\alpha)$, the closed-loop vector field satisfies
\[
L_fh(\vx)+L_gh(\vx)\pi(\vx)\ge -\alpha(h(\vx)).
\]
On $\partial\calC$, where $h(\vx)=0$, this condition gives $L_fh(\vx)+L_gh(\vx)\pi(\vx)\ge0$. If $\nabla h(\vx)\ne0$ on $\partial\calC$, this is the boundary tangency condition for $\calC$, and \autoref{thm:nagumo} gives forward invariance.

\subsection{Input Constrained CBFs \label{subsec:iccbf}}

The concept of CBFs has been generalized to HOCBFs in both continuous-time~\cite{xiao_control_2019} and discrete time~\cite{xiong_discretetime_2023} settings, which can be used for constraints of high relative degree. Input Constrained CBFs~(ICCBFs)~\cite{agrawal_safe_2021} further generalize HOCBFs, allowing for constraints of higher relative degrees and non-uniform relative degrees for control inputs.

The ICCBF construction addresses input constraints, i.e., $\ControlSpace \neq \Rm$, by restricting a candidate safe set to an inner set on which admissible control inputs satisfy the barrier constraint. Let $h$ be an $r$-times continuously differentiable candidate CBF. We assume that the recursion below produces continuously differentiable functions $b_i$ and that the class-$\calK_\infty$ functions are locally Lipschitz. Define $b_i:\StateSpace\to\RealSpace$, $i=0,\ldots,r-1$, as
{\setlength{\jot}{0.5ex} %
\begin{equation}\label{eq:iccbf-recursion}
\begin{split}
\raisetag{9.0ex} %
b_0(\vx;\bm{\alpha}) &\coloneqq h(\vx), \\
b_1(\vx;\bm{\alpha}) &\coloneqq \inf_{\vu \in \ControlSpace}\left[ \dot{b}_0(\vx, \vu;\bm{\alpha}) \right] + \alpha_1 ( b_0(\vx;\bm{\alpha}) ), \\
&\vdots \\
b_{r-1}(\vx;\bm{\alpha}) &\coloneqq \inf_{\vu \in \ControlSpace}\left[ \dot{b}_{r-2}(\vx, \vu;\bm{\alpha}) \right] + \alpha_{r-1} ( b_{r-2}(\vx;\bm{\alpha}) ),
\end{split}
\end{equation} }
and the constraint function $b_{r}:\StateSpace \times \ControlSpace \to \RealSpace$ as
\begin{equation} \label{eq:iccbf-constraint}
b_{r}(\vx, \vu;\bm{\alpha}) \coloneqq \dot{b}_{r-1}(\vx, \vu;\bm{\alpha}) + \alpha_{r}\left( b_{r-1}(\vx;\bm{\alpha}) \right) ,
\end{equation}
where $\bm{\alpha} =\{ \alpha_1,\ldots,\alpha_r \}$ is a set of extended class $\calK_{\infty}$ functions. We also define the following sets:
{\setlength{\jot}{0.5ex} %
\begin{equation}
\begin{split}
\raisetag{9.0ex} %
\calC_{0}(\bm{\alpha}) &\coloneqq \{\vx \in \StateSpace \mid b_{0}(\vx; \bm{\alpha}) \geq 0\} \subseteq \calS\\
\calC_{1}(\bm{\alpha}) &\coloneqq \{\vx \in \StateSpace \mid b_{1}(\vx; \bm{\alpha}) \geq 0\} \\
&\vdots \\
\calC_{r-1}(\bm{\alpha}) &\coloneqq \{\vx \in \StateSpace \mid b_{r-1}(\vx; \bm{\alpha}) \geq 0\} .
\end{split}
\label{eq:iccbf-sets}
\end{equation} }

\begin{definition}[Inner Safe Set~\cite{kim_learning_2025}]\label{def:inner-safe-set}
The set $\calC^{\star}(\bm{\alpha}) \subseteq \calS$,
\begin{equation} \label{eq:inner-safe-set}
\calC^{\star}(\bm{\alpha}) \coloneqq \bigcap_{i=0}^{r-1} \calC_i(\bm{\alpha}) ,
\end{equation}
where $\calC_{i}(\bm{\alpha})$ is defined in \eqref{eq:iccbf-sets}, is called an inner safe set for System~\eqref{eq:closed_loop}.
\end{definition}

\begin{definition}[ICCBF~\cite{agrawal_safe_2021}]\label{def:iccbf}
The function $h$ is an ICCBF\footnote{In \cite{agrawal_safe_2021}, $b_{r-1}$ is referred to as ICCBF; in this paper, for notational consistency with HOCBF~\cite{xiao_control_2019}, we designate $h$ as ICCBF, with all corresponding theoretical properties preserved.} on $\calC^{\star}(\bm{\alpha})$~\eqref{eq:inner-safe-set} for System~\eqref{eq:closed_loop} if there exist extended class $\calK_{\infty}$ functions $\bm{\alpha} = \{\alpha_1, \ldots, \alpha_r\}$ such that
\begin{equation}\label{eq:iccbf-condition}
\sup_{\vu \in \ControlSpace} \left[ b_r(\vx, \vu; \bm{\alpha}) \right] \geq 0 \quad \forall \vx \in \calC^{\star}(\bm{\alpha}) .
\end{equation}
\end{definition}
This condition allows the formulation of safety constraints for systems of relative degree at most $r$.%

For a fixed ICCBF parameter $\bm{\alpha}$, the set of admissible controls satisfying the ICCBF constraint is
\[
K_{\textup{iccbf}}(\vx;\bm{\alpha}) \coloneqq \{\vu\in\ControlSpace \mid b_r(\vx,\vu;\bm{\alpha})\ge0\}.
\]
The corresponding allowable state derivatives are
\[
G(\vx;\bm{\alpha}) \coloneqq \{f(\vx)+g(\vx)\vu \mid
\vu\in K_{\textup{iccbf}}(\vx;\bm{\alpha})\}.
\]
If $h$ is an ICCBF on $\calC^{\star}(\bm{\alpha})$ and the applied feedback satisfies $\pi(\vx)\in K_{\textup{iccbf}}(\vx;\bm{\alpha})$ there, then $\calC^{\star}(\bm{\alpha})\subseteq\calS$ is forward invariant~\cite{agrawal_safe_2021}.

Given a feedback controller $\vu_\textup{nom}(\vx)$ for System~\eqref{eq:ct_dynamics}, one can design a safety-critical controller using the following Quadratic Program~(QP) formulation that finds the minimally deviated safe control input $\vu^{\star}$ from the nominal control input:
\noindent\rule{\columnwidth}{0.4pt}
\textbf{CBF-QP:}
\begin{subequations}
\label{eq:cbf-qp-formulation}
\begin{align}
\vu^{\star} = \pi(\vx;\bm{\alpha}) & \coloneqq \argmin_{\vu \in \ControlSpace} \| \vu - \vu_\textup{nom}(\vx) \|^{2} \\
& \text{s.t.} \quad b_{r}(\vx,\vu;\bm{\alpha}) \geq 0 \label{eq:cbf-qp-constraint}
\end{align}
\end{subequations}
\noindent\rule{\columnwidth}{0.4pt}

The main text uses the continuous-time CBF-QP form for notational clarity. In sampled-data implementations, the same CBF parameter enters a discrete-time CBF or MPC-CBF constraint; the corresponding discrete-time ICCBF and MPC-CBF formulation is given in \discreteappendix.\footnote{Extended version: \url{https://arxiv.org/abs/2504.03038}.}

\section{OA-CBF PROBLEM FORMULATION\label{sec:problem}}
\subsection{Online Adaptation and Local Validation}

\begin{figure}[tbp]
\centering
\includegraphics[width=0.93 \linewidth]{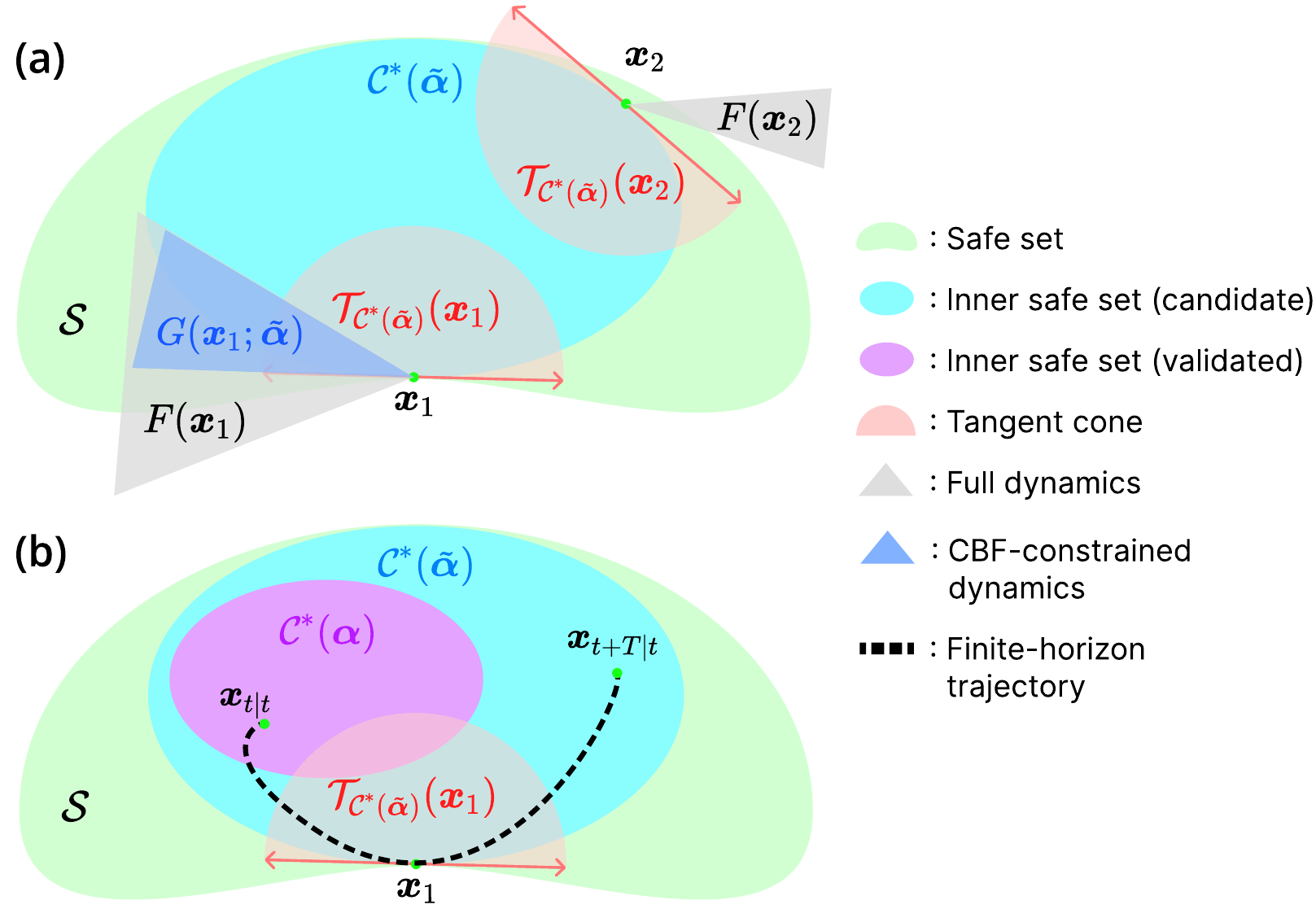}
\caption{
Conceptual illustration of the inner safe set and locally validated CBF parameter. (a) Candidate inner safe set~$\calC^{\star}(\widetilde{\bm{\alpha}})$ defined via an ICCBF with a CBF parameter~$\widetilde{\bm{\alpha}}$. At boundary point $\vx_1$, the CBF-constrained dynamics~$G(\vx_1; \widetilde{\bm{\alpha}})$ is non-empty, whereas at $\vx_2$ it is empty, indicating that the set cannot be rendered forward invariant. (b) With locally validated CBF parameters, the trajectory remains within the inner safe set shown in (a) over the finite-horizon, ensuring safety for that interval. By adapting the CBF parameters, the corresponding inner safe set is reshaped dynamically, alleviating conservatism by allowing the trajectory to extend beyond a fixed, globally verified inner safe set~$\calC^{\star}(\bm{\alpha})$.}
\label{fig:locally_validated}
\vspace{-5pt}
\end{figure}

We now formalize the CBF-parameter adaptation problem for input-constrained systems. In ICCBFs, the parameter $\bm{\alpha}=\{\alpha_1,\ldots,\alpha_r\}$ determines not only the decay rate in the safety constraint, but also the associated inner safe set $\calC^\star(\bm{\alpha})$. Therefore, adapting $\bm{\alpha}$ changes both the admissible-control set and the region that the controller can certify as safe. We refer to these extended class-$\calK_\infty$ functions collectively as the \textbf{\emph{CBF parameter}}. In the learned implementation, we use linear functions $\alpha_i(s)=\widetilde\alpha_i s$ with coefficients $\widetilde{\bm\alpha}\in\Theta\subseteq\RealSpace_{>0}^{r}$. Other parameterizations can also be used~\cite{xiao_feasibilityguided_2020}.

Let $\mathcal J=\{1,\ldots,M\}$ index the constraints and let $\calE=\{\mathbf o^j\}_{j\in\mathcal J}$ collect their parameters, with $\mathbf o^j\in\mathcal O\subseteq\RealSpace^{N_o}$. We consider differentiable constraint functions $h^j:\StateSpace\times\mathcal O\to\RealSpace$ that define
\begin{equation}
\calS=\bigcap_{j\in\mathcal J}\{\vx\in\StateSpace\mid h^j(\vx;\mathbf o^j)\ge0\}.
\label{eq:constraint-family}
\end{equation}
For example, $\mathbf o^j$ can contain an obstacle's position and size. The constraint parameters are fixed during a rollout, or their known evolution is included in the state. Applying the ICCBF construction to each constraint gives $b_i^j$ and $\calC_j^\star(\widetilde{\bm\alpha})$. For multiple constraints, we use $\calC^\star(\widetilde{\bm\alpha})\coloneqq\bigcap_{j\in\mathcal J}\calC_j^\star(\widetilde{\bm\alpha})$ and suppress the dependence on $\calE$ when it is clear from context.

\begin{remark}
Given \textbf{\emph{candidate ICCBFs}} \footnote{As with candidate CBFs in \autoref{def:candidate_cbf}, the ICCBF construction has been applied, but the ICCBF condition has not yet been verified.} $h^j$ and a specified parameter $\widetilde{\bm\alpha}$, we define, at each state $\vx\in\StateSpace$, the set of admissible control inputs satisfying all candidate ICCBF constraints as
\begin{equation}
K_{\textup{cand}}(\vx;\widetilde{\bm\alpha})
\coloneqq \left\{\vu\in\ControlSpace\mid
b_r^j(\vx,\vu;\widetilde{\bm\alpha})\geq 0,\ 
\forall j\in\mathcal J\right\}.
\label{eq:k_cand}
\end{equation}
The corresponding set of state derivatives induced by the control inputs in $K_{\textup{cand}}(\vx;\widetilde{\bm\alpha})$ is denoted by $G(\vx;\widetilde{\bm\alpha})$, as illustrated in Fig.~\ref{fig:locally_validated}(a).
For certain choices of $\widetilde{\bm\alpha}$, there may exist states $\vx\in\partial\calC^\star(\widetilde{\bm\alpha})$ at which $K_{\textup{cand}}(\vx;\widetilde{\bm\alpha})$ is empty.
At such states, no admissible control input satisfies all candidate ICCBF constraints.
\end{remark}

For comparison, define the set of \emph{globally valid} CBF parameters by
\begin{equation}
\calA_{\infty}\coloneqq\{\bm\alpha\mid K_{\textup{cand}}(\vx;\bm\alpha)\ne\varnothing,
\ \forall\vx\in\calC^\star(\bm\alpha)\}.
\label{eq:A_infty}
\end{equation}
A fixed parameter in $\calA_{\infty}$ allows an admissible CBF-based controller to preserve $\calC^\star(\bm\alpha)$ under the conditions in \autoref{subsec:iccbf}. Identifying such a parameter under input constraints can be difficult and may be overly conservative.

This paper introduces \textbf{\emph{Online Adaptive CBF~(OA-CBF)}} to address the following parameter adaptation problem:

\begin{problem}[Online Adaptive CBF]\label{prob:framework}
Given the closed-loop system~\eqref{eq:closed_loop} and a safe set $\calS$ (for which a valid ICCBF is unknown), design an online CBF parameter adaptation framework that yields an adapted parameter~$\widetilde{\bm{\alpha}}(t)$ such that a controller satisfying
\begin{equation}
\vu(t) \in K_{\textup{cand}}(\vx(t); \widetilde{\bm{\alpha}}(t)),
\end{equation}
where $K_{\textup{cand}}$ is given by \eqref{eq:k_cand}, yields a system trajectory satisfying $\vx(t) \in \calS$ for all $t \geq t_{0}$.
\end{problem}

The OA-CBF adaptation mechanism is applicable to both CBF-QP and sampled-data MPC-CBF controllers. In the CBF-QP formulation, OA-CBF replaces the fixed parameter~$\bm{\alpha}$ in~\eqref{eq:cbf-qp-constraint} with the adapted parameter~$\widetilde{\bm{\alpha}}(t)$. In the MPC-CBF formulation, the same replacement is made in the stage constraint; the discrete-time formulation is provided in \discreteappendix.

To address Problem~\ref{prob:framework}, we introduce a tractable, state-dependent approach based on \emph{locally validated CBF parameters} (see Fig.~\ref{fig:locally_validated}(b)). Rather than requiring a globally valid CBF parameter, we validate a candidate parameter over a finite prediction horizon $[t,t+T]$, where $T>0$. We denote by $\vx_{t+\tau|t}$ the predicted closed-loop trajectory of System~\eqref{eq:closed_loop} from the current state $\vx(t)$, with the candidate parameter $\widetilde{\bm{\alpha}}(t)$ held fixed over the prediction horizon:
\begin{equation}
\begin{aligned}
\frac{\mathrm d}{\mathrm d\tau}\vx_{t+\tau|t}
&=f(\vx_{t+\tau|t})+g(\vx_{t+\tau|t})\vu_{t+\tau|t},\\
\vu_{t+\tau|t}&=\pi(\vx_{t+\tau|t};\widetilde{\bm{\alpha}}(t)),
\qquad \vx_{t|t}=\vx(t),
\end{aligned}
\label{eq:rollout}
\end{equation}
for $\tau\in[0,T]$.

\begin{definition}[Locally Validated CBF Parameter\footnote{It was also termed a locally valid parameter in \cite{kim_learning_2025}.}]
\label{def:locally-validated}
Let $t \ge t_0$ and $T>0$. A CBF parameter $\widetilde{\bm{\alpha}}(t)$ is said to be \emph{locally validated} over the interval $[t,t+T]$ if:
\begin{enumerate}[leftmargin=5pt, labelindent=5pt, labelsep=0.0em, itemindent=*, align=left]
    \item the current state belongs to the inner safe set,
    \begin{equation}
    \vx_{t|t} \in \calC^{\star}(\widetilde{\bm{\alpha}}(t)),
    \label{eq:local-state-membership}
    \end{equation}

    \item for the rollout~\eqref{eq:rollout} $\vx_{t+\tau|t}$ generated under a control policy satisfying
    $\pi(\vx_{t+\tau|t}; \widetilde{\bm{\alpha}}(t)) \in K_{\textup{cand}}(\vx_{t+\tau|t};\widetilde{\bm{\alpha}}(t))$,
    the tangent-cone condition
    \begin{equation}
    f(\vx_{t+\tau|t}) + g(\vx_{t+\tau|t})\pi(\vx_{t+\tau|t}; \widetilde{\bm{\alpha}}(t))
    \in \calT_{\calC^{\star}(\widetilde{\bm{\alpha}}(t))}(\vx_{t+\tau|t})
    \label{eq:locally-validated}
    \end{equation}
    is satisfied for all $\tau \in [0,T]$.
\end{enumerate}
Similar to \autoref{thm:nagumo}, the condition in \eqref{eq:locally-validated} is only non-trivial on $\partial \calC^{\star}(\widetilde{\bm{\alpha}}(t))$.
\end{definition}

The following theorem states the safety property of OA-CBF under continued local validation.

\begin{theorem}[Safety under Continued Local Validation]
\label{thm:local-safety}
Consider System~\eqref{eq:closed_loop} and a fixed prediction horizon $T>0$. Let $t_0<t_1<\cdots$ be update times satisfying $t_k\to\infty$ as $k\to\infty$ and $t_{k+1}-t_k\le T$. Suppose that, at every update $t_k$, $k\ge0$, the selected parameter $\widetilde{\bm\alpha}_k$ is locally validated over $[t_k,t_k+T]$ according to \autoref{def:locally-validated}, with the rollout initialized at $\vx(t_k)$. If each $\widetilde{\bm\alpha}_k$ is held fixed on $[t_k,t_{k+1})$ and the control law $\vu(t)=\pi(\vx(t);\widetilde{\bm\alpha}_k)$ used in~\eqref{eq:rollout} is applied over this interval, then
\begin{equation}
\vx(t)\in\calS,\qquad \forall t\ge t_0.
\label{eq:local-safety}
\end{equation}
\end{theorem}

\begin{proof}
Since $t_k\to\infty$, any $t\ge t_0$ belongs to an update interval $[t_k,t_{k+1})$ for some $k\ge0$. Local validation at $t_k$ and the ICCBF inequalities~\eqref{eq:iccbf-recursion}--\eqref{eq:iccbf-constraint} ensure that the predicted closed-loop trajectory remains in $\calC^\star(\widetilde{\bm\alpha}_k)\subseteq\calS$ throughout $[t_k,t_k+T]$. Since the same control law is used for prediction and applied over
$[t_k,t_{k+1})\subseteq[t_k,t_k+T]$, it follows that
$\vx(t)\in\calS$ for all $t\in[t_k,t_{k+1}]$ and all $k\ge0$.
\end{proof}

Compared with choosing a fixed parameter in $\calA_\infty$, local validation allows the inner safe set to be reshaped around the current operating state, as illustrated in Fig.~\ref{fig:locally_validated}(b), thereby reducing the conservatism of relying on a single globally valid parameter. The remaining challenge is to efficiently identify locally validated CBF parameters at runtime:
\begin{problem}[Identifying Locally Validated CBF Parameters]
\label{prob:identifying}
At time $t$, given the current state $\vx_{t|t}$ and candidate ICCBFs, identify parameters from the locally validated set associated with the rollout and admissible policy in \autoref{def:locally-validated}:
\begin{equation}
\begin{aligned}
\calA(t)\coloneqq\bigl\{\widetilde{\bm\alpha}\in\Theta\mid {}&
\vx_{t|t}\in\calC^\star(\widetilde{\bm\alpha})\text{ and }\\
&\eqref{eq:locally-validated}\text{ holds over }[t,t+T]\bigr\}.
\end{aligned}
\label{eq:A_t}
\end{equation}
\end{problem}

\begin{figure}[!t]
\centering
\includegraphics[width=0.99\linewidth]{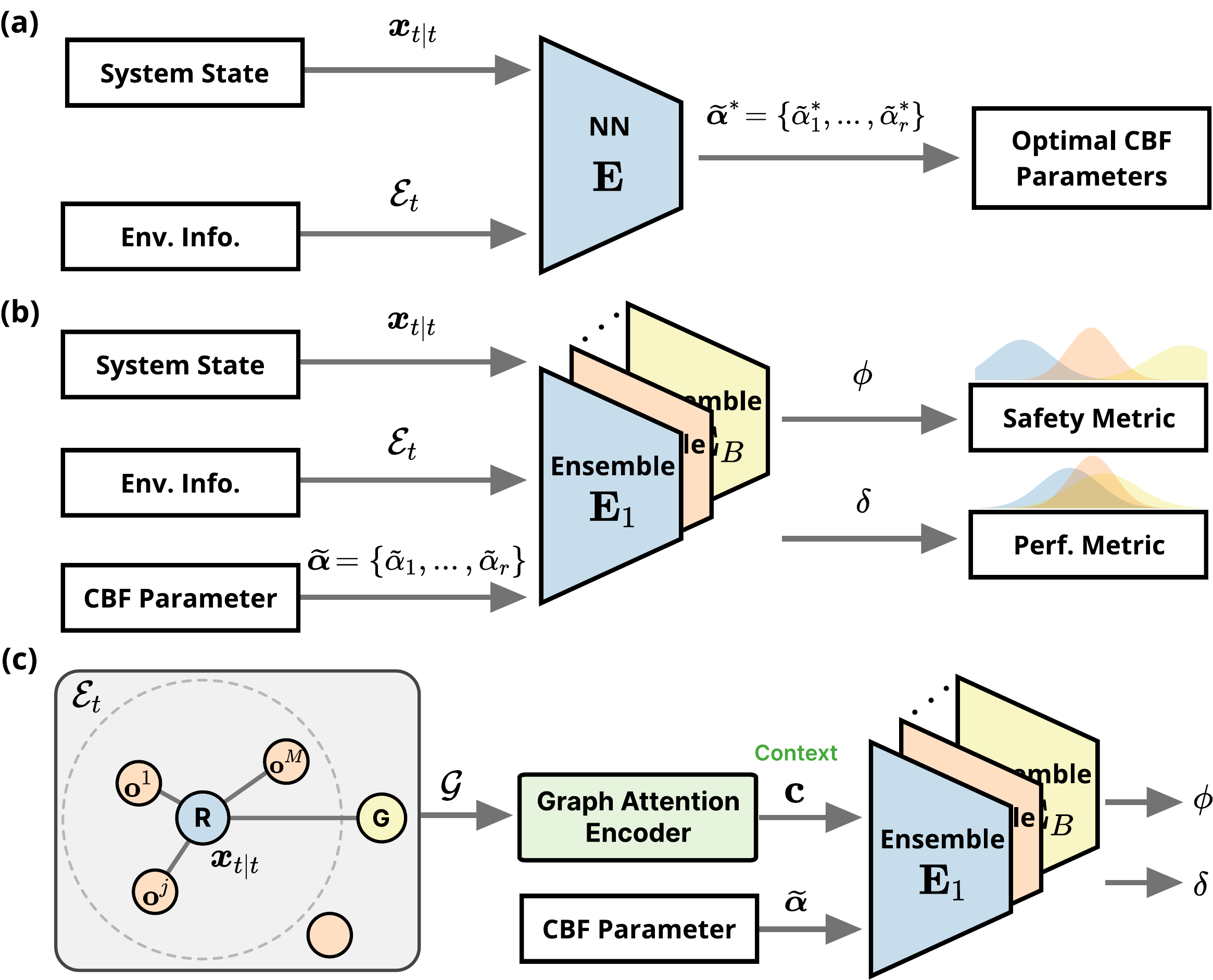}
\caption{
Neural network configurations for CBF parameter adaptation. (a) Direct-prediction approaches map the system state and environment information to a single adapted CBF parameter $\widetilde{\bm{\alpha}}^{\star}$.
(b) The proposed method instead takes a queried candidate CBF parameter $\widetilde{\bm{\alpha}}$ as input and predicts safety and performance metrics $(\phi,\delta)$ for that candidate.
(c) OA-CBF with graph attention first encodes the robot-goal-obstacle graph $\mathcal G$ into a context feature $\vc$, then evaluates the concatenated input $[\vc^\top,\widetilde{\bm{\alpha}}^\top]^\top$ using the PENN.
}
\label{fig:nn-config}
\label{fig:nn}
\vspace{-5pt}
\end{figure}

\subsection{Limitations of Direct Parameter Prediction \label{sec:basic}}

The core idea behind OA-CBF is to identify \emph{online} locally validated CBF parameters $\widetilde{\bm{\alpha}}(t_{k})$, $k \in \NaturalNumber \cup \{0\}$. A na\"{i}ve approach is to resort to an exhaustive search by forward propagating the closed-loop system trajectory over all possible parameters; however, this method is computationally intractable. An alternative is to collect data samples offline and use prediction models for online estimation. Neural networks can approximate the dependence of local validity on the state, constraint parameters, and CBF parameter while permitting efficient online evaluation.

One approach is to train a neural network to predict CBF parameters online and then directly adapt those parameters. Let $\widebar{X}$ denote an input representation of the system state $\vx_{t|t}$, the task goal, and the collection of constraint parameters $\calE_t=\{\mathbf o_t^j\}_{j\in\mathcal J}$ at time $t$ (e.g., obstacle positions and sizes in a collision avoidance task). The output $\widebar{Y} = \widetilde{\bm{\alpha}}^{\star}$ represents the predicted CBF parameter. Then, a deterministic neural network~$\vE$, parameterized by $\bm{\theta}$, defines the mapping (see Fig.~\ref{fig:nn}(a)):
\begin{equation} \label{eq:prior}
\widetilde{\bm{\alpha}}^{\star} = \vE(\widebar{X};\bm{\theta}) .
\end{equation}

The approaches in~\cite{ma_learning_2022, xiao_barriernet_2023, gao_online_2023, berducci_learning_2024, mohammad_soft_2025} fall into this category and have demonstrated promising results in reducing conservatism through online adaptation of CBF parameters using neural networks. However, a deterministic parameter prediction alone provides no assessment of the uncertainty in that prediction. In particular, the predicted CBF parameter $\widetilde{\bm{\alpha}}^{\star}$ in~\eqref{eq:prior} may not belong to the locally validated set $\calA(t)$.

\subsection{Identification under Predictive Uncertainty}
\label{subsec:uncertainty-problems}
For a learned approach to \autoref{prob:identifying}, two sources of uncertainty must be distinguished~\cite{kendall_what_2017,chua_deep_2018}. \emph{Epistemic uncertainty} arises from limited knowledge of the prediction model, particularly for queried states, constraint parameters, or CBF parameters poorly represented in the training data. Such queries may be out-of-distribution (OOD). \emph{Aleatoric uncertainty} describes variability in the predicted outcome that can remain for well-represented inputs, for example due to observation noise.

Let $X$ denote an input to the prediction model. We formulate the following two subproblems.

\begin{subproblem}[Identification under Epistemic Uncertainty]
\label{prob:epistemic}
Given a trained prediction model and held-out in-distribution (ID) trajectories, identify and reject potentially out-of-distribution (OOD) inputs while limiting the probability of incorrectly rejecting an ID trajectory, i.e.,
\begin{equation}
\Prob\left(\mathcal X_{\textup{test}}\subseteq
\mathcal X_{\textup{ID}}\right)
\geq 1-\delta_{\textup{cal}},
\label{eq:problem-epistemic}
\end{equation}
where $\mathcal X_{\textup{test}}$ denotes the set of inputs from an ID test trajectory, $\mathcal X_{\textup{ID}}$ denotes the set of inputs classified as ID, and $\delta_{\textup{cal}}\in(0,1)$ is a prescribed tolerance.
\end{subproblem}

\begin{subproblem}[Identification under Aleatoric Uncertainty]
\label{prob:aleatoric}
Among the candidates retained by \autoref{prob:epistemic}, identify CBF parameters predicted to satisfy the local-validation conditions in \autoref{def:locally-validated} with probability at least $1-\varepsilon$, i.e.,
\begin{equation}
\inf_{P\in\mathfrak P(X)}
\Prob^{P}\left(\widehat{\mathcal V}(X)\right)
\geq 1-\varepsilon,
\label{eq:problem-aleatoric}
\end{equation}
where $\widehat{\mathcal V}(X)$ denotes the event that the predicted outcome satisfies a sufficient condition for local validation, $\mathfrak P(X)$ denotes the family of predictive distributions associated with $X$, and $\varepsilon\in(0,1)$.
\end{subproblem}

\section{PROBABILISTIC VERIFICATION OF CBF PARAMETERS}
\label{sec:offline}
\label{sec:online}

To address \autoref{prob:epistemic} and \autoref{prob:aleatoric}, we propose a probabilistic verification method for queried CBF parameters, as illustrated in Fig.~\ref{fig:nn-config}(b)--(c). Instead of outputting a single CBF parameter $\widetilde{\bm{\alpha}}^{\star}$ as in~\eqref{eq:prior}, the model evaluates queried parameters by predicting distributions over a safety metric $\phi$ and a performance metric $\delta$. The key novelty is that this approach enables us to assess the local validity of candidate CBF parameters while explicitly accounting for uncertainty in neural network predictions. This shifts the burden of identifying locally validated parameters from direct prediction by the trained model to a verification process that accounts for predictive uncertainty. We first introduce the prediction model and define the safety and performance metrics, followed by the verification methods under epistemic and aleatoric uncertainty.

\subsection{Probabilistic Ensemble Neural Network Model}
\label{subsec:graph}
\label{subsec:penn}
The graph-attention encoder maps the system state, task goal, and constraint parameters $\calE$ to a context vector $\vc\in\RealSpace^{n_c}$. It aggregates a variable number of constraint features, allowing the same predictor to evaluate scenarios with different numbers of observed constraints without selecting a fixed number of neighbors. The candidate CBF parameter remains an explicit input to the model through
\begin{equation}
X\coloneqq[\vc^\top,\widetilde{\bm\alpha}^\top]^\top\in\RealSpace^{n_c+r}.
\label{eq:network_embedding}
\end{equation}
A shared context can therefore be used to evaluate a batch of candidates. In the navigation experiments, the graph represents the robot state, goal, and observed obstacles; its construction and attention mechanism are given in Appendix~\appendixref{app:graph}{\mbox  {A-A}}.

Given the embedding $X$, we employ a probabilistic ensemble neural network~(PENN), denoted by $\vF$, to predict the random output $\widehat Y=[\widehat\phi,\widehat\delta]^\top$. These quantities predict the safety and performance metrics defined in \autoref{subsec:safety-loss}. Each of its $B$ members specifies a Gaussian predictive law
\begin{equation}
\vE_b(X;\bm\theta_b)=\mathcal N\!\left(\vmu_{\bm\theta_b}(X),\VSigma_{\bm\theta_b}(X)\right),
\quad b=1,\ldots,B,
\label{eq:gaussian}
\end{equation}
with mean $\vmu_{\bm\theta_b}(X)\in\RealSpace^2$ and positive diagonal covariance $\VSigma_{\bm\theta_b}(X)\in\RealSpace^{2\times2}$. Uniform weighting of the ensemble $\mathfrak E=\{\vE_1,\ldots,\vE_B\}$ yields
\begin{equation}
\widehat Y\sim\vF(X;\bm\theta_{1:B})
=\frac{1}{B}\sum_{b=1}^{B}\vE_b(X;\bm\theta_b).
\label{eq:penn}
\end{equation}
The family $\mathfrak P(X)$ in \autoref{prob:aleatoric} consists of the safety-loss marginals of these member distributions.

Disagreement among members provides the epistemic score, while each member's output variance describes the aleatoric variability represented by the model~\cite{chua_deep_2018,kendall_what_2017}. The resulting Gaussian mixture~\eqref{eq:penn} supports both verification rules below. Network architecture and training details are provided in Appendices~\appendixref{app:penn}{\mbox  {A-B}} and~\appendixref{app:training}{\mbox  {A-D}}.

\subsection{Characterization of Safety and Performance Metrics \label{subsec:safety-loss}}
Let $\mathcal D_T\subseteq\StateSpace\times\Theta$ contain the state--parameter pairs for which the rollout~\eqref{eq:rollout} exists with admissible inputs over $[t,t+T]$. We use a safety-loss function and a performance-cost function
\begin{equation}
\phi_T:\mathcal D_T\to\RealSpace_{\ge0},
\qquad
\delta_T:\mathcal D_T\to[0,T],
\label{eq:metric-maps}
\end{equation}
whose definitions are given below in~\eqref{eq:finite-horizon-risk} and~\eqref{eq:performance-metric}. Their dependence on the constraint parameters, task goal, and control policy is suppressed. For complete admissible rollouts, the training labels are the evaluated values $Y=[\phi_T,\delta_T]^\top$; the random outputs $\widehat\phi$ and $\widehat\delta$ of~\eqref{eq:penn} predict these labels. The safety threshold $\widebar\phi$ is chosen so that $\phi_T(\vx_t;\widetilde{\bm\alpha})\le\widebar\phi$, together with~\eqref{eq:local-state-membership}, implies local validation; \autoref{prop:loss-certificate} establishes this implication.

For a fixed update state, define the \emph{safety margin} $\sigma$ as the infimum of the ICCBF constraint values along~\eqref{eq:rollout}:
\begin{equation}
\sigma(\widetilde{\bm{\alpha}}) \coloneqq
\inf_{\tau\in[0,T]}\min_{j\in\mathcal J}
b_r^j(\vx_{t+\tau|t},\vu_{t+\tau|t};\widetilde{\bm{\alpha}}).
\label{eq:safety_margin}
\end{equation}
Here, $\mathcal J$ is the constraint index set in~\eqref{eq:constraint-family}. Thus, $\sigma(\widetilde{\bm{\alpha}})\ge0$ enforces all candidate ICCBF inequalities throughout the prediction horizon. Together with~\eqref{eq:local-state-membership}, this is sufficient for local validation.

To construct a scalar learning target, define, for a fixed candidate $\widetilde{\bm\alpha}$, the function $\lambda^j:\StateSpace\times\ControlSpace\to\RealSpace_{>0}$ by
\begin{equation}
\lambda^j(\vx,\vu)\coloneqq
\lambda_1\exp\!\left(-\lambda_2 b_r^j(\vx,\vu;\widetilde{\bm\alpha})\right),
\label{eq:lambda_j}
\end{equation}
where $\lambda_1,\lambda_2>0$. Let $D_j:\StateSpace\to[1,\infty)$ be a weighting function for constraint $j$. It can encode application-specific proximity and direction information; the obstacle-avoidance construction is given in Appendix~\appendixref{app:safety-weighting}{\mbox  {A-C}}. Define the safety-loss density $\Phi:\StateSpace\times\ControlSpace\times\Theta\to\RealSpace_{\ge0}$ as
\begin{equation}
\Phi(\vx,\vu;\widetilde{\bm\alpha})
\coloneqq\max_{j\in\mathcal J}\max\!\left\{
\lambda^j(\vx,\vu),\,
\eta\frac{\lambda^j(\vx,\vu)}{D_j(\vx)}\right\},
\quad\eta\ge1.
\label{eq:safety_density}
\end{equation}
The first term ensures that every negative ICCBF margin yields a loss greater than $\lambda_1$. The second term permits additional weighting without weakening this implication. For an empty constraint set, we use the conventions $\Phi=0$ and $\sigma=+\infty$.

The finite-horizon safety loss is the function $\phi_T$ in~\eqref{eq:metric-maps}, evaluated along~\eqref{eq:rollout}:
\begin{equation}
\phi_T(\vx_t;\widetilde{\bm\alpha})\coloneqq
\sup_{\tau\in[0,T]}
\Phi(\vx_{t+\tau|t},\vu_{t+\tau|t};\widetilde{\bm\alpha}).
\label{eq:finite-horizon-risk}
\end{equation}

The following result states the certificate property of this loss.
\begin{proposition}[Sufficient Finite-Horizon Certificate]
\label{prop:loss-certificate}

Consider the admissible rollout~\eqref{eq:rollout} of System~\eqref{eq:ct_dynamics} under $\pi(\cdot;\widetilde{\bm\alpha})$ over $[t,t+T]$. Suppose that~\eqref{eq:local-state-membership} holds.

 Then
\begin{equation}
\phi_T(\vx_t;\widetilde{\bm{\alpha}})
\leq \widebar{\phi}\coloneqq\lambda_1
\quad\Longrightarrow\quad
\sigma(\widetilde{\bm{\alpha}})\geq0.
\label{eq:condition2}
\end{equation}
This condition therefore ensures that the candidate parameter $\widetilde{\bm{\alpha}}$ satisfies \autoref{def:locally-validated}.
\end{proposition}

\begin{proof}
The threshold condition yields the following implication chain, where the intermediate inequalities hold for %
every $j\in\mathcal J$ and every $\tau\in[0,T]$:
\begin{align*}
&\, \phi_T(\vx_t;\widetilde{\bm{\alpha}})
\leq \widebar{\phi}=\lambda_1
\\
\overset{\text{(a)}}{\Longrightarrow}
&\, \lambda^j(
\vx_{t+\tau|t},
\vu_{t+\tau|t})
\leq \lambda_1
\\
\overset{\text{(b)}}{\Longrightarrow}
&\, \lambda_1
\exp\!\left(
-\lambda_2 b_r^j(
\vx_{t+\tau|t},
\vu_{t+\tau|t};
\widetilde{\bm{\alpha}})
\right)
\leq \lambda_1
\\
\overset{\text{(c)}}{\Longrightarrow}
&\, b_r^j(
\vx_{t+\tau|t},
\vu_{t+\tau|t};
\widetilde{\bm{\alpha}})
\geq 0
\\
\overset{\text{(d)}}{\Longrightarrow}
&\, \sigma(\widetilde{\bm{\alpha}})\geq0.
\end{align*}
Here, (a) follows because, for %
every $j\in\mathcal J$ and every $\tau\in[0,T]$,
\begin{align*}
\lambda^j(\vx_{t+\tau|t},\vu_{t+\tau|t})
&\leq
\Phi(\vx_{t+\tau|t},\vu_{t+\tau|t};
\widetilde{\bm{\alpha}})
\\
&\leq
\phi_T(\vx_t;\widetilde{\bm{\alpha}}).
\end{align*}
The first inequality follows from the maximum in~\eqref{eq:safety_density}, and %
the second follows from~\eqref{eq:finite-horizon-risk}.
Step~(b) substitutes~\eqref{eq:lambda_j}.
Step~(c) uses $\lambda_1,\lambda_2>0$ and the strict monotonicity of the exponential function.
Finally, (d) follows from the definition of $\sigma$ in~\eqref{eq:safety_margin} %
as the infimum of the ICCBF margins over %
all constraints and the prediction horizon.

Together with~\eqref{eq:local-state-membership}, the recursive ICCBF inequalities~\eqref{eq:iccbf-recursion}--\eqref{eq:iccbf-constraint} preserve $b_{r-1}^j,\ldots,b_0^j\ge0$ over $[t,t+T]$, as in the proof of \autoref{thm:local-safety}. The predicted trajectory therefore remains in $\calC^\star(\widetilde{\bm\alpha})$~\eqref{eq:inner-safe-set} and satisfies~\eqref{eq:locally-validated}.

\end{proof}

For performance, let $v_{\mathrm{prog}}(\vx)\ge0$ denote the task-relevant translational speed and let $v_{\mathrm{min}}>0$ be the low-speed threshold. Let $T_{\mathrm g}\le T$ be the elapsed time from the rollout start to first arrival at the goal region, or $T$ if the goal is not reached within the horizon. We define the accumulated low-speed duration by
\begin{equation}
\delta_T(\vx_t;\widetilde{\bm\alpha})\coloneqq
\int_0^{T_{\mathrm g}}\!\mathbf 1_{\{v_{\mathrm{prog}}(\vx_{t+\tau|t})<v_{\mathrm{min}}\}}\,\mathrm d\tau.
\label{eq:performance-metric}
\end{equation}
Here, $\mathbf 1$ is the indicator function. Lower values of $\delta_T$ indicate less time spent in slow motion or deadlock. The PENN~\eqref{eq:penn} is trained on these metrics as described in Appendix~\appendixref{app:training}{\mbox  {A-D}}.

The trained model~\eqref{eq:penn} is used below to verify candidate parameters under the uncertainty requirements in \autoref{prob:epistemic} and \autoref{prob:aleatoric}.

\subsection{Verification Under Epistemic Uncertainty \label{subsec:eu}}

We evaluate epistemic uncertainty in the PENN predictions~\eqref{eq:penn} by quantifying disagreement among the ensemble members $\mathfrak E=\{\vE_1,\ldots,\vE_B\}$. With equal ensemble weights, the quadratic Jensen-R\'enyi divergence~(JRD) has the closed-form expression~\cite{wang_closedform_2009,kim_bridging_2023}
\begin{equation}
\begin{aligned}
D(X;\vE_{1:B})
&=-\log\left(\frac{1}{B^2}\sum_{b=1}^{B}\sum_{c=1}^{B}\mathfrak D_{bc}\right)\\
&\quad+\frac{1}{B}\sum_{b=1}^{B}\log\mathfrak D_{bb},
\end{aligned}
\label{eq:jrd}
\end{equation}
where
\[
\mathfrak D_{bc}
\coloneqq\frac{1}{|\bm{\mathfrak V}_{bc}|^{1/2}}
\exp\!\left(-\frac{1}{2}\bm\Delta_{bc}^{\top}
\bm{\mathfrak V}_{bc}^{-1}\bm\Delta_{bc}\right),
\]
$\bm{\mathfrak V}_{bc}\coloneqq\VSigma_{\bm\theta_b}(X)+\VSigma_{\bm\theta_c}(X)$,
$\bm\Delta_{bc}\coloneqq\vmu_{\bm\theta_b}(X)-\vmu_{\bm\theta_c}(X)$,
and $|\bm{\mathfrak V}_{bc}|$ denotes the determinant. In essence, the JRD~$D(X) \in \RealSpace$ measures how \emph{``spread out”} or inconsistent the ensemble predictions are. We write $D(X)$ when the trained ensemble is fixed.

We retain a candidate prediction only when its ensemble disagreement is below the calibrated threshold $D_{\textup{thr}}$. For an embedded input $X$~\eqref{eq:network_embedding}, the epistemic gate is
\begin{equation}
D\left(X;\vE_{1:B}\right) \le D_{\textup{thr}}.
\label{eq:epistemic_filter}
\end{equation}
If this condition is violated, the candidate prediction is discarded as epistemically unreliable. This gate screens queries using ensemble disagreement before the predicted risk is assessed. The calibration below specifies its ID recall under an exchangeability assumption.

\subsection{Conformal Calibration of the Epistemic Gate \label{subsec:cp}}

The threshold $D_{\textup{thr}}$ determines the tradeoff between rejecting ID predictions and retaining potentially unreliable queries. Existing approaches often choose such thresholds heuristically or tune them as hyperparameters~\cite{yang_generalized_2024, yang_rtdiff_2024, qing_bitrajdiff_2025, sun_socially_2025, kim_how_2025, kim_learning_2025}. To address \autoref{prob:epistemic}, we calibrate the gate using held-out ID trajectories without requiring representative OOD examples. Conformal prediction provides a finite-sample ID recall guarantee for this one-class screening problem~\cite{angelopoulos_conformal_2023, ding_classconditional_2023}.

Let $\calD_{\textup{calib}}=\{\uptau_i\}_{i=1}^{N_{\textup{traj}}}$ be a calibration dataset not used to train the PENN~\eqref{eq:penn}. Each $\uptau_i$ contains a state rollout~\eqref{eq:rollout} and the associated constraint parameters, task goal, and candidate CBF parameters used to construct the model inputs~\eqref{eq:network_embedding}. The trajectories are generated independently by sampling the initial system state, constraint parameters $\calE$, and candidate CBF parameters. States within a trajectory are temporally dependent, so complete trajectories are the units of exchangeability~\cite{ren_robots_2023}.

For each trajectory, we evaluate $D(X_{\tau,i})$ over $\tau\in[0,T]$ and define the trajectory score
\begin{equation}
Q_i\coloneqq\sup_{\tau\in[0,T]}D(X_{\tau,i}).
\label{eq:trajectory-score}
\end{equation}
The calibration and test scores use the same trained model and evaluation rule. As required in \autoref{prob:epistemic}, the test trajectory must be exchangeable with the calibration trajectories conditional on the trained model.

Let $Q_{(1)}\le\cdots\le Q_{(N_{\textup{traj}})}$ denote the nonconformity scores of the calibration trajectories, sorted in ascending order. For a prescribed $\delta_{\textup{cal}}\in(0,1)$, the conformal threshold is defined as
\begin{equation}
\begin{aligned}
k_{\textup{cal}}&\coloneqq
\left\lceil(N_{\textup{traj}}+1)(1-\delta_{\textup{cal}})\right\rceil,\\
D_{\textup{thr}}&\coloneqq Q_{(k_{\textup{cal}})},
\end{aligned}
\label{eq:d_thr_calibration}
\end{equation}
with the convention $Q_{(N_{\textup{traj}}+1)}=+\infty$.

\begin{lemma}[In-Distribution Recall Guarantee]
\label{lem:epistemic}
\label{thm:cp_guarantee}
For the trained prediction model~\eqref{eq:penn}, suppose that the calibration trajectories and an ID test trajectory $\uptau_{\textup{test}}$ are exchangeable and their scores are computed by~\eqref{eq:trajectory-score}. Then the threshold~\eqref{eq:d_thr_calibration} satisfies
\begin{equation}
\Prob\!\left(Q_{\textup{test}}\le D_{\textup{thr}}\right)
\ge1-\delta_{\textup{cal}},
\label{eq:cp_guarantee}
\end{equation}
where $Q_{\textup{test}}=\sup_{\tau\in[0,T]}D(X_{\tau,\textup{test}})$. Consequently, the gate~\eqref{eq:epistemic_filter} satisfies~\eqref{eq:problem-epistemic}, the trajectory-level ID recall requirement of \autoref{prob:epistemic}.
\end{lemma}

\begin{proof}
For $k_{\textup{cal}}\leq N_{\textup{traj}}$, the definition of the empirical quantile implies that $Q_{\textup{test}}\leq D_{\textup{thr}}$ whenever $Q_{\textup{test}}$ is among the $k_{\textup{cal}}$ smallest scores in $\{Q_1,\ldots,Q_{N_{\textup{traj}}},Q_{\textup{test}}\}$, including all scores tied at the cutoff. At least $k_{\textup{cal}}$ of these $N_{\textup{traj}}+1$ scores satisfy this criterion. Exchangeability therefore gives
\[
\Prob\!\left(Q_{\textup{test}}\leq D_{\textup{thr}}\right)
\geq
\frac{k_{\textup{cal}}}{N_{\textup{traj}}+1}
\geq 1-\delta_{\textup{cal}},
\]
which proves~\eqref{eq:cp_guarantee}. If $k_{\textup{cal}}=N_{\textup{traj}}+1$, then $D_{\textup{thr}}=+\infty$ and the claim holds directly.
\end{proof}

\autoref{lem:epistemic} ensures that the JRD does not exceed $D_{\textup{thr}}$ throughout an exchangeable ID trajectory with marginal probability at least $1-\delta_{\textup{cal}}$.

\subsection{Verification Under Aleatoric Uncertainty \label{subsec:au}}

For prediction results that pass the epistemic filter, we next examine the aleatoric uncertainty in the predicted safety metric. For notational simplicity, we focus only on the first output component of $Y=[\phi,\delta]^\top$, namely the safety metric~$\phi$, and denote its prediction by~$\widehat{\phi}$; the performance metric~$\delta$ will be used later in \autoref{subsec:select}. For the $b$-th ensemble member, let $ \mu^{\phi}_{\bm{\theta}_{b}}(X)$ and $ \sigma^{\phi}_{\bm{\theta}_{b}}(X)$ denote the mean and standard deviation of this scalar prediction. Then, under $\vE_b(X)$, the marginal prediction of~$\widehat{\phi}$ is
\begin{equation}
\widehat{\phi} \sim \mathcal{N} \left( \mu^{\phi}_{\bm{\theta}_{b}}(X), \bigl( \sigma^{\phi}_{\bm{\theta}_{b}}(X) \bigr)^2 \right).
\label{eq:phi-marginal}
\end{equation}

Recall from~\eqref{eq:condition2} that the deterministic condition $\phi_T(\vx_t;\widetilde{\bm{\alpha}})\leq\widebar{\phi}$, with $\widebar{\phi}=\lambda_1$, is sufficient for the finite-horizon ICCBF margin condition $\sigma(\widetilde{\bm{\alpha}})\geq0$ along the predicted trajectory. Since larger values of $\phi$ correspond to higher predicted risk, the relevant aleatoric uncertainty lies in the right tail of $\widehat{\phi}$. We therefore employ the right-tail Value at Risk~(VaR) and Conditional Value at Risk~(CVaR)~\cite{cai_evora_2024, rockafellar_optimization_2000}.

For a single ensemble member's prediction, the right-tail VaR with tail probability $\varepsilon\in(0,1)$ is defined as

\begin{equation}
\textnormal{VaR}_{\varepsilon}^{\vE_b(X)}\left(\widehat{\phi}\right)
\coloneqq
\inf_{\nu \in \RealSpace}
\left\{
\nu \mid \textnormal{Prob}^{\vE_b(X)} \left(\widehat{\phi} \le \nu\right) \ge 1-\varepsilon
\right\}. \nonumber
\label{eq:right-var}
\end{equation}
Thus, $\textnormal{VaR}_{\varepsilon}^{\vE_b(X)}(\widehat{\phi})$ is the $(1-\varepsilon)$-quantile of the predicted risk level.

Building on VaR, we define the corresponding right-tail CVaR as
\begin{equation}
\textnormal{CVaR}_{\varepsilon}^{\vE_b(X)} \left(\widehat{\phi}\right)
\coloneqq
\inf_{\eta \in \RealSpace}
\left\{
\eta + \frac{1}{\varepsilon}\mathbb{E}_{\vE_b(X)}
\left[(\widehat{\phi}-\eta)^+\right]
\right\}, \nonumber
\label{eq:right-cvar}
\end{equation}
where $(\cdot)^+ \coloneqq \max\{\cdot,0\}$. The right-tail CVaR represents the expected value of the worst $\varepsilon$-fraction of predicted risk levels, and therefore accounts for both tail probability and tail severity.

For the Gaussian prediction in \eqref{eq:phi-marginal}, both quantities admit closed-form expressions:
\begin{equation}
\textnormal{VaR}_{\varepsilon}^{\vE_b(X)} \left(\widehat{\phi}\right)
=
 \mu^{\phi}_{\bm{\theta}_{b}}(X)
+
 \sigma^{\phi}_{\bm{\theta}_{b}}(X)\,F_{\mathcal N}^{-1}(1-\varepsilon), \nonumber
\label{eq:right-var-gaussian}
\end{equation}
\begin{equation}
\textnormal{CVaR}_{\varepsilon}^{\vE_b(X)} \left(\widehat{\phi}\right)
=
 \mu^{\phi}_{\bm{\theta}_{b}}(X)
+
 \sigma^{\phi}_{\bm{\theta}_{b}}(X)\,
\frac{f_{\mathcal N} \left(F_{\mathcal N}^{-1}(1-\varepsilon)\right)}{\varepsilon}, \nonumber
\label{eq:right-cvar-gaussian}
\end{equation}
where $F_{\mathcal N}$ and $f_{\mathcal N}$ denote the cumulative distribution function and probability density function of the standard normal distribution, respectively.

To robustly account for aleatoric uncertainty, we treat the ensemble set~$\mathfrak{E}=\{\vE_{1},\ldots,\vE_{B}\}$ as an ambiguity set and impose the following distributionally robust CVaR constraint~\cite{rahimian_distributionally_2022}:
\begin{equation}
\sup_{\vE_b \in \mathfrak{E}}
\textnormal{CVaR}_{\varepsilon}^{\vE_b(X)}\left(\widehat{\phi}\right)
\le \widebar{\phi}.
\label{eq:distributionally-robust-cvar}
\end{equation}
This condition requires the most conservative ensemble member to predict a right-tail risk no larger than the sufficient threshold~$\widebar{\phi}$.

\begin{lemma}[Probabilistic Verification under Aleatoric Uncertainty]
\label{lem:aleatoric}
Consider a candidate CBF parameter $\widetilde{\bm{\alpha}}$ satisfying~\eqref{eq:local-state-membership}. If the distributionally robust CVaR condition~\eqref{eq:distributionally-robust-cvar} holds, then
\begin{equation}
\inf_{\vE_b\in\mathfrak{E}}
\textnormal{Prob}^{\vE_b(X)}\!
\left(\widehat{\phi}\leq\widebar{\phi}\right)
\geq1-\varepsilon.
\label{eq:prob-cvar-implication}
\end{equation}
Consequently, the CVaR condition is sufficient to satisfy~\eqref{eq:problem-aleatoric}, the predictive risk requirement of \autoref{prob:aleatoric}.
\end{lemma}

\begin{proof}
For every ensemble member $\vE_b\in\mathfrak{E}$, the relation between VaR and CVaR together with~\eqref{eq:distributionally-robust-cvar} gives
\[
\textnormal{VaR}_{\varepsilon}^{\vE_b(X)}(\widehat{\phi})
\leq
\textnormal{CVaR}_{\varepsilon}^{\vE_b(X)}(\widehat{\phi})
\leq \widebar{\phi}.
\]
By the definition of VaR, this implies
\[
\textnormal{Prob}^{\vE_b(X)}\!
\left(\widehat{\phi}\leq\widebar{\phi}\right)
\geq1-\varepsilon.
\]
Taking the infimum over $\vE_b\in\mathfrak{E}$ proves~\eqref{eq:prob-cvar-implication}.
\end{proof}

\autoref{lem:aleatoric} ensures that the predicted safety loss does not exceed $\widebar{\phi}$ with probability at least $1-\varepsilon$ under every ensemble member.

Let $\calA_{\textup{cand}}$ denote the set of sampled CBF parameters satisfying~\eqref{eq:local-state-membership}. Applying the epistemic gate and the distributionally robust CVaR condition yields the retained candidate set
\begin{equation}
\begin{aligned}
\widehat{\calA}
=\bigl\{\widetilde{\bm{\alpha}}\in\calA_{\textup{cand}}
\,\big|\,&
D(X;\vE_{1:B})\leq D_{\textup{thr}}
\textnormal{ and }\\
&\sup_{\vE_b\in\mathfrak{E}}
\mathrm{CVaR}_{\varepsilon}^{\vE_b(X)}
(\widehat{\phi})
\leq\widebar{\phi}\bigr\}.
\end{aligned}
\label{eq:verified-set}
\end{equation}
The epistemic gate and the CVaR condition provide calibrated ID recall and a bound on predicted risk, respectively, as established in \autoref{lem:epistemic} and \autoref{lem:aleatoric}.

\subsection{Performance-Guided CBF Parameter Selection \label{subsec:select}}

\begin{algorithm}[t]\footnotesize
\caption{Online Adaptive CBF (OA-CBF)}
\label{alg:online_adaptation}
\textbf{Given:} initial parameter $\widetilde{\bm{\alpha}}_0^{\star}$; CBF-based controller $\pi(\cdot;\widetilde{\bm{\alpha}})$; trained PENN $\vF(\cdot;\bm{\theta}_{1:B})$; calibrated epistemic threshold $D_{\textup{thr}}$. \\
\For{$t=1,\ldots,T_{\textup{max}}$}
{
    $\vx_t,\calE_t \leftarrow \texttt{Sensor()}$; \algovspace\\
    $\calA_{\textup{sample}} \leftarrow \texttt{SampleCandidates}(\widetilde{\bm{\alpha}}_{t-1}^{\star})$; \algovspace\\

    \tcp*[h]{--- Parallelizable block ---} \algovspace \algovspace \\

    $\calA_{\textup{cand}} \leftarrow \{ \widetilde{\bm{\alpha}} \in \calA_{\textup{sample}} \mid \vx_t \in \calC^{\star}(\widetilde{\bm{\alpha}}) \}$; \algovspace\\
    $\vX =\{X^{(k)} \}^{K}_{k=1} \leftarrow \texttt{CreateInputs}(\vx_t,\calE_t,\calA_{\textup{cand}})$; \algovspace\\
    $\mathfrak{E}, \widehat{\vY} =\{\widehat{Y}^{(k)} \}^{K}_{k=1} \leftarrow \vF(\vX;\bm{\theta}_{1:B})$; \algovspace\\
    $\widehat{\calA} \leftarrow \{\widetilde{\bm{\alpha}} \in \calA_{\textup{cand}} \mid D(X^{(k)};\vE_{1:B}) \le D_{\textup{thr}}$ \algovspace\\
    $\qquad \qquad \qquad \qquad \text{and } \sup_{\vE_b \in \mathfrak{E}} \textnormal{CVaR}_\varepsilon^{\vE_b(X^{(k)})}(\widehat{\phi}) \leq \widebar{\phi} \}$; \algovspace\\

    \tcp*[h]{--- End parallelizable block ---} \algovspace \algovspace \\

    \If{$\widehat{\calA}=\varnothing$}{
        $\widetilde{\bm{\alpha}}_t^\star\leftarrow\texttt{ApplyContingency}(\vx_t,\widetilde{\bm{\alpha}}_{t-1}^\star)$;\\
    }
    $\widetilde{\bm{\alpha}}_t^{\star} \leftarrow \argmin_{\widetilde{\bm{\alpha}} \in \widehat{\calA}} \hat{\delta}$; \algovspace\\
    $\vu_t \leftarrow \pi(\vx_t;\widetilde{\bm{\alpha}}_t^{\star})$; \algovspace\\
    $\texttt{ApplyControlInput}(\vu_t)$;
}
\end{algorithm}

After the two screening stages, $\widehat{\calA}$~\eqref{eq:verified-set} may contain multiple candidate parameters. When this set is nonempty, we select among them using the second output of the prediction model, the predicted performance metric $\hat{\delta}$, where smaller values indicate more efficient progress toward the goal:
\begin{equation}
\widetilde{\bm{\alpha}}_t^{\star}
=\argmin_{\widetilde{\bm{\alpha}}\in\widehat{\calA}}\hat{\delta}.
\label{eq:performance_selection}
\end{equation}
Because the optimization is restricted to $\widehat{\calA}$, performance-based selection does not weaken either screening condition. In practice, when multiple candidates attain nearly identical values of $\hat{\delta}$, one may additionally prefer the candidate with smaller deviation from the previously applied parameter to reduce chattering.

\subsection{Online CBF Parameter Adaptation \label{subsec:online-cbf}}

Integrating the verification and selection steps above, \autoref{alg:online_adaptation} implements the proposed OA-CBF framework. At each decision instant, a batch of candidate parameters is sampled around the previously applied parameter. Candidates satisfying~\eqref{eq:local-state-membership} are evaluated by the PENN, passed through the epistemic and aleatoric verification stages, and ranked using the predicted performance metric.

The following theorem combines the two verification results and states the guarantees of the parameter-selection procedure.

\begin{theorem}[Probabilistic Verification of OA-CBF Parameters]
\label{thm:main-result}
Consider \autoref{alg:online_adaptation} with the model~\eqref{eq:penn}. Under the exchangeability assumptions of \autoref{lem:epistemic}, the epistemic gate~\eqref{eq:epistemic_filter} satisfies the ID recall requirement~\eqref{eq:problem-epistemic} of \autoref{prob:epistemic}.

Whenever $\widehat{\calA}$~\eqref{eq:verified-set} is nonempty, the parameter selected by~\eqref{eq:performance_selection} satisfies~\eqref{eq:local-state-membership} and the predictive risk requirement~\eqref{eq:problem-aleatoric} of \autoref{prob:aleatoric}. It minimizes the predicted performance cost over $\widehat{\calA}$. Thus, the two verification steps satisfy \autoref{prob:epistemic} and \autoref{prob:aleatoric}.
\end{theorem}

\begin{proof}
The ID recall bound follows from \autoref{lem:epistemic}. Every parameter in $\widehat{\calA}$ satisfies~\eqref{eq:local-state-membership} and the CVaR condition~\eqref{eq:distributionally-robust-cvar}; \autoref{lem:aleatoric} therefore gives~\eqref{eq:problem-aleatoric}. Selection by~\eqref{eq:performance_selection} preserves these conditions and minimizes the predicted performance cost within the retained set.
\end{proof}

Since the probabilistic guarantees in \autoref{thm:main-result} are based on the predictions of the trained model rather than exact local validation, \autoref{thm:main-result} does not guarantee recursive feasibility of OA-CBF.

\begin{remark}[Contingency When the Parameter Set Is Empty]
\label{rem:empty-screening-set}
The set $\widehat{\calA}(t)$ may be empty if the training data do not adequately represent the current operating conditions, the sampled candidates are unsuitable, or the user-specified tolerances for epistemic uncertainty and aleatoric risk are overly restrictive. The selection conclusion of \autoref{thm:main-result} then does not apply. A practical contingency in \autoref{alg:online_adaptation} is to gradually reduce the previously applied CBF-parameter coefficients to encourage more conservative behavior. Parameter decay alone does not ensure local validity or feasibility under input constraints and carries no safety guarantee.

A certified backup controller can be used when the current state lies within its certified domain. For example, a backup-CBF construction requires an admissible backup policy, a terminal set invariant under that policy, and safe reachability of that set~\cite{chen_backup_2021}. Establishing these properties can be as challenging as constructing a valid CBF for a nonlinear system with input constraints. Backup design is therefore a related research problem beyond the scope of this paper. Such a backup is not assumed to be available, and the terminal neighborhoods used for offline data generation do not establish the required backup certificate.
\end{remark}

For the discrete-time realization in \discreteappendix, the corresponding interval-by-interval argument applies at the modeled sampling instants. Continuous-time safety between samples additionally requires the corresponding sampled-data conditions.

From a computational standpoint, the proposed CBF-parameter verification pipeline is fully batch-parallelizable. At each update step, we perform batch-wise PENN inference over the candidate set $\calA_{\textup{cand}}$ and across all ensemble members, which can be executed efficiently on GPU. Moreover, the JRD and the Gaussian CVaR both admit closed-form solutions, so the epistemic gate and aleatoric verification can be computed directly from the predicted moments in parallel, without Monte Carlo approximation. We emphasize that there is \textbf{\emph{no single iteration loop}} in the entire CBF parameter verification and adaptation procedure, ensuring efficient real-time implementation.

\section{RESULTS \label{sec:simulation}}

We evaluate the proposed OA-CBF method through simulation experiments on multiple nonlinear, input-constrained dynamical systems. The experiments are designed to answer four key questions. \textbf{(Q1)} Can OA-CBF serve as an effective solution to \autoref{prob:identifying}, i.e., can it identify locally validated CBF parameters online for nonlinear systems with input constraints? \textbf{(Q2)} Does the proposed online adaptation reduce the conservatism of fixed-parameter CBF controllers while preserving safety? \textbf{(Q3)} Does the proposed graph-attention-based environment encoding allow OA-CBF to generalize to obstacle configurations and obstacle counts that differ from those used during training? \textbf{(Q4)} Is OA-CBF applicable beyond classical signed-distance-function~(SDF)-type CBF candidates?

We first summarize the experimental setup and CBF candidates, then compare OA-CBF against fixed-parameter controllers, online-decay CBF baselines, and learning-based baselines, including BarrierNet~\cite{xiao_barriernet_2023} and an OA-CBF variant with fully connected environment encoding~\cite{kim_learning_2025,kim_how_2025}. The system equations and implementation settings are provided in Appendix~\appendixref{app:implementation}{A}.

\subsection{Experimental Setup}
\label{subsec:experimental_setup}

The safety-critical controller~$\pi$ is instantiated differently across baselines. The fixed-parameter baselines use the sampled-data MPC-CBF realization~\cite{zeng_safetycritical_2021} summarized in \discreteappendix, unless stated otherwise. OA-CBF w/ GAT and OA-CBF w/ FC use CBF-QP safety filters on all four randomized systems, while the VTOL case study uses MPC-CBF. The OD-CBF with QP baseline \cite{zeng_safetycritical_2021_decay} follows the standard CBF-QP formulation, while the BarrierNet baseline uses a CBF-QP as the last layer, as in the original algorithm~\cite{xiao_barriernet_2023}. In contrast, OD-CBF with MPC~\cite{zeng_enhancing_2021} uses the same MPC-CBF structure as the MPC-based baselines, but optimizes the decay parameters online inside the MPC problem. Thus, all methods are evaluated on the same dynamics, obstacle configurations, and input constraints, but they may differ in whether the underlying safety filter is realized as a CBF-QP or as a finite-horizon MPC-CBF. Details of the compared baselines are provided in \autoref{subsec:baselines}. %%

The controller update period is $\Delta t=0.05$~s, which is also the MPC discretization step. The admissible input set is enforced in each optimization problem, with system-specific limits listed in \autoref{tab:input_limits}. The training and testing environments are generated independently, and test configurations differ from those used for training. Each evaluated combination of method and system uses 1,000 cases. All methods on a case share its initial conditions, obstacle motion, route, and sensor-noise draws. The test layouts contain 24, 36, 48, or 60 obstacles.

OA-CBF uses a PENN with $B=4$ ensemble members, a CVaR upper-tail probability of $\varepsilon=0.01$, and a conformal coverage level of $1-\delta_{\mathrm{cal}}=0.95$. The epistemic threshold is calibrated from $200{,}000$ trajectories for each system. Network settings, data generation, and controller-cost choices are given in Appendix~\appendixref{app:experimental_setup}{\mbox  {A-E}}.

The reported metrics are goal-reaching rate (Success), infeasibility rate (Inf.), collision rate (Col.), and average reach time (Time [s]). Inf. counts trials terminated due to solver infeasibility, while Col. includes both collisions and workspace constraint violations. Success is the percentage of trials that successfully reach the goal within the maximum simulation time. The remaining unsuccessful trials are attributed to timeouts. Average reach time is computed only over successful goal-reaching trials; failed, infeasible, and collision trials are excluded from this average.

\begin{table*}[t]
\centering
\caption{Input constraints used in the simulation studies.}
\label{tab:input_limits}
{\renewcommand{\arraystretch}{1.15}
\setlength{\tabcolsep}{5pt}
\begin{tabular}{@{} p{0.2\textwidth} p{0.74\textwidth} @{}}
\toprule
\textbf{System} & \textbf{Input Constraints} \\
\midrule
Dynamic unicycle
& $a \in [-0.5,0.5]~\mathrm{m/s^2}$, $\omega \in [-0.5,0.5]~\mathrm{rad/s}$. \\

Quad2D
& $\vu = [u_{\textup{r}},u_{\textup{l}}]^\top$, with $u_{\textup{r}},u_{\textup{l}} \in [2.5,5.5]$. \\

Quad3D 
& $\vu=[u_1,u_2,u_3,u_4]^\top$, with $u_i\in[-2.0,5.0]$. \\

Kinematic bicycle
& $a \in [-5.0,5.0]~\mathrm{m/s^2}$ and $\beta \in [-\beta_{\max},\beta_{\max}]$, where $\beta_{\max}=\tan^{-1}\!\left(\frac{\ell_r}{\ell_f+\ell_r}\tan(32^\circ)\right) \approx 0.303~\mathrm{rad}$ ($17.4^\circ$).\\

VTOL quadplane
& $\delta_f,\delta_r,\delta_p \in [0,1]$, $\delta_e \in [-0.5,0.5]$. \\
\bottomrule
\end{tabular}}
\end{table*}

\begin{table*}[!t]
\centering
\caption{CBF candidates used in simulation. The entry $r$ denotes the relative degree of the continuous-time CBF candidate.}
\label{tab:cbf_relative_degree}
\resizebox{\textwidth}{!}{
\begin{tabular}{lccccc}
\toprule
CBF candidate 
& Dynamic Unicycle 
& Quad2D 
& Quad3D 
& Kinematic Bicycle 
& VTOL Quadplane \\
\midrule
Distance / planar obstacle CBF \appendixeqref{eq:distance_cbf_candidate}{56}
& $r=2$, ICCBF 
& $r=2$, ICCBF 
& $r=4$ in continuous time; RK4 discrete-time CBF used
& N/A
& $r=2$, ICCBF \\
DPCBF \appendixeqref{eq:dpcbf_candidate}{57}
& N/A
& N/A
& N/A
& $r=1$, CBF
& N/A \\
\bottomrule
\end{tabular}
}
\end{table*}

\subsection{Applied Dynamical Systems}
\label{subsec:applied_dynamical_systems}

We evaluate OA-CBF on five dynamical systems: a dynamic unicycle, a planar quadrotor (Quad2D), a three-dimensional quadrotor (Quad3D), a kinematic bicycle, and a VTOL quadplane (Fig.~\ref{fig:quadplane}). The first four systems are used for randomized benchmark comparisons, while the VTOL quadplane is used for a cruise-to-hover transition and landing case study. The Quad3D model follows~\cite{annaswamy_integration_2023}, the bicycle benchmark uses the Dynamic Parabolic CBF (DPCBF) formulation of~\cite{park_collision_2026}, and the VTOL aerodynamic model follows~\cite{willis_pitch_2022,beard_small_2012}. The state and input definitions, dynamics, and model assumptions are given in Appendix~\appendixref{app:system_models}{\mbox  {A-F}}.

\begin{figure}[!t]
\centering
\includegraphics[width=0.99\linewidth]{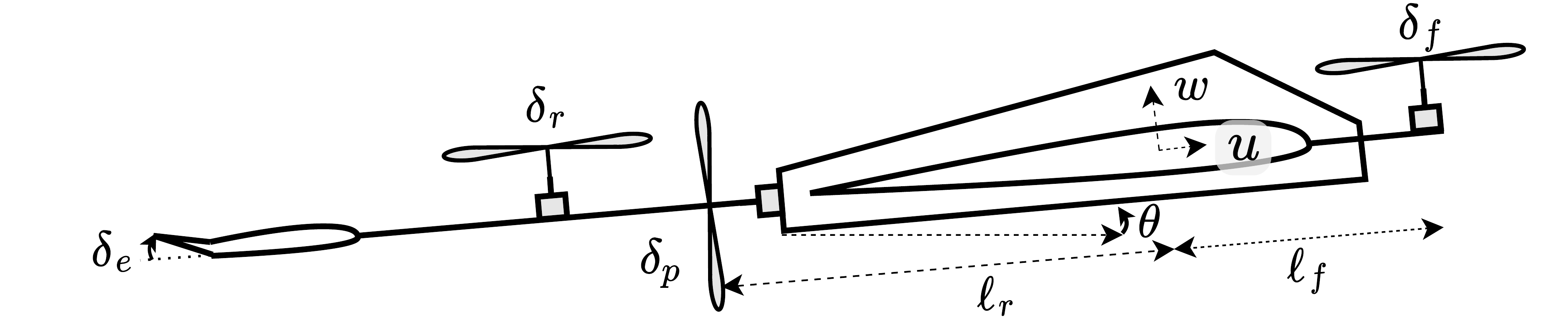}
\caption{
Illustration of a VTOL quadplane aircraft.}
\label{fig:quadplane}
\vspace{-5pt}
\end{figure}

\subsection{CBF Candidates and Relative-Degree Structure}
\label{subsec:cbf_candidates}

The dynamic unicycle, Quad2D, Quad3D, and VTOL quadplane use the distance-based obstacle-avoidance candidate in~\appendixeqref{eq:distance_cbf_candidate}{56}. The dynamic unicycle, Quad2D, and VTOL candidate constraints have relative degree two; their ICCBFs use two linear class-$\calK_\infty$ functions, $\alpha_i(s)=\widetilde{\alpha}_i s$, $i\in\{1,2\}$. For Quad3D, the planar avoidance constraint has relative degree four, and the MPC-based methods use the sampled-data RK4 CBF formulation of~\cite{taylor_safety_2022}.

The kinematic bicycle benchmark uses the relative-degree-one DPCBF~\cite{park_collision_2026} in~\appendixeqref{eq:dpcbf_candidate}{57}, which evaluates safety in relative-velocity space. OA-CBF is coupled with a CBF-QP for this benchmark. The candidate constructions are detailed in Appendix~\appendixref{app:cbf_candidates}{\mbox  {A-G}}, and their relative-degree structures are summarized in \autoref{tab:cbf_relative_degree}.

\begin{table*}[t]
\centering
\caption{Performance comparison across three system dynamics using 1,000 trials for each evaluated combination of method and system. Success, Inf., and Col. are percentages; Time is the mean simulated reach time over successful trials. The metrics are defined in \autoref{subsec:experimental_setup}. A dash indicates no successful trials. $^\dagger$ For Quad3D, the planar obstacle-avoidance constraint has relative degree four in continuous time; for the signed-clearance CBF-QP realization, the obstacle function is not four-times differentiable, so the corresponding continuous-time fourth-order CBF-QP constraint is not constructed.}
\label{tab:sdf_results}
{\small\setlength{\tabcolsep}{3pt}
\begin{tabular}{l|rrrr|rrrr|rrrr}
\toprule
 & \multicolumn{4}{c|}{\textbf{Dynamic Unicycle}} & \multicolumn{4}{c|}{\textbf{Quad2D}} & \multicolumn{4}{c}{\textbf{Quad3D}} \\
\cmidrule(r){2-5} \cmidrule(lr){6-9} \cmidrule(l){10-13}
\textbf{Method} & \textbf{Success} & \textbf{Inf.} & \textbf{Col.} & \textbf{Time [s]} & \textbf{Success} & \textbf{Inf.} & \textbf{Col.} & \textbf{Time [s]} & \textbf{Success} & \textbf{Inf.} & \textbf{Col.} & \textbf{Time [s]} \\
\midrule
Fixed Low & 22.7 & 20.1 & 42.9 & 23.90 & 0.0 & 96.5 & 2.6 & \textemdash & 14.7 & 81.2 & 0.0 & 68.08 \\
Fixed High & 16.7 & 82.5 & 0.3 & 10.70 & 1.6 & 37.5 & 1.3 & 66.21 & 46.4 & 46.7 & 0.0 & 46.69 \\
OD-CBF w/ QP~\cite{zeng_safetycritical_2021_decay} & 0.4 & 98.5 & 0.6 & 12.89 & 0.0 & 27.7 & 72.3 & \textemdash & \multicolumn{4}{c}{N/A (not $C^4$)$^\dagger$} \\
OD-CBF w/ MPC~\cite{zeng_enhancing_2021} & 26.1 & 3.0 & 59.4 & 26.78 & 0.0 & 5.2 & 45.1 & \textemdash & 93.1 & 0.1 & 0.0 & 57.88 \\
BarrierNet~\cite{xiao_barriernet_2023} & 15.1 & 6.7 & 65.9 & 19.50 & 0.0 & 95.6 & 4.4 & \textemdash & 0.0 & 0.0 & 100.0 & \textemdash \\
OA-CBF w/ FC \cite{kim_learning_2025, kim_how_2025} & 93.5 & 0.0 & 0.0 & 23.41 & 96.0 & 0.0 & 0.0 & 23.68 & 94.6 & 5.4 & 0.0 & 27.21 \\
\midrule
\rowcolor{gray!20} OA-CBF w/ GAT & \textbf{93.6} & 0.0 & 0.0 & 23.33 & \textbf{96.7} & 0.0 & 0.0 & 24.84 & \textbf{99.7} & 0.3 & 0.0 & 41.73 \\
\bottomrule
\end{tabular}}
\end{table*}

\subsection{Baselines}
\label{subsec:baselines}

We compare OA-CBF against fixed-parameter, optimal-decay, and learning-based baselines. \textbf{Fixed Low} uses conservative fixed CBF parameters, while \textbf{Fixed High} uses aggressive fixed parameters. \textbf{OD-CBF with QP} implements the optimal-decay CBF-QP method~\cite{zeng_safetycritical_2021_decay,ong_properties_2025}, and \textbf{OD-CBF with MPC} implements the optimal-decay MPC-CBF method of~\cite{zeng_enhancing_2021}. For relative-degree-two CBFs, both CBF parameters are optimized in optimal-decay baselines. \textbf{BarrierNet}~\cite{xiao_barriernet_2023} is used as a differentiable CBF-QP-layer baseline when the corresponding continuous-time CBF-QP constraint is well defined. We also compare against OA-CBF with fully connected closest-obstacle encoding, denoted \textbf{OA-CBF w/ FC}, to isolate the effect of replacing heuristic obstacle selection with the proposed graph-attention encoder. This baseline follows the implementation provided by \cite{kim_learning_2025,kim_how_2025}. Our proposed method is denoted \textbf{OA-CBF w/ GAT}.  Detailed baseline settings and implementation details are provided in Appendix~\appendixref{app:baseline_details}{\mbox  {A-I}}.

For OA-CBF, candidate parameters are sampled from a system-specific adaptation interval $\widetilde{\bm{\alpha}}\in[\widetilde{\alpha}_{\min},\widetilde{\alpha}_{\max}]$, and the initially applied parameter is set to the lower endpoint unless stated otherwise. The candidate is accepted only if it passes both the epistemic uncertainty gate and the distributionally robust CVaR constraint described in \autoref{sec:online}. Among the accepted candidates, the parameter with the lowest predicted performance cost is applied.

\begin{figure*}[t]
\centering
\includegraphics[width=0.99\textwidth]{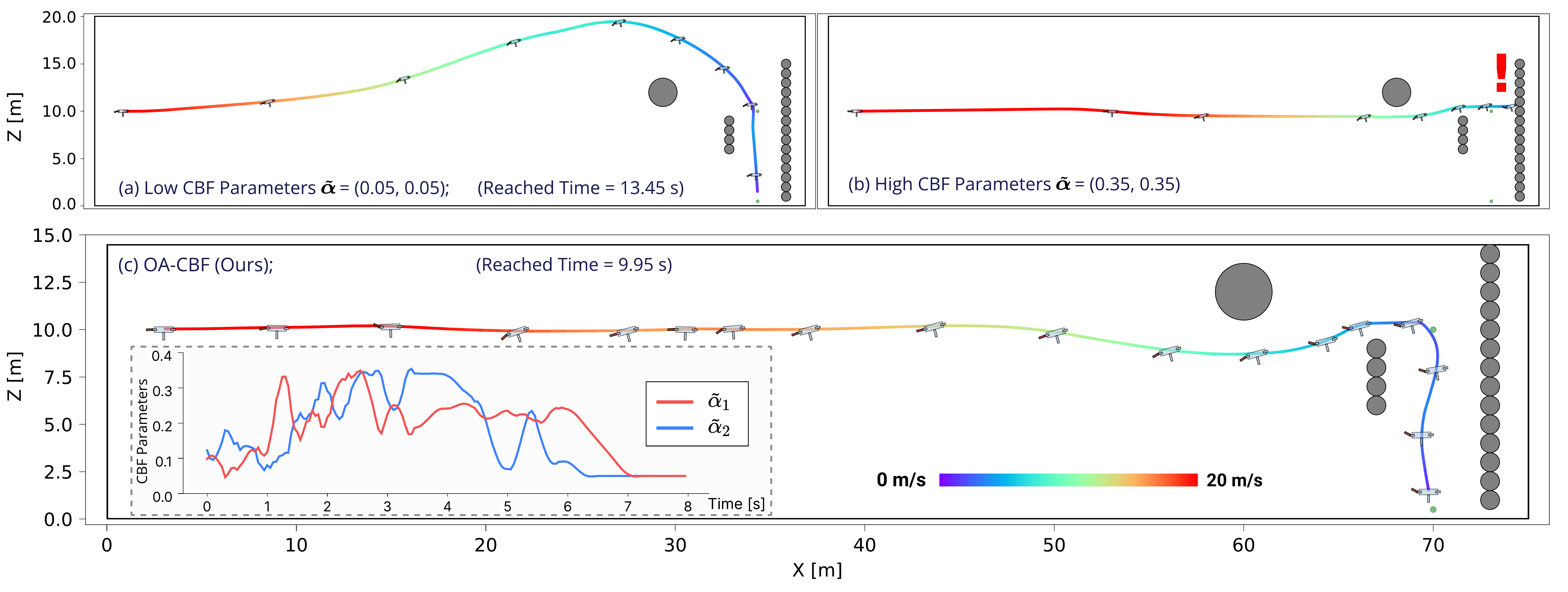}
\caption{
The figure shows the aircraft trajectory along with its speed profile, rigid-body pose, and elevator angle $\delta_e$, with obstacles depicted in gray. (a) With fixed low parameters, the trajectory shows a significant altitude detour as the ICCBF constraints force the aircraft to pitch up to decelerate in response to the obstacle. (b) With fixed high parameters, the controller becomes infeasible, ultimately resulting in a collision. (c) Our adaptive approach dynamically adjusts the CBF parameters based on the aircraft's speed and position. Initially, due to the high speed, it maintains low CBF parameters, prompting the elevator to pitch up and generate additional drag. As the aircraft slows down, the parameters increase to enhance performance, enabling the aircraft to fly safely beneath the obstacle and executing a smooth transition afterward. \textbf{Inset}: Evolution of the adapted CBF parameters over time.}
\label{fig:vtol_results}
\vspace{-5pt}
\end{figure*}

\subsection{Results and Discussion}
\label{subsec:results_discussion}

\subsubsection{\textbf{(Q1)-(Q2)}: Safety and Conservatism Across Nonlinear Systems}

\autoref{tab:sdf_results} compares the methods on the dynamic unicycle, Quad2D, and Quad3D benchmarks using distance-based obstacle-avoidance CBFs. Representative trajectories are shown in Fig.~\appendixref{fig:traj}{9} in Appendix~\appendixref{app:additional_results}{\mbox  {A-H}}, illustrating how each method behaves in the same obstacle-rich environments.

The fixed-low controller achieves goal-reaching rates of $22.7\%$, $0\%$, and $14.7\%$ on the dynamic unicycle, Quad2D, and Quad3D benchmarks, respectively. It has high infeasibility rates on Quad2D and Quad3D ($96.5\%$ and $81.2\%$), while $42.9\%$ of the unicycle trials end in collision. Thus, small fixed CBF parameters alone do not ensure reliable task completion in these benchmarks. %%

The fixed-high controller also has limited goal-reaching performance. Its goal-reaching rates are $16.7\%$, $1.6\%$, and $46.4\%$ on the dynamic unicycle, Quad2D, and Quad3D benchmarks, respectively, with high infeasibility rates of $82.5\%$, $37.5\%$, and $46.7\%$. These results show that a single aggressive CBF parameter cannot reliably handle all operating conditions under input constraints. %%

The optimal-decay baselines partially address this issue, but their performance is not uniform across systems. OD-CBF with QP has a $98.5\%$ infeasibility rate on the dynamic unicycle and fails to reach the goal in the Quad2D benchmark due to input constraints. OD-CBF with MPC achieves a $93.1\%$ goal-reaching rate on Quad3D, but reaches the goal in only $26.1\%$ of unicycle trials and has no successful Quad2D trials. %%

BarrierNet exhibits a different failure mode. Its average reach time is $19.50$~s on the dynamic unicycle, but this value is conditional on the $15.1\%$ of trials that successfully reach the goal. It has no successful trials on Quad2D or Quad3D: $95.6\%$ of Quad2D trials end in solver infeasibility, while all Quad3D trials end in collision. Thus, a short conditional reach time does not imply reliable task completion under the input constraints. %%

In contrast, OA-CBF with graph-attention encoding has no observed collisions or state-bound violations and achieves goal-reaching rates of $93.6\%$, $96.7\%$, and $99.7\%$ on the dynamic unicycle, Quad2D, and Quad3D benchmarks, respectively. Its average reach times are $23.33$~s, $24.84$~s, and $41.73$~s. OA-CBF w/ GAT gives the highest observed goal-reaching rate on all three benchmarks. %

These results support \textbf{(Q1)} and \textbf{(Q2)}: in these trials, OA-CBF improves goal-reaching performance over the fixed-parameter controllers with no observed collisions or workspace constraint violations.

\begin{figure*}[!t]
\centering
\includegraphics[width=0.99\textwidth]{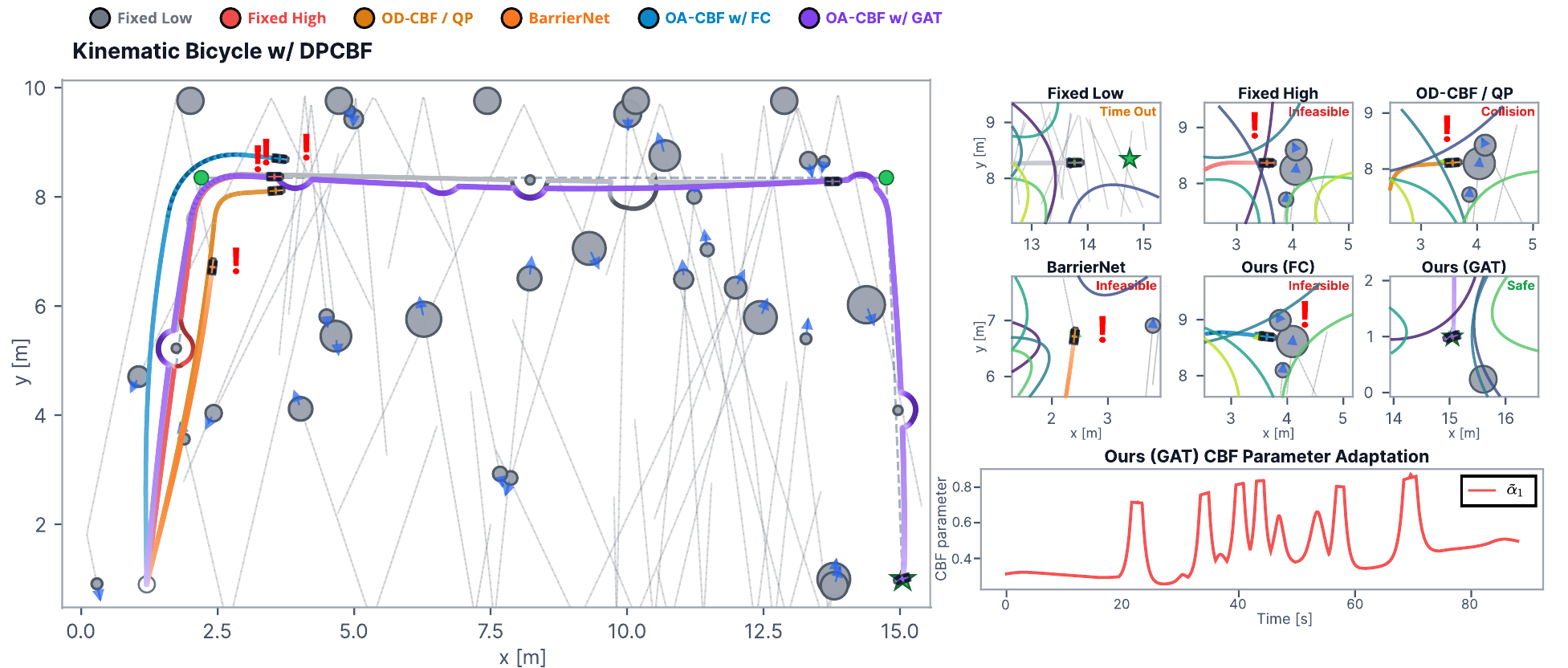}
\caption{
Representative trajectories for the kinematic bicycle benchmark with DPCBF in Table~\ref{tab:dpcbf_results}. The DPCBF safety boundary associated with each obstacle is also visualized on the right panels.}
\label{fig:dpcbf_traj}
\vspace{-5pt}
\end{figure*}

\subsubsection{VTOL Quadplane Case Study}

Fig.~\ref{fig:vtol_results} shows the VTOL quadplane case study. The quadplane starts in cruise mode with a high initial horizontal velocity, transitions toward hover mode near the intermediate waypoint, and then lands while avoiding obstacles. The initial position and horizontal velocity are $(p_{x,0},p_{z,0})=(2,10)~\text{m}, v_{x,0}=20~\text{m/s}$. The intermediate waypoint is $(70,10)$~m and the landing goal is $(70,0)$~m.

With fixed low CBF parameters, the controller remains safe but reacts conservatively to the obstacle. The aircraft pitches upward and takes a large altitude detour to reduce speed before passing the obstacle. With fixed high CBF parameters, the controller is less conservative initially but becomes infeasible near the obstacle, resulting in collision. OA-CBF adapts the CBF parameters based on the current speed and position of the aircraft. At high speed, it keeps the parameters small enough to preserve feasibility and induce early deceleration. Once the aircraft slows down, the parameters increase, reducing conservatism and allowing the aircraft to pass below the obstacle before transitioning smoothly to hover and landing.

In the simulation shown in Fig.~\ref{fig:vtol_results}, OA-CBF reaches the goal in approximately $9.95$~s, compared with approximately $13.45$~s for the fixed-low baseline, while the fixed-high baseline collides. %

This case study further supports \textbf{(Q1)} and \textbf{(Q2)} for a nonlinear system with aerodynamic effects and strict input constraints: the proposed parameter selection reduces the detour and transition time in the simulated VTOL maneuver.

\begin{table}[t]
\centering
\caption{Performance comparison on the kinematic bicycle model with DPCBF using 1,000 trials per method. Success, Inf., and Col. are percentages; Time is the mean simulated reach time over successful trials. A dash indicates no successful trials.}
\label{tab:dpcbf_results}
{\small\setlength{\tabcolsep}{3pt}
\begin{tabular}{l|rrrr}
\toprule
 & \multicolumn{4}{c}{\textbf{Kinematic Bicycle w/ DPCBF}} \\
\cmidrule(r){2-5}
\textbf{Method} & \textbf{Success} & \textbf{Inf.} & \textbf{Col.} & \textbf{Time [s]} \\
\midrule
Fixed Low & 3.6 & 3.3 & 93.0 & 45.49 \\
Fixed High & 0.0 & 100.0 & 0.0 & \textemdash \\
OD-CBF w/ QP~\cite{zeng_safetycritical_2021_decay} & 0.0 & 0.2 & 99.8 & \textemdash \\
BarrierNet~\cite{xiao_barriernet_2023} & 1.0 & 0.5 & 98.5 & 19.03 \\
OA-CBF w/ FC \cite{kim_learning_2025, kim_how_2025} & 93.2 & 0.4 & 0.0 & 23.25 \\
\midrule
\rowcolor{gray!20} OA-CBF w/ GAT & \textbf{93.4} & 0.6 & 0.0 & 27.28 \\
\bottomrule
\end{tabular}}
\end{table}

\subsubsection{\textbf{(Q3)}: Effect of Graph-Attention Encoding}

The comparison between OA-CBF w/ FC and OA-CBF w/ GAT serves as the encoder ablation in this section: both variants use the same offline labels, PENN ensemble size, uncertainty filtering, and online candidate-selection rule, but differ in whether the environment is represented by the closest-obstacle feature vector or by graph-attention aggregation.

In \autoref{tab:sdf_results}, OA-CBF w/ FC achieves goal-reaching rates of $93.5\%$, $96.0\%$, and $94.6\%$ on the dynamic unicycle, Quad2D, and Quad3D benchmarks, respectively. The corresponding rates for OA-CBF w/ GAT are $93.6\%$, $96.7\%$, and $99.7\%$. Both variants have no observed collisions. The largest difference occurs on Quad3D, where solver infeasibility decreases from $5.4\%$ to $0.3\%$.
These results support \textbf{(Q3)}: the graph-attention encoder handles varying obstacle counts without a fixed-size obstacle vector or heuristic obstacle selection, while its observed performance advantage over FC depends on the dynamics. %

\subsubsection{\textbf{(Q4)}: Adaptation with a Non-Distance-Based CBF}

\autoref{tab:dpcbf_results} evaluates OA-CBF on the kinematic bicycle model with the DPCBF candidate in \appendixeqref{eq:dpcbf_candidate}{57}. This experiment tests whether OA-CBF is tied to classical distance-based CBFs or can also be applied to CBF formulations whose safety semantics are defined in the relative-velocity space. Representative trajectories for the DPCBF benchmark are shown in Fig.~\ref{fig:dpcbf_traj}, illustrating the behavior of each method in the dynamic-obstacle setting.

The fixed-low controller reaches the goal in only $3.6\%$ of trials, with an average reach time of $45.49$~s, while $93.0\%$ of trials end in collision. The fixed-high controller has no successful trials: all terminate with solver infeasibility. These outcomes show that neither fixed-parameter choice reliably completes the DPCBF benchmark.

The QP-based adaptive baselines also have limited success in this benchmark. OD-CBF with QP has no successful trials, with $99.8\%$ ending in collision due to input constraints. BarrierNet succeeds in only $1.0\%$ of trials, with a conditional mean reach time of $19.03$~s and a collision rate of $98.5\%$. The short BarrierNet reach time therefore represents only its small successful subset.

OA-CBF improves goal-reaching performance without observed collisions. OA-CBF w/ FC and OA-CBF w/ GAT achieve goal-reaching rates of $93.2\%$ and $93.4\%$, with average reach times of $23.25$~s and $27.28$~s, respectively. For GAT, only $0.6\%$ of the trials were unsuccessful because no locally validated CBF parameter was identified. These results support \textbf{(Q4)}, demonstrating that the proposed uncertainty-aware parameter validation mechanism is not restricted to signed-distance-function-type CBF candidates.

Taken together, \autoref{tab:sdf_results}, \autoref{tab:dpcbf_results}, and Fig.~\ref{fig:vtol_results} support \textbf{(Q1)}-\textbf{(Q4)}. In these benchmarks, OA-CBF improves goal-reaching performance relative to fixed-parameter controllers, achieves the highest observed goal-reaching rate with graph-attention encoding on all four randomized systems, and applies beyond distance-based CBF candidates.

\section{CONCLUSION \label{sec:conclusion}}

In this paper, we studied online CBF-parameter adaptation for nonlinear robotic systems subject to input constraints. CBF parameters affect safety-filter feasibility, the shape of the inner safe set, and the conservatism of the resulting closed-loop behavior; fixed conservative parameters can slow task execution, whereas aggressive parameters can lead to collision or controller infeasibility. To address this issue, we proposed OA-CBF, which predicts a finite-horizon safety loss and a performance metric for queried parameters. A graph-attention-based probabilistic ensemble model represents the system state and constraint information. Conformal calibration bounds rejection of exchangeable in-distribution trajectories, while a distributionally robust CVaR test bounds predictive tail risk. OA-CBF selects the retained candidate with the lowest predicted performance cost. The framework can be coupled with both MPC-CBF and CBF-QP safety filters and is shown in simulation to reduce conservatism while maintaining low safety-failure rates across distance-based ICCBF and DPCBF benchmarks.

\FloatBarrier
\bibliographystyle{IEEEtran}
\typeout{}
\bibliography{references.bib}

% Appendices are included only in the arXiv version.
\ifarxivversion
\appendices
\section{Implementation and Additional Experimental Material}
\label{app:implementation}

\subsection{Graph-Attention Encoder for the Navigation Case Studies}
\label{app:graph}
The main formulation uses the constraint parameters $\calE$ in~\eqref{eq:constraint-family}. For the navigation experiments, $\mathbf o^j$ contains the position, velocity, and size of the $j$-th observed obstacle within sensing range $l_{\mathrm{range}}$. The number $M$ of observed obstacles can vary, so we encode their features using a graph-attention network~\cite{li_graph_2019,zhang_gcbf_2025}. This avoids the fixed-neighbor padding or truncation used in~\cite{kim_learning_2025,kim_how_2025}.

\subsubsection{Graph Construction}
Let $\calG=(\calV,\calE_{\mathrm{edge}})$ have $M+2$ nodes representing the robot state, the goal, and the obstacle parameters. For each node $i$, let $\vp_i$ and $\vv_i$ denote its position and translational velocity in the task space, and let $R_i$ denote its geometric radius. The robot features are obtained from $\vx_t$; obstacle features are obtained from $\mathbf o^j$; the goal is specified by its position $\vp_{\mathrm{goal}}$, with zero velocity and radius. A separate one-hot vector $\mathbf t_i$ encodes node type: $[1,0,0]^\top$ for the robot, $[0,1,0]^\top$ for an obstacle, and $[0,0,1]^\top$ for the goal.

An edge $e_{j\to i}\in\calE_{\mathrm{edge}}$ transmits features from a neighboring node $j$ to receiver $i$. Its feature vector is
\begin{equation}
\ve_{ji}\coloneqq[(\vp_j-\vp_i)^\top,(\vv_j-\vv_i)^\top,
\operatorname{dist}(\vp_j,\vp_i)]^\top,
\label{eq:graph-edge}
\end{equation}
where the surface-to-surface distance is
\begin{equation}
\operatorname{dist}(\vp_j,\vp_i)\coloneqq
\max(0,\|\vp_j-\vp_i\|-R_j-R_i).
\label{eq:graph-distance}
\end{equation}

\subsubsection{Aggregation via Graph Attention}
For each edge, a multilayer perceptron (MLP) encoder maps $\vz_{ji}\coloneqq[\mathbf t_i^\top,\mathbf t_j^\top,\ve_{ji}^\top]^\top$ to a latent feature $\vq_{ji}=\psi_{\mathrm{enc}}(\vz_{ji})$. A gate network computes normalized attention weights
\begin{equation}
w_{ji}=\frac{\exp(\psi_{\mathrm{gate}}(\vq_{ji}))}
{\sum_{k\in\calN_i}\exp(\psi_{\mathrm{gate}}(\vq_{ki}))},
\label{eq:graph-attention}
\end{equation}
where $\calN_i$ is the neighbor set of node $i$. Thus $w_{ji}\ge0$ and $\sum_{j\in\calN_i}w_{ji}=1$. The context vector is
\begin{equation}
\vc_i=\sum_{j\in\calN_i}w_{ji}\,\psi_{\mathrm{mes}}(\vq_{ji}),
\label{eq:graph-context}
\end{equation}
where $\psi_{\mathrm{mes}}$ is an MLP. We use the robot node's context $\vc$ and concatenate the queried CBF parameter according to~\eqref{eq:network_embedding}.

\subsection{Probabilistic Ensemble Architecture}
\label{app:penn}
The PENN consists of $B$ independently initialized neural networks with parameters $\bm\theta_b$. Each network predicts the mean and diagonal covariance in~\eqref{eq:gaussian}, and uniform weighting yields the mixture~\eqref{eq:penn}. The member-specific variances represent heteroscedastic aleatoric variability, while disagreement among member predictions is used for epistemic verification. Ensemble inference can be evaluated in parallel~\cite{kim_bridging_2023}. The architecture used in the experiments is specified in Appendix~\ref{app:experimental_setup}.

\subsection{Obstacle-Avoidance Weighting in the Safety Loss}
\label{app:safety-weighting}
For the navigation case studies, the generic weighting function in~\eqref{eq:safety_density} is chosen as
\begin{align}
\beta^j(\vx)&\coloneqq\beta_1\exp\!\left(\beta_2(1-\cos\Delta\theta_j)\right),
\label{eq:beta_j}\\
D_j(\vx)&\coloneqq1+\beta^j(\vx)\operatorname{dist}(\vp,\vp_{\mathrm{obs},j})^2,
\label{eq:D_j}
\end{align}
where $\beta_1,\beta_2>0$, $\vp$ is the robot position, $\vp_{\mathrm{obs},j}$ is the obstacle center, and $\operatorname{dist}$ is defined by~\eqref{eq:graph-distance}. The angle $\Delta\theta_j$ is measured between the robot's motion direction and the direction toward the obstacle center. When the speed is zero, the prescribed robot heading supplies the direction. At zero surface clearance, set $D_j=1$, so no angle is needed at coincident positions. This construction gives $D_j\ge1$ and places greater weight on nearby obstacles in the direction of travel. It is one application-specific choice; the certificate in \autoref{prop:loss-certificate} depends on the unweighted term $\lambda^j$ in~\eqref{eq:safety_density}, not on circular-obstacle geometry.

\begin{figure}[tbp]
\centering
\includegraphics[width=0.99\linewidth]{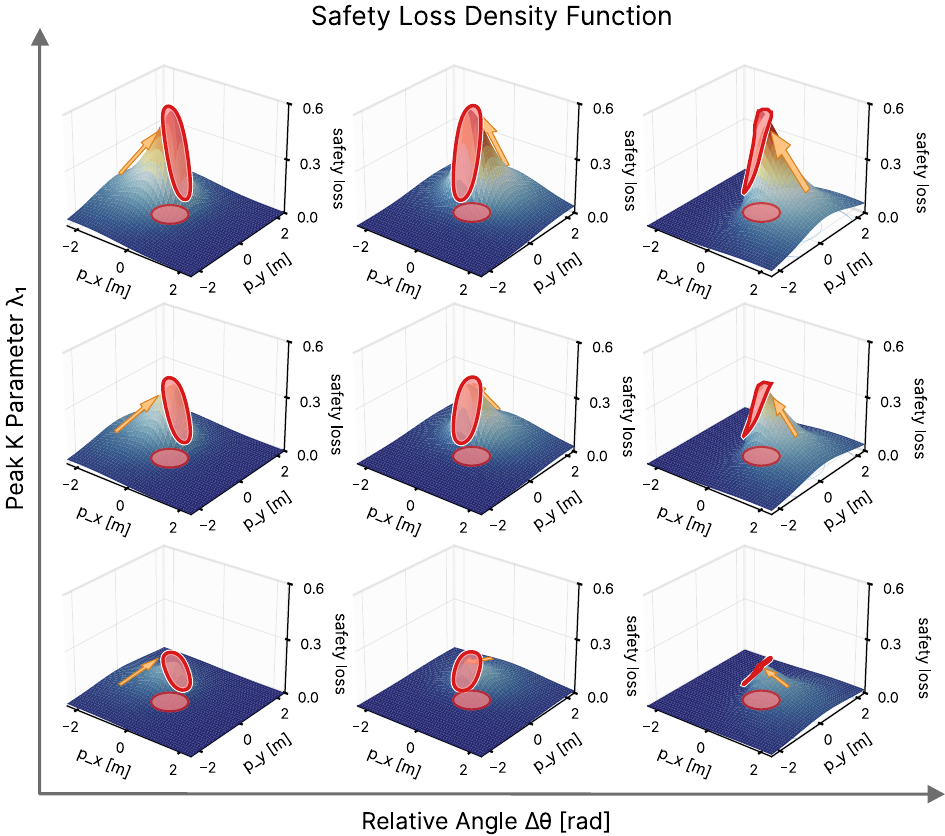}
\caption{
Safety-loss density $\Phi(\vx,\vu;\widetilde{\bm{\alpha}})$ in a canonical single-obstacle setting. A circular obstacle with radius $1\,\mathrm{m}$ is placed at the origin, $\vu$ is fixed to zero, and $\Phi(\vx,\vu;\widetilde{\bm{\alpha}})$ is evaluated over planar robot positions $\vp=[p_x,p_y]^\top$ while varying the peak parameter $\lambda_1$ and the relative angle $\Delta\theta$. The visualization shows how the loss concentrates near the obstacle and is reshaped by the heading-dependent dissipation term, assigning larger loss to configurations that are more directly aligned with the robot's approach direction. The red disk marks the obstacle 
collision region, and the orange arrow indicates an example robot motion direction 
used to interpret $\Delta\theta$.
}
\label{fig:safety-loss-density}
\end{figure}

\subsection{Training Data Generation and Loss Function}
\label{app:training}
\label{subsec:training}
To train the model~\eqref{eq:penn}, we generate a dataset $\calD=\{(\calG_k,\widetilde{\bm\alpha}_k,Y_k)\}_{k=1}^{N_{\mathrm{data}}}$ from simulated closed-loop rollouts. For each sample, we draw the initial state $\vx_0$, the number of constraints $M$, their parameters $\mathbf o^j$, and the candidate CBF parameter $\widetilde{\bm\alpha}_k$. The navigation graph is constructed as in Appendix~\ref{app:graph}, and the controller $\pi(\cdot;\widetilde{\bm\alpha}_k)$ generates the rollout~\eqref{eq:rollout}.

For each admissible rollout, we record $Y_k=[\phi_k,\delta_k]^\top$ using~\eqref{eq:finite-horizon-risk} and~\eqref{eq:performance-metric}. The horizon $T$ is chosen sufficiently long to reach the goal when feasible. In training scenarios with a terminal controlled-invariant neighborhood $\calX_f\subset\calS$ and an admissible stabilizing controller, a trajectory that reaches $\calX_f$ can be continued safely under that controller~\cite{agrawal_gatekeeper_2024,kim_backupbased_2026}. Such an offline construction does not establish recursive feasibility during deployment. If the rollout becomes infeasible or violates a constraint before time $T$, its safety label is assigned a large positive penalty. The perturbation convention used for the reported experiments is specified in Appendix~\ref{app:experimental_setup}.

The dataset is split into training and validation sets in a 7:3 ratio. Test cases are generated separately, with different constraint counts and configurations as described in \autoref{sec:simulation}. The calibration trajectories used in \autoref{subsec:cp} are also held out from model training.

Each ensemble member is trained by minimizing a Gaussian negative log-likelihood loss. Let $\mathcal I_{\mathrm{train}}$ index the training split and let $X^{(k)}$ be the embedding associated with sample $k$. The loss, up to an additive constant and positive scale, is
\begin{equation}
\begin{aligned}
\calL_{\mathrm{train}}(\bm\theta_b)
={}&\sum_{k\in\mathcal I_{\mathrm{train}}}\Bigl[
\log\det\VSigma_{\bm\theta_b}(X^{(k)})\\
&+(\vmu_{\bm\theta_b}(X^{(k)})-Y_k)^\top
\VSigma_{\bm\theta_b}(X^{(k)})^{-1}\\
&\hspace{2em}\cdot(\vmu_{\bm\theta_b}(X^{(k)})-Y_k)\Bigr].
\end{aligned}
\label{eq:train_loss}
\end{equation}

\begin{figure}[tbp]
\centering
\includegraphics[width=0.99\linewidth]{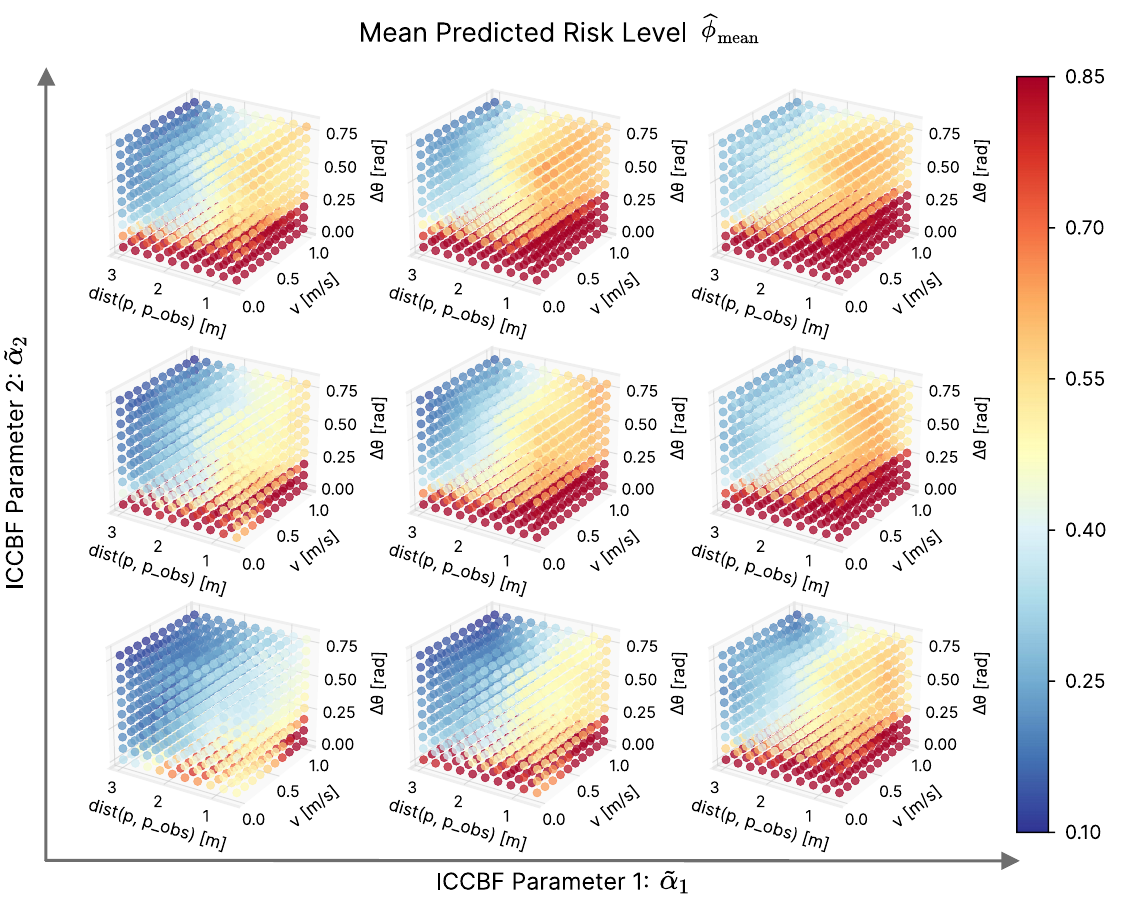}
\caption{
Sensitivity visualization of the ensemble-mean predicted safety-risk level 
$\widehat{\phi}_{\mathrm{mean}}$ from the trained GAT-PENN model used in the 
dynamic-unicycle experiments. Each subplot corresponds to one candidate parameter pair $(\widetilde{\alpha}_1,\widetilde{\alpha}_2)$, while the axes show obstacle distance, robot speed, and relative angle to the obstacle.
}
\label{fig:predicted-risk-level}
\end{figure}

\textbf{Visualization of the learned risk level predictor:} We additionally inspected one-obstacle slices of the trained GAT-PENN predictor, evaluated using the dynamic unicycle model, and observed that the ensemble-mean predicted risk increases for smaller obstacle distances, larger robot speeds, and headings more directly oriented toward the obstacle (see Fig.~\ref{fig:predicted-risk-level}). Online verification uses the full predictive distribution, including epistemic gating and the CVaR condition, rather than the ensemble mean alone.

\subsection{Experimental Implementation}
\label{app:experimental_setup}

OA-CBF w/ GAT and OA-CBF w/ FC use CBF-QP safety filters on the dynamic unicycle, Quad2D, Quad3D, and kinematic bicycle benchmarks. The same finite-horizon verification and parameter-selection procedure is used across these systems. The VTOL case study uses the MPC-CBF realization. All implementations use JAX.

For MPC-CBF controllers, the stage and terminal costs are quadratic tracking costs that drive the robot toward a goal state or waypoint. The cost matrices are tuned once for each dynamical system so that, in obstacle-free environments, the nominal MPC controller reaches the goal without excessive overshoot. For QP-based baselines, the nominal input is generated by the hand-tuned PD tracking controller and minimally modified through the CBF-QP safety layer. In all cases, the adapted or fixed CBF parameter enters only through the CBF/ICCBF safety constraint. %%

The offline dataset is generated following the procedure in \autoref{subsec:training}. For each sampled scenario, we sample an initial robot state, an obstacle environment, and a candidate CBF parameter~$\widetilde{\bm{\alpha}}$. We then roll out the corresponding, potentially unsafe, closed-loop system using this candidate parameter and record the resulting safety and performance labels. Infeasible rollouts and collision rollouts are assigned a large risk penalty so that they are separated from marginally safe cases. The training and testing environments are generated independently. In particular, the test environments include obstacle configurations that are not matched to the training distribution, which allows us to evaluate the scalability of the graph-attention encoding.

The PENN model used by OA-CBF consists of $B=4$ ensemble members. Each ensemble member is implemented as a five-layer stochastic multilayer perceptron with hidden widths $40$, $80$, $120$, and $40$ and ReLU activations. For the graph-attention variant, the graph encoder described in \autoref{subsec:graph} is trained jointly with the PENN prediction head and produces a $16$-dimensional robot-context embedding that is concatenated with the candidate CBF parameter. To model observation and dynamics uncertainty during offline data generation, we add $3\%$ i.i.d. Gaussian noise to the neural-network input features, excluding the candidate CBF parameters. The CVaR upper-tail probability is set to $\varepsilon=0.01$. The epistemic uncertainty threshold $D_{\mathrm{thr}}$ is calibrated at coverage level $1-\delta_{\mathrm{cal}}=0.95$ using graph-level Jensen-Renyi divergence scores computed from $200{,}000$ calibration trajectories
for each system.

The randomized tests use 250 cases at each sensor-noise multiplier $0$, $0.5$, $1$, and $2$, giving 1,000 cases per system. Initial conditions, obstacle motion, routes, and noise realizations are shared across methods, and the four test layouts contain 24, 36, 48, or 60 obstacles. The final comparison contains 26,000 controller evaluations.

\subsection{System Models}
\label{app:system_models}

We evaluate OA-CBF on five dynamical systems: a dynamic unicycle, a planar quadrotor (Quad2D), a three-dimensional quadrotor (Quad3D), a kinematic bicycle model, and a VTOL quadplane. The first four systems are used for randomized benchmark comparisons, while the VTOL quadplane is used as a high-dimensional nonlinear case study.

\subsubsection{Dynamic Unicycle}

The dynamic unicycle state is $\vx = [p_x,p_y,\theta,v]^\top$ where $(p_x,p_y)$ is the planar position, $\theta$ is the heading angle, and $v$ is the forward speed. The input is $\vu = [a,\omega]^\top$, where $a$ is the longitudinal acceleration and $\omega$ is the angular velocity. The dynamics are
\begin{equation}
\begin{alignedat}{2}
    \dot p_x &= v\cos\theta,
    &\qquad
    \dot p_y &= v\sin\theta, \\
    \dot\theta &= \omega,
    &\qquad
    \dot v &= a.
\end{alignedat}
\label{eq:dynamic_unicycle}
\end{equation}

\subsubsection{Quad2D}
The Quad2D benchmark uses a planar thrust-controlled rotor model with state $\vx=[p_x,p_z,\theta,v_x,v_z,\dot\theta]^\top$, where $(p_x,p_z)$ is the planar position, $\theta$ is the pitch angle, and $(v_x,v_z)$ is the planar velocity. The input is $\vu=[u_{\textup{r}},u_{\textup{l}}]^\top$, where $u_{\textup{r}}$ and $u_{\textup{l}}$ are the right and left rotor thrusts. The implemented dynamics are
\[
\begin{aligned}
\dot p_x &= v_x, &
\dot p_z &= v_z, \\
\dot v_x &= -\frac{\sin\theta}{m}(u_{\textup{r}}+u_{\textup{l}}), &
\dot v_z &= -g+\frac{\cos\theta}{m}(u_{\textup{r}}+u_{\textup{l}}), \\
\ddot\theta &= \frac{r}{I}(u_{\textup{r}}-u_{\textup{l}}).
\end{aligned}
\]

\subsubsection{Quad3D}
The Quad3D benchmark uses a linearized 6-DOF quadrotor model with 12 states and 4 motor-force inputs~\cite{annaswamy_integration_2023}. The state is $\vx =
[\,p_x,p_y,p_z,\phi,\theta,\psi,v_x,v_y,v_z,\dot\phi,\dot\theta,\dot\psi\,]^\top$, where $(p_x,p_y,p_z)$ is the position, $(\phi,\theta,\psi)$ are the roll, pitch, and yaw angles, and $(v_x,v_y,v_z)$ is the translational velocity. The control input is $\vu=[u_1,u_2,u_3,u_4]^\top$, where $u_i$ denotes the $i$-th motor-force deviation from the hover equilibrium. These inputs are mapped to the incremental vertical force and body torques by
\[
\begin{bmatrix}
F\\ \tau_\phi\\ \tau_\theta\\ \tau_\psi
\end{bmatrix}
=
\begin{bmatrix}
1 & 1 & 1 & 1\\
L & 0 & -L & 0\\
0 & L & 0 & -L\\
\nu & -\nu & \nu & -\nu
\end{bmatrix}
\begin{bmatrix}
u_1\\u_2\\u_3\\u_4
\end{bmatrix}.
\]
Here $F$ is the incremental vertical force about hover, $\tau_\theta$, $\tau_\phi$, and $\tau_\psi$ are the pitch, roll, and yaw torques, respectively, $L$ is the quadrotor arm length, and $\nu$ is the rotor yaw-torque coefficient.

The dynamics are
\begin{equation}
\begin{aligned}
    \dot p_x &= v_x, &
    \dot p_y &= v_y, &
    \dot p_z &= v_z,\\
    \dot v_x &= g\theta, &
    \dot v_y &= -g\phi, &
    \dot v_z &= \frac{1}{m}F,\\
    \ddot\phi &= \frac{1}{I_x}\tau_\phi, &
    \ddot\theta &= \frac{1}{I_y}\tau_\theta, &
    \ddot\psi &= \frac{1}{I_z}\tau_\psi .
\end{aligned}
\label{eq:quad3d}
\end{equation}
In these equations, $m$ is the quadrotor mass, $g$ is the gravitational acceleration, and $I_x$, $I_y$, and $I_z$ are the moments of inertia about the body-fixed roll, pitch, and yaw axes, respectively.

\subsubsection{Kinematic Bicycle}
The kinematic bicycle model is used to evaluate OA-CBF with a non-SDF-type CBF candidate. The state is $\vx=[p_x,p_y,\theta,v]^\top$, where $(p_x,p_y)$ denotes the center-of-mass position, $\theta$ is the heading, and $v$ is the forward speed. The physical inputs are the longitudinal acceleration $a$ and the front-wheel steering angle $\delta$. Let $\ell_f$ and $\ell_r$ denote the distances from the center of mass to the front and rear axles. The slip angle is $\beta = \tan^{-1}\!\left( \frac{\ell_r}{\ell_f+\ell_r}\tan\delta \right)$.
Using the small-slip approximation, the model can be written in control-affine form with input $\vu=[a,\beta]^\top$:
\begin{equation}
\underbrace{\begin{bmatrix}
\dot p_x\\
\dot p_y\\
\dot\theta\\
\dot v
\end{bmatrix}}_{\dot{\vx}}
=
\underbrace{\begin{bmatrix}
v\cos\theta\\
v\sin\theta\\
0\\
0
\end{bmatrix}}_{f(\vx)}
+
\underbrace{\begin{bmatrix}
0 & -v\sin\theta\\
0 & \phantom{-}v\cos\theta\\
0 & v/\ell_r\\
1 & 0
\end{bmatrix}}_{g(\vx)}
\underbrace{\begin{bmatrix}
a\\
\beta
\end{bmatrix}}_{\vu}.
\label{eq:kinematic_bicycle}
\end{equation}
The dynamic obstacles in this benchmark follow constant-velocity unicycle dynamics. For obstacle $j$, $\vx_{\mathrm{obs}}^j = [p_{\mathrm{obs},x}^j, p_{\mathrm{obs},y}^j, \theta_{\mathrm{obs}}^j, v_{\mathrm{obs}}^j]^\top$,
with $\dot p_{\mathrm{obs},x}^j = v_{\mathrm{obs}}^j\cos\theta_{\mathrm{obs}}^j$, $\dot p_{\mathrm{obs},y}^j = v_{\mathrm{obs}}^j\sin\theta_{\mathrm{obs}}^j$. Obstacle states are assumed to be observable.

\subsubsection{VTOL Quadplane}
We also consider a two-dimensional VTOL quadplane as a case study (see Fig.~\ref{fig:quadplane}). This platform has nonlinear, high-dimensional dynamics and strict input limits, making CBF-parameter tuning particularly challenging during transitions between cruise and hover flight.

The state is $\vx=[\vp^\top,\vv^\top,\theta,q]^\top$, where $\vp=[p_x,p_z]^\top\in\RealSpace^2$ is the inertial-frame position, $\vv\in\RealSpace^2$ is the inertial-frame velocity, $\theta$ is the pitch angle, and $q=\dot\theta$ is the pitch rate. The control input is $\vu=[\delta_f,\delta_r,\delta_p,\delta_e]^\top$, where $\delta_f,\delta_r\in[0,1]$ are the front and rear rotor throttles, $\delta_p\in[0,1]$ is the pusher throttle, and $\delta_e\in[-0.5,0.5]$ is the elevator deflection.

The rigid-body equations are
\begin{equation}
\begin{alignedat}{2}
    \dot{\vp} &= \vv,
    &\qquad
    \dot\theta &= q, \\
    m\dot{\vv} &= -mg\ve_z + \mathbf R(\theta)\vF_{\mathrm{total}},
    &\qquad
    J_y\dot q &= M_{\mathrm{total}}.
\end{alignedat}
\label{eq:vtol_rigid_body}
\end{equation}
where $m$ is the mass, $J_y$ is the pitch moment of inertia, $\ve_z=[0,1]^\top$, $\vF_{\mathrm{total}}$ is the total body-frame force, and $M_{\mathrm{total}}$ is the total pitching moment.

Let $[u,w]^\top=\mathbf R^{-1}(\theta)\vv$ denote the body-frame velocity, $V=\sqrt{u^2+w^2}$ the airspeed, and $\alpha_a=\tan^{-1}(-w/u)$ the angle of attack. We use $\alpha_a$ for angle of attack to avoid conflict with the CBF parameter $\widetilde{\bm{\alpha}}$. The lift coefficient is modeled as a smooth blend between a linear lift curve and a flat-plate model~\cite{willis_pitch_2022},
\begin{equation}
\begin{aligned}
C_L(\alpha_a)
&=
(1-\sigma(\alpha_a))(C_{L_0}+C_{L_\alpha}\alpha_a) \\
& \qquad \qquad
+\sigma(\alpha_a)(2\sin\alpha_a\cos\alpha_a) , \nonumber
\end{aligned}
\label{eq:vtol_lift_coefficient}
\end{equation}
where
\begin{equation}
\sigma(\alpha_a)
=
\frac{
1+e^{-w_s(\alpha_a-\alpha_s)}
 +e^{w_s(\alpha_a+\alpha_s)}
}{
\left(1+e^{-w_s(\alpha_a-\alpha_s)}\right)
\left(1+e^{w_s(\alpha_a+\alpha_s)}\right)
}. \nonumber
\end{equation}
The drag coefficient is
\[
    C_D(\alpha_a)=C_{D_0}+C_{D_{\alpha^2}}\alpha_a^2. 
\]
The baseline aerodynamic force is
\begin{equation}
\vF_0(V,\alpha_a)
=
\frac{1}{2}\rho V^2 S\mathbf R(\alpha_a)
\begin{bmatrix}
-C_D(\alpha_a)\\
C_L(\alpha_a)
\end{bmatrix}, \nonumber
\end{equation}
and the elevator force is
\begin{equation}
\vF_{\delta_e}(V,\alpha_a)\delta_e
=
\frac{1}{2}\rho V^2 S\mathbf R(\alpha_a)
\begin{bmatrix}
-C_{D_{\delta_e}}(\alpha_a)\\
C_{L_{\delta_e}}(\alpha_a)
\end{bmatrix}
\delta_e. \nonumber
\end{equation}
The rotor thrusts are modeled as $T_f=k_f\delta_f$, $T_r=k_r\delta_r$, and $T_p=k_p\delta_p$. Thus,
\begin{equation}
\vF_{\mathrm{total}}
=
\vF_0(V,\alpha_a)
+
\vF_{\delta_e}(V,\alpha_a)\delta_e
+
\begin{bmatrix}
T_p\\
T_f+T_r
\end{bmatrix}. \nonumber
\end{equation}
The total pitching moment $M_{\mathrm{total}}$ is computed from the rotor, elevator, and aerodynamic moment contributions~\cite{beard_small_2012}.

\begin{figure*}[!t]
\centering
\includegraphics[width=0.95\textwidth]{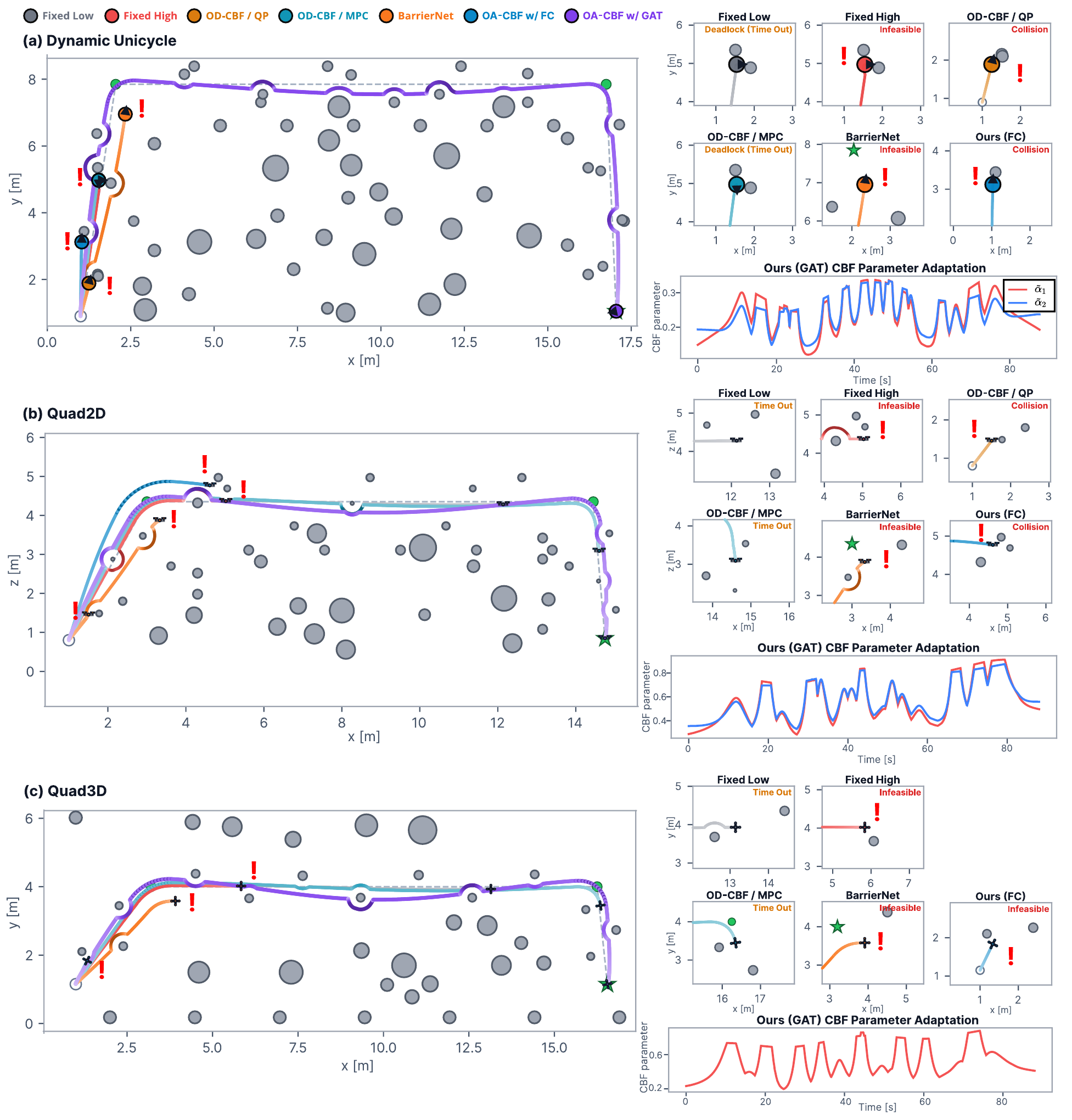}
\caption{Representative trajectories for the distance-based ICCBF benchmarks in \autoref{tab:sdf_results}: (a) dynamic unicycle, (b) Quad2D, and (c) Quad3D. Gray circles denote obstacles, and trajectory color encodes robot speed, with more saturated colors indicating faster motion. Green circles and stars denote reached and next waypoints, respectively. The right panels show representative terminal snapshots for each baseline, labeled by the corresponding outcome: safe, infeasible, collision, or timeout. The bottom plots show the adapted CBF parameters for OA-CBF with graph attention.}
\label{fig:traj}
\vspace{-5pt}
\end{figure*}

\subsection{CBF Candidate Construction}
\label{app:cbf_candidates}

For all systems except the kinematic bicycle benchmark, we use a distance-based collision-avoidance candidate CBF. For obstacle $j$, let $\vp_{\mathrm{obs},j}$ be the obstacle position and let $\ell_{\mathrm{robot}}$ and $\ell_{\mathrm{obs},j}$ be the robot and obstacle radii. The candidate CBF is
\begin{equation}
    \tilde h_j(\vx)
    =
    \|\mathbf P\vp-\vp_{\mathrm{obs},j}\|^2
    -
    (\ell_{\mathrm{robot}}+\ell_{\mathrm{obs},j})^2,
    \label{eq:distance_cbf_candidate}
\end{equation}
where $\mathbf P$ selects the coordinates used for collision avoidance. For planar systems, $\mathbf P$ is the identity. For Quad3D, $\mathbf P$ projects the three-dimensional position onto the horizontal $(x,y)$ plane, so the unsafe set is a planar zone independent of altitude. The CBF-based controllers enforce these inequalities jointly for all obstacles.

For the dynamic unicycle, Quad2D, and VTOL quadplane, the distance-based candidate has relative degree two. Therefore, we use an ICCBF with two linear class-$\mathcal K_\infty$ functions,
\[
    \alpha_i(s)=\widetilde{\alpha}_i s,
    \qquad i\in\{1,2\},
\]
and adapt $\widetilde{\bm{\alpha}} =[\widetilde{\alpha}_1,\widetilde{\alpha}_2]^\top$ online. %

For Quad3D, the same planar obstacle-avoidance constraint has relative degree four under the full rigid-body actuation model. The OA-CBF variants use CBF-QP safety filters, as in the other randomized benchmarks. The MPC-based methods use the sampled-data RK4 CBF formulation of \cite{taylor_safety_2022}, which produces a discrete-time CBF constraint that can be imposed directly in the MPC controller.

To address \textbf{(Q4)}, we also evaluate OA-CBF with a non-distance-based CBF on the kinematic bicycle model. Specifically, we use a Dynamic Parabolic CBF (DPCBF) \cite{park_collision_2026}. Let
\[
    \vp_{\mathrm{rel}}
    =
    \begin{bmatrix}
    p_{\mathrm{obs},x}-p_x\\
    p_{\mathrm{obs},y}-p_y
    \end{bmatrix},
\]
and let $\vv_{\mathrm{rel}}$ be the relative velocity between the obstacle and the robot. Define the line-of-sight angle
\[
    \psi_{\mathrm{los}}
    =
    \operatorname{atan2}(p_{\mathrm{rel},y},p_{\mathrm{rel},x})
\]
and rotate the relative velocity into the line-of-sight frame:
\[
    \tilde{\vv}_{\mathrm{rel}}
    =
    \begin{bmatrix}
    \tilde v_{\mathrm{rel},x}\\
    \tilde v_{\mathrm{rel},y}
    \end{bmatrix}
    =
    \mathbf R(-\psi_{\mathrm{los}})\vv_{\mathrm{rel}}.
\]

Let $r_j=\ell_{\mathrm{robot}}+\ell_{\mathrm{obs},j}$ and
\[
    d_j(\vx)
    =
    \sqrt{\|\vp_{\mathrm{rel}}\|^2-r_j^2},
\]
which is evaluated on the collision-free domain
$\|\vp_{\mathrm{rel}}\|>r_j$. The DPCBF uses two state-dependent design maps,
\[
    \lambda_j(\vx)
    =
    k_\lambda
    \frac{d_j(\vx)}{\|\vv_{\mathrm{rel}}\|},
    \qquad
    \mu_j(\vx)
    =
    k_\mu d_j(\vx),
\]
where $k_\lambda,k_\mu>0$. The CBF candidate is
\begin{equation}
    h_j^{\mathrm{DPCBF}}(\vx)
    =
    \tilde v_{\mathrm{rel},x}
    +
    \lambda_j(\vx)\tilde v_{\mathrm{rel},y}^2
    +
    \mu_j(\vx).
    \label{eq:dpcbf_candidate}
\end{equation}
Unlike the distance-based CBF in \eqref{eq:distance_cbf_candidate}, this candidate evaluates safety in the relative-velocity plane. The parabolic boundary shifts and changes curvature as the relative distance and relative speed vary, so the controller does not treat all motion toward an obstacle as immediately unsafe. For the kinematic bicycle model, this DPCBF has relative degree one. The CBF candidates and their corresponding relative-degree structures are summarized in~\autoref{tab:cbf_relative_degree}.

\subsection{Additional Trajectory Plots}
\label{app:additional_results}

Figure~\ref{fig:traj} provides the trajectory comparisons for the distance-based ICCBF results in \autoref{tab:sdf_results}.

\subsection{Baseline Implementation Details} \label{app:baseline_details}

This appendix provides additional implementation details for the baselines used in \autoref{sec:simulation}. For the MPC-based benchmarks, the fixed-parameter baselines use linear CBF parameters in $(0,1)$ for the discrete-time CBF or ICCBF constraints. The OA-CBF variants use CBF-QP safety filters on the four randomized systems and MPC-CBF in the VTOL case study. The DPCBF benchmark follows the original CBF-QP realization in~\cite{park_collision_2026}.

\textbf{(a) Fixed Low:}
    This baseline uses a fixed low CBF parameter for all time. Low CBF parameters typically make the controller more conservative because the safety constraint is activated earlier and the resulting inner safe set is smaller. We use $\widetilde{\alpha}_1=\widetilde{\alpha}_2=0.01$ for the dynamic unicycle and Quad2D, $\widetilde{\alpha}=0.01$ for Quad3D, $\widetilde{\alpha}_1=\widetilde{\alpha}_2=0.05$ for the VTOL case study, and $\widetilde{\alpha}=0.1$ for the kinematic bicycle.

\textbf{(b) Fixed High.}
    This baseline uses a fixed high CBF parameter. High CBF parameters generally allow more aggressive behavior when the system is far from obstacles, but they can also cause loss of feasibility near obstacles under input constraints. We use $\widetilde{\alpha}_1=\widetilde{\alpha}_2=0.35$ for the dynamic unicycle and VTOL case study, $\widetilde{\alpha}_1=\widetilde{\alpha}_2=0.99$ for Quad2D, and $\widetilde{\alpha}=0.99$ for Quad3D. For the kinematic-bicycle DPCBF benchmark, we set $\widetilde{\alpha}=70.0$ because the closed-loop behavior under the DPCBF-QP formulation empirically saturates around this parameter value.

\textbf{(c) OD-CBF with QP.}
    This baseline uses the optimal-decay CBF-QP method \cite{zeng_safetycritical_2021_decay, ong_properties_2025}. The CBF parameters are optimized online inside an instantaneous CBF-QP and penalized in the objective. For relative-degree-two CBFs, both CBF coefficients are optimized. This method is included when the corresponding continuous-time CBF-QP is well defined.

\textbf{(d) OD-CBF with MPC.}
    This baseline uses the optimal-decay MPC-CBF method \cite{zeng_enhancing_2021}, which incorporates online decay-rate optimization inside the MPC-CBF formulation. This baseline uses the same dynamics, constraints, prediction horizon, and nominal cost as the other MPC-based controllers.

\textbf{(e) BarrierNet.}
    We include BarrierNet~\cite{xiao_barriernet_2023} as a differentiable CBF-QP-layer baseline. Its upstream neural network predicts task-directed reference-control and CBF-related quantities from the current state and environment, and the final layer solves a multi-constraint CBF-QP. At each control step, it selects the five closest obstacles by signed clearance and constructs a fixed-size neural input from per-obstacle features $(\Delta p_x,\Delta p_y,\Delta\theta,d_{\mathrm{clr}},v)$, where $\Delta p_x$ and $\Delta p_y$ denote the obstacle position relative to the robot in the robot body frame, $\Delta\theta$ denotes the relative bearing angle to the obstacle, $d_{\mathrm{clr}}$ denotes the signed clearance to the obstacle boundary, and $v$ denotes the robot speed. The raw QP context contains the reduced robot state, the goal, and the selected obstacle states. A shared per-obstacle encoder predicts obstacle-wise CBF parameters, while a control head predicts a nominal correction around the reference input. In our implementation, the differentiable QP layer is trained with slack variables and without hard input bounds to avoid degenerate training batches; during deployment, the QP is evaluated with input bounds, and the computed input is clipped at the end. The network is trained by imitation on safe, goal-reaching CBF-QP rollouts generated from the same obstacle distribution, while failed, colliding, or timeout rollouts are discarded. Because BarrierNet relies on a continuous-time CBF-QP layer, it is evaluated only when the corresponding CBF-QP constraint is well defined.

\textbf{(f) OA-CBF with fully connected encoding (OA-CBF w/ FC).}
    This baseline follows the fully connected environment encoding implementation provided by \cite{kim_learning_2025,kim_how_2025}. In that implementation, the neural-network input contains the robot state, the candidate CBF parameter, and the closest-obstacle information only. This produces a fixed-dimensional input vector and avoids padding, but it can discard safety-relevant information when multiple obstacles are simultaneously relevant. This baseline is included to isolate the effect of replacing the fixed-dimensional closest-obstacle encoding with the proposed graph-attention encoding.

\textbf{(g) OA-CBF with graph-attention encoding (OA-CBF w/ GAT).}
    This is the proposed method. The environment is represented as a graph consisting of the robot, the goal, and all observed obstacles. Attention-based aggregation produces a fixed-dimensional context vector without requiring padding, truncation, or heuristic obstacle selection. The resulting context is concatenated with the candidate CBF parameter and evaluated by the PENN.

\fi % arxivversion: Appendix A

\fi

\ifarxivversion
  
\section{Discrete-Time ICCBF and MPC-CBF Formulation}
\label{app:discrete_time}

The main development of this paper is stated in continuous time to keep the local-validation and uncertainty-aware adaptation arguments independent of a particular numerical implementation. In sampled-data implementations, the same CBF parameter can be used inside a discrete-time CBF or MPC-CBF safety filter. This appendix describes the corresponding discrete-time formulation used when the controller \(\pi(\cdot;\bm{\alpha})\) is realized by MPC. The online adaptation procedure in \autoref{sec:online} remains unchanged: fixed parameters \(\bm{\alpha}\) in the safety-filter constraint are replaced by the adapted parameter \(\widetilde{\bm{\alpha}}\).\footnote{With a slight abuse of notation, the symbols \(h\), \(b_i\), \(\calC_i\), \(\calC^{\star}\), and related sets/functions are reused in Appendix~\ref{app:discrete_time} for their discrete-time counterparts. Throughout the main text, these symbols refer to the continuous-time definitions.}

Consider the discrete-time nonlinear system
\begin{equation}
    \vx_{k+1}=f_d(\vx_k,\vu_k),
    \label{eq:dt-dynamics}
\end{equation}
where \(\vx_k\in\StateSpace_d\subset\RealSpace^n\) and \(\vu_k\in\ControlSpace_d\subset\RealSpace^m\) denote the state and admissible control input at time step \(k\in\mathbb{Z}_{\ge 0}\), respectively. The map \(f_d:\StateSpace_d\times\ControlSpace_d\to\StateSpace_d\) is assumed to be locally Lipschitz continuous. When \eqref{eq:dt-dynamics} is obtained from a continuous-time model, \(f_d\) denotes the chosen sampled-data flow map over one control period, e.g., the one-step Runge--Kutta integration map used in implementation. For any scalar function \(q:\StateSpace_d\to\RealSpace\), define
\begin{equation}
    \Delta q(\vx_k,\vu_k) \coloneqq q(f_d(\vx_k,\vu_k))-q(\vx_k).
\end{equation}

\begin{definition}[Discrete-Time CBF~\cite{agrawal_discrete_2017}]
\label{def:dt-cbf}
Let \(\calC=\{\vx\in\StateSpace_d\mid h(\vx)\ge 0\}\subseteq\calS\), where \(h:\StateSpace_d\to\RealSpace\) is continuous. The function \(h\) is a discrete-time CBF for \eqref{eq:dt-dynamics} on \(\calC\) if there exists a class-\(\calK\) function \(\alpha\) satisfying \(\alpha(z)<z\) for all \(z>0\) such that
\begin{equation}
    \sup_{\vu_k\in\ControlSpace_d}
    \left[\Delta h(\vx_k,\vu_k)+\alpha(h(\vx_k))\right]
    \ge 0,
    \qquad \forall \vx_k\in\calC .
    \label{eq:dt-cbf-condition}
\end{equation}
\end{definition}

The corresponding admissible-control set is
\begin{equation}
    K_{\mathrm{dtcbf}}(\vx_k;\alpha)
    \coloneqq
    \left\{\vu_k\in\ControlSpace_d\mid
    \Delta h(\vx_k,\vu_k)+\alpha(h(\vx_k))\ge 0
    \right\}. \nonumber
    \label{eq:k-dtcbf} 
\end{equation}

\begin{lemma}
\label{lem:dt-cbf}
Suppose that \(h\) is a discrete-time CBF for \eqref{eq:dt-dynamics} with parameter \(\alpha\). If \(\vx_0\in\calC\) and \(\vu_k\in K_{\mathrm{dtcbf}}(\vx_k;\alpha)\) for all \(k\ge 0\), then \(\vx_k\in\calC\) for all \(k\ge 0\).
\end{lemma}

\begin{proof}
If \(\vx_k\in\calC\), then \(h(\vx_k)\ge 0\). Since \(\vu_k\in K_{\mathrm{dtcbf}}(\vx_k;\alpha)\),
\[
    h(\vx_{k+1})
    \ge h(\vx_k)-\alpha(h(\vx_k))
    \ge 0,
\]
where the last inequality follows from \(\alpha(z)<z\) for \(z>0\) and \(\alpha(0)=0\). Thus \(\vx_{k+1}\in\calC\). Induction proves the result.
\end{proof}

\begin{remark}
A common choice is \(\alpha(s)=\gamma s\), where \(\gamma\in(0,1)\) for the strict condition in \autoref{def:dt-cbf}. Then the discrete-time CBF condition gives \(h(\vx_{k+1})\ge(1-\gamma)h(\vx_k)\), and hence \(h(\vx_k)\ge(1-\gamma)^k h(\vx_0)\) for all \(k\ge0\). The limiting choice \(\gamma=1\) still preserves forward invariance and recovers the standard one-step discrete-time exponential CBF condition.
\label{rem:linear_dt_cbf}
\end{remark}

We next state the discrete-time counterpart of the ICCBF construction in \autoref{subsec:iccbf}. Let \(h\) be the candidate safety function and let \(r\) denote the order used in the discrete-time ICCBF construction. Define \(b_i:\StateSpace_d\to\RealSpace\), \(i=0,\ldots,r-1\), recursively as
\begin{subequations}
\label{eq:dt-iccbf-functions}
\begin{align}
    b_0(\vx_k;\bm{\alpha})
    &\coloneqq h(\vx_k), \\
    b_i(\vx_k;\bm{\alpha})
    &\coloneqq
    \inf_{\vu_k\in\ControlSpace_d}
    \left[
    \Delta b_{i-1}(\vx_k,\vu_k;\bm{\alpha})
    \right]
    \notag\\[-0.2em]
    &\quad
    +\alpha_i\!\left(b_{i-1}(\vx_k;\bm{\alpha})\right),
    \quad i=1,\ldots,r-1 .
\end{align}
\end{subequations}
The terminal constraint function is
\begin{equation}
    b_r(\vx_k,\vu_k;\bm{\alpha})
    \coloneqq
    \Delta b_{r-1}(\vx_k,\vu_k;\bm{\alpha})
    +\alpha_r\!\left(b_{r-1}(\vx_k;\bm{\alpha})\right), \nonumber
    \label{eq:dt-iccbf-constraint}
\end{equation}
where \(\bm{\alpha}=\{\alpha_1,\ldots,\alpha_r\}\) is a set of class-\(\calK\) functions satisfying \(\alpha_i(z)<z\) for all \(z>0\).

Define
\begin{subequations}
\label{eq:dt-iccbf-sets}
\begin{align}
    \calC_0(\bm{\alpha})
    &\coloneqq
    \{\vx\in\StateSpace_d\mid b_0(\vx;\bm{\alpha})\ge 0\}\subseteq\calS, \nonumber \\
    \calC_i(\bm{\alpha}) 
    &\coloneqq
    \{\vx\in\StateSpace_d\mid b_i(\vx;\bm{\alpha})\ge 0\},
    \quad i=1,\ldots,r-1, \nonumber
\end{align}
\end{subequations}
and the corresponding discrete-time inner safe set
\begin{equation}
    \calC^{\star}(\bm{\alpha})
    \coloneqq
    \bigcap_{i=0}^{r-1}\calC_i(\bm{\alpha}) . \nonumber
    \label{eq:dt-inner-safe-set}
\end{equation}

\begin{definition}[Discrete-Time ICCBF]
\label{def:dt-iccbf}
The function \(h\) is a discrete-time ICCBF on \(\calC^{\star}(\bm{\alpha})\) for \eqref{eq:dt-dynamics} if there exist class-\(\calK\) functions \(\bm{\alpha}=\{\alpha_1,\ldots,\alpha_r\}\) such that
\begin{equation}
    \sup_{\vu_k\in\ControlSpace_d}
    b_r(\vx_k,\vu_k;\bm{\alpha})
    \ge 0,
    \qquad
    \forall \vx_k\in\calC^{\star}(\bm{\alpha}) . \nonumber
    \label{eq:dt-iccbf-condition}
\end{equation}
\end{definition}

The set of admissible controls satisfying the discrete-time ICCBF constraint is
\begin{equation}
    K_{\mathrm{dticcbf}}(\vx_k;\bm{\alpha})
    \coloneqq
    \left\{\vu_k\in\ControlSpace_d\mid
    b_r(\vx_k,\vu_k;\bm{\alpha})\ge 0
    \right\}. \nonumber
    \label{eq:k-dticcbf}
\end{equation}

\begin{lemma}
\label{lem:dt-iccbf}
Suppose that \(h\) is a discrete-time ICCBF for \eqref{eq:dt-dynamics} with parameter \(\bm{\alpha}\). If \(\vx_0\in\calC^{\star}(\bm{\alpha})\) and \(\vu_k\in K_{\mathrm{dticcbf}}(\vx_k;\bm{\alpha})\) for all \(k\ge 0\), then \(\vx_k\in\calC^{\star}(\bm{\alpha})\subseteq\calS\) for all \(k\ge0\).
\end{lemma}

\begin{proof}
Assume \(\vx_k\in\calC^{\star}(\bm{\alpha})\). Since \(\vu_k\in K_{\mathrm{dticcbf}}(\vx_k;\bm{\alpha})\),
\[
    b_{r-1}(\vx_{k+1};\bm{\alpha})
    \ge
    b_{r-1}(\vx_k;\bm{\alpha})
    -\alpha_r(b_{r-1}(\vx_k;\bm{\alpha}))
    \ge 0.
\]
Moreover, for each \(i=1,\ldots,r-1\), the condition \(b_i(\vx_k;\bm{\alpha})\ge0\), together with the definition of \(b_i\) in \eqref{eq:dt-iccbf-functions}, implies
\[
    \Delta b_{i-1}(\vx_k,\vu;\bm{\alpha})
    +\alpha_i(b_{i-1}(\vx_k;\bm{\alpha}))
    \ge0,
    \qquad \forall \vu\in\ControlSpace_d .
\]
This inequality holds in particular for the applied input \(\vu_k\), and therefore
\[
    b_{i-1}(\vx_{k+1};\bm{\alpha})
    \ge
    b_{i-1}(\vx_k;\bm{\alpha})
    -\alpha_i(b_{i-1}(\vx_k;\bm{\alpha}))
    \ge0 .
\]
Applying this argument recursively for \(i=r-1,r-2,\ldots,1\) shows that \(b_i(\vx_{k+1};\bm{\alpha})\ge0\) for every \(i=0,\ldots,r-1\). Thus \(\vx_{k+1}\in\calC^{\star}(\bm{\alpha})\). Induction completes the proof.
\end{proof}

A sampled-data safety-critical controller can then be implemented through an MPC-CBF problem in which the CBF or ICCBF constraint is enforced along a finite prediction horizon~\cite{zeng_safetycritical_2021,zeng_enhancing_2021}. Given the current state \(\vx_k\), the prediction horizon \(\horizon\), stage cost \(L\), and terminal cost \(M\), the controller solves 

\noindent\rule{\columnwidth}{0.4pt}
\begingroup
\small
\setlength{\jot}{2pt}
\textbf{MPC-CBF:}
\begin{subequations}
\label{eq:mpc-cbf-formulation}
\begin{alignat}{2}
& J^{\star}(\vx_k;\bm{\alpha})
= {} \min_{\vu_{k:k+\horizon-1|k}}
 \quad \quad\bigg[
&& M(\vx_{k+\horizon|k})
 \notag\\[-0.2em]
&\quad
+\sum_{\tau=0}^{\horizon-1}
L(\vx_{k+\tau|k},\vu_{k+\tau|k})
\bigg]
&& \label{eq:mpc-cbf-cost}\\
\text{s.t.}\quad
& \vx_{k|k}=\vx_k,
&& \label{eq:mpc-cbf-initial-condition}\\
& \vx_{k+\tau+1|k}
=f_d(\vx_{k+\tau|k},\vu_{k+\tau|k}),
&& \; \tau=0,\ldots,\horizon-1,
\label{eq:mpc-cbf-dynamics}\\
& \vu_{k+\tau|k}\in\ControlSpace_d,
&& \; \tau=0,\ldots,\horizon-1,
\label{eq:mpc-cbf-input-constraint}\\
& b_r(\vx_{k+\tau|k},\vu_{k+\tau|k};\bm{\alpha})\ge0,
&& \; \tau=0,\ldots,\horizon-1.
\label{eq:mpc-cbf-constraint}
\end{alignat}
\end{subequations}
\endgroup
\par\noindent\rule{\columnwidth}{0.4pt}

The predicted state \(\vx_{k+\tau|k}\) denotes the state at future step \(k+\tau\) computed from information available at step \(k\). The first optimizer \(\vu^{\star}_{k|k}\) is applied to the system, and the problem is solved again at the next time step. When a standard discrete-time CBF is used instead of a higher-order/input-constrained construction, the last constraint in \eqref{eq:mpc-cbf-constraint} is replaced by \(\Delta h(\vx_{k+\tau|k},\vu_{k+\tau|k})+\alpha(h(\vx_{k+\tau|k}))\ge0\). In fixed-parameter baselines, \(\bm{\alpha}\) is held constant; in OA-CBF, \(\bm{\alpha}\) is replaced at each update by the verified parameter \(\widetilde{\bm{\alpha}}^{\star}_k\) selected by Alg.~\ref{alg:online_adaptation}.

\else
  % Biography layout and placeholder portrait follow the supplied TCST template.
% Replace the placeholder with each author's portrait when available.
\begin{authorbiography}[{\includegraphics[width=1in,height=1.25in,clip,keepaspectratio]{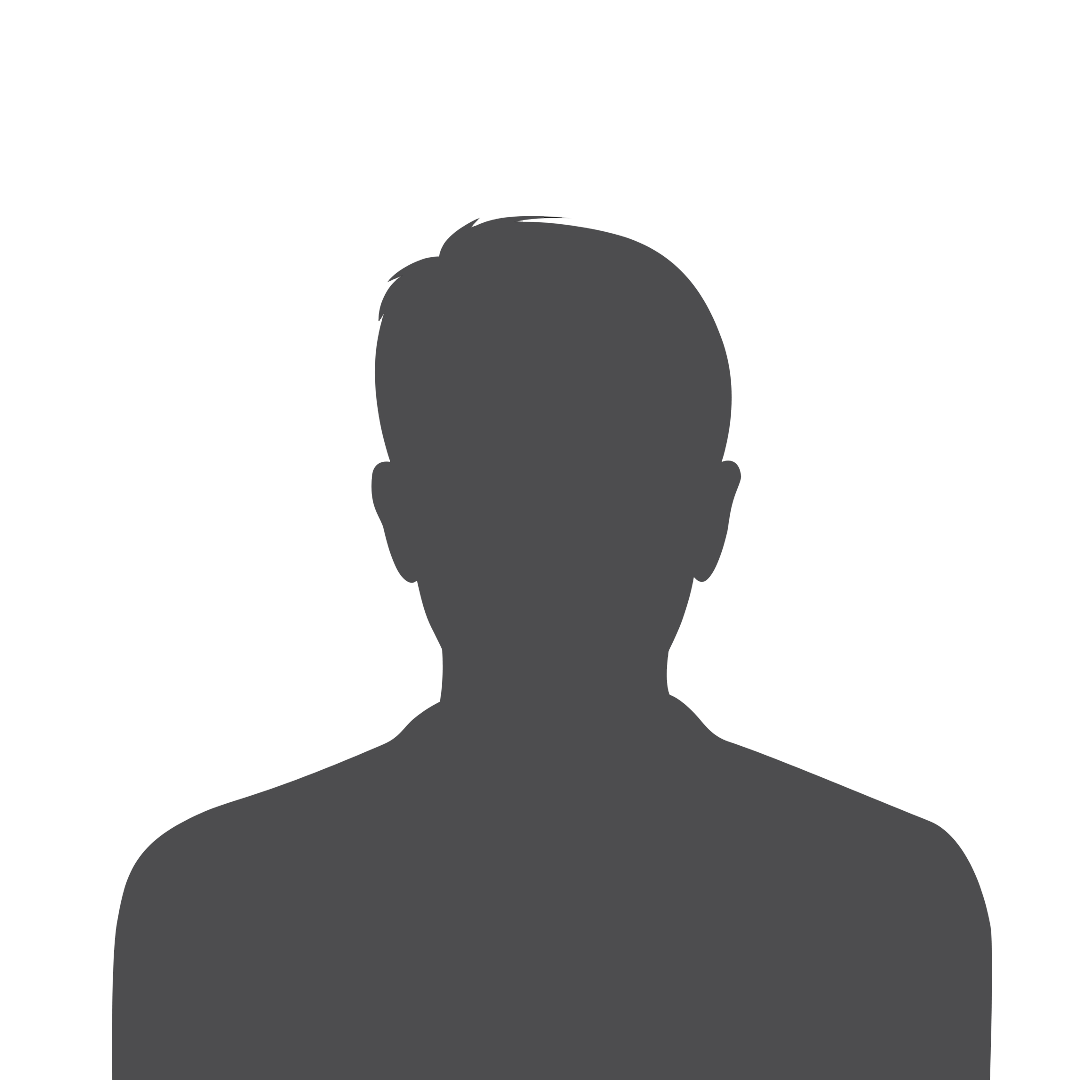}}]{Taekyung Kim}
is a Ph.D. candidate in Robotics at the University of Michigan. He received his B.S. degree from the Daegu Gyeongbuk Institute of Science and Technology (DGIST), Republic of Korea, in 2020. From 2020 to 2023, he was a researcher at the Agency for Defense Development (ADD), Republic of Korea. His research interests include safety-critical planning and control, safe robot learning, and model-based reinforcement learning, with an emphasis on provably safe robotic autonomy under uncertainty.
\end{authorbiography}

\begin{authorbiography}[{\includegraphics[width=1in,height=1.25in,clip,keepaspectratio]{Figures/Headshot-Placeholder.png}}]{Robin Inho Kee}
is a Draper Scholar and a Ph.D. student in robotics at the University of Michigan, Ann Arbor. He received the B.S. degree in mechanical engineering from Yonsei University in 2022 and the M.S. degree in mechanical engineering from the University of Michigan in 2024. His research interests include multi-robot autonomy, communication-aware coordination, and learning-enabled control.
\end{authorbiography}

\begin{authorbiography}[{\includegraphics[width=1in,height=1.25in,clip,keepaspectratio]{Figures/Headshot-Placeholder.png}}]{Dimitra Panagou}
received the Diploma and Ph.D. degrees in mechanical engineering from the National Technical University of Athens. She is currently a Full Professor with the Department of Robotics and the Department of Aerospace Engineering at the University of Michigan. Her research interests include nonlinear systems and control, multiagent systems, autonomy, and aerospace robotics, with emphasis on provably correct methods for the safe and resilient operation of autonomous systems under uncertainty. She is a recipient of the NASA Early Career Faculty Award, the AFOSR Young Investigator Award, and the NSF CAREER Award.
\end{authorbiography}

\fi

\end{document}